\documentclass[preprint,12pt, 3p]{elsarticle}
\usepackage{amsthm}
\theoremstyle{plain} 
\newtheorem{theorem}{Theorem}
\newtheorem{proposition}[theorem]{Proposition}
\makeatletter
\renewenvironment{proof}[1][\proofname]{%
  \par\pushQED{\qed}%
  \normalfont
  \topsep6\p@\@plus6\p@\relax
  \trivlist
  \item[\hskip\labelsep\bfseries #1\@addpunct{.}]%
}{%
  \popQED\endtrivlist\@endpefalse
}
\makeatother
\theoremstyle{definition}

\usepackage{amsmath}
\usepackage{amssymb}
\usepackage{mathrsfs}
\usepackage{hyperref}
\usepackage{bm}
\usepackage{accents}
\usepackage{xcolor}
\usepackage[table]{xcolor}
\usepackage{graphicx}
\usepackage{wrapfig}
\usepackage{soul}

\newcommand{\REVISE}[1]{\textcolor{black}{{#1}}}
\newcommand{\rev}[1]{\REVISE{#1}}

\usepackage{float}
\usepackage{placeins}
\usepackage{algorithm}
\usepackage{algpseudocode}
\usepackage{tcolorbox}
\tcbuselibrary{breakable}

\newtcolorbox{codeboxraw}{%
  breakable,
  colback=gray!10,
  colframe=gray!10,
  boxrule=0pt,
  left=6pt,
  right=6pt,
  top=6pt,
  bottom=6pt
}

\usepackage{caption}
\usepackage{subcaption}
\usepackage{booktabs}
\usepackage{multirow}
\usepackage{tabularx}
\usepackage{array}
\usepackage{longtable}
\usepackage{pdflscape}

\makeatletter

\newcommand{\maybeincludegraphics}[2][]{%
  \IfFileExists{#2.pdf}{\includegraphics[#1]{#2.pdf}}{%
    \IfFileExists{../#2.pdf}{\includegraphics[#1]{#2.pdf}}{%
      \IfFileExists{#2}{\includegraphics[#1]{#2}}{%
        \IfFileExists{../#2}{\includegraphics[#1]{#2}}{%
          \IfFileExists{#2.png}{\includegraphics[#1]{#2.png}}{%
            \IfFileExists{../#2.png}{\includegraphics[#1]{#2.png}}{%
              \fbox{%
                \parbox[c][3.0cm][c]{0.95\linewidth}{\centering\footnotesize Missing file:\ \texttt{\detokenize{#2}}}%
              }%
            }%
          }%
        }%
      }%
    }%
  }%
}
\makeatother

\newcommand{\stresspanel}[1]{%
  \maybeincludegraphics[width=\linewidth]{#1}%
}
\newcommand{\stresslegend}[1]{%
  \includegraphics[width=\linewidth,height=0.14\textwidth,keepaspectratio]{#1}%
}
\newcommand{\energypanel}[1]{%
  \maybeincludegraphics[width=\linewidth]{#1}%
}
\newcommand{\energylegend}[1]{%
  \includegraphics[width=0.95\linewidth,height=0.12\textwidth,keepaspectratio]{#1}%
}
\newcommand{\membershiplegend}[1]{%
  \includegraphics[width=\linewidth,height=0.10\textwidth,keepaspectratio]{#1}%
}
\newcommand{\membershipbounds}[1]{%
  \includegraphics[width=0.72\linewidth,keepaspectratio]{#1}%
}
\newcommand{\stressfigcaption}[1]{%
  \vspace{-0.45cm}%
  \caption{#1}%
}

\graphicspath{{../}}

\journal{Computer Methods in Applied Mechanics and Engineering / arXiv}

\begin{document}

\begin{frontmatter}

\title{\textbf{Interval and fuzzy physics-augmented neural networks (iPANN and fPANN) for uncertainty quantification and propagation in constitutive modeling}}

\author[1]{Somesh Pratap Singh}
\author[2]{Govinda Anantha Padmanabha}
\author[4]{Jingye Tan}
\author[1]{Steven Yang}
\author[3]{{Reese~E.~Jones}}
\author[3]{D. Thomas Seidl} 
\author[1,5]{Nikolaos Bouklas\corref{cor1}}
\ead {nbouklas@cornell.edu}

\cortext[cor1]{corresponding author: Nikolaos Bouklas}

\address[1]{Sibley School of Mechanical and Aerospace Engineering, Cornell University, Ithaca, 14853, NY, USA\fnref{label1}}
\address[2]{Ecole Polytechnique Federale de Lausanne (EPFL), 1015, Lausanne, Switzerland\fnref{label2}}
\address[3]{Sandia National Laboratories, Livermore, 94551, CA, USA\fnref{label2}}
\address[4]{Department of Aerospace and Mechanical Engineering, University of Southern California, Los Angeles,
90007, CA, USA\fnref{label4}}
\address[5]{Pasteur Labs, Brooklyn, 11205, NY, USA\fnref{label5}}

\begin{abstract}
Constitutive modeling under uncertainty remains a central challenge for reliable mechanics simulations, particularly when the available stress--deformation data are sparse, noisy, or heterogeneous.
We propose interval and fuzzy physics-augmented neural networks (iPANNs and fPANNs)  for uncertainty-aware hyperelastic constitutive modeling.
iPANNs learn sparse lower, mean, and upper free energy density branches whose stresses, obtained by automatic differentiation, ultimately enclose noisy stress observations.
In contrast to this deterministic interval description, fPANNs embed the learned iPANN
 branches into a fuzzy-set representation through \rev{membership-level ($\alpha$-cut)} interpolation, yielding a nested 
 family of admissible responses. 
iPANNs and fPANNs encode mechanistic constraints by preserving objectivity, consistency and promoting polyconvexity and, smoothed $L_0$ regularization promoting interpretable energy representations.
The bound models are trained through a two-stage transfer-learning procedure in which a sparse mean constitutive response is learned first and then fine-tuned into lower and upper energy branches.
We evaluate the framework on synthetic isotropic hyperelastic data with heteroscedastic noise, varying random realizations, shifted noise means, and varying noise magnitudes.
The results show that the learned bounds enclose noisy stress observations while generalizing to the test set. 
Further, we examine the propagation of uncertainty through the mean, upper and lower bound  predictions of the learned iPANN models in a finite element setting.
The proposed framework provides a compact,  physics-consistent route for distribution-free aleatoric uncertainty quantification in hyperelastic constitutive modeling, and propagation in downstream finite element simulations.

\end{abstract}


\end{frontmatter}

\section{Introduction}

\subsection{Motivation}
Constitutive laws have been one of the most important areas of research in mechanics since they form the link between the 
kinematics and the response of the material to an applied load; however, developing new constitutive laws, calibrating their parameters, 
and quantifying the uncertainty of the data and the developed models has been a persistent bottleneck for a multitude of reasons.
Traditionally, constitutive laws have been developed using \rev{a} phenomenological approach where one poses an ansatz based on experimental observations and physical requirements.
The parameters in the ansatz, as well as the form of the ansatz itself, are selected to characterize the material response~\cite{rivlin,mooney,ogden}.
Phenomenological modeling is the archetype of data-driven modeling in mechanics.
However, in order to capture the response in complex three-dimensional deformation states, constitutive laws are commonly cast in terms of tensor-valued functions between stress and various deformation metrics~\cite{gurtin,holzapfel}.
Due to the limitations associated with traditional mechanical testing, macroscopically homogeneous deformations are only accessible for a limited set of load cases (uniaxial, biaxial, hydrostatic, simple shear, etc.), which restricts access to data that are uniformly sampled in deformation or deformation-rate space.
An individual loading curve might have \rev{many (correlated) data points}, but in strain or strain-rate space, which are inherently high-dimensional, only a few states are accessible, providing limited observations.
The field of continuum mechanics traditionally enabled closing the gap between limited experimental observations (as is common in mechanical testing) and constitutive laws that proficiently generalize to complex loading conditions.
This approach is similar to current trends in improving machine learning (ML)  under sparse observations utilizing implicit biases.
Beyond simple mechanical testing, maturation of imaging technologies enabled the acquisition of full-field observations and the detailed analysis of heterogeneous deformation states.
Such approaches include digital image correlation (DIC)~\cite{dic1,dic2}, digital volume correlation~\cite{dvc1}, X-ray computed tomography~\cite{xrt1}, and X-ray diffraction more broadly~\cite{matt_miller}, which provide richer information for mechanical testing and more complex datasets. 
Consequently, verification and  validation of constitutive laws in more complex loading states exhibiting heterogeneous deformations have also been intense topics of investigation.
Finite element method updating (FEMU)~\cite{mathieu2015estimation,femu} and the virtual fields method (VFM)~\cite{vfm} have been the major approaches for parameter calibration in this context.
Overall, the paradigm of phenomenological modeling often remains a bottleneck in engineering practice, as constitutive modeling remains a heavily user-controlled iterative process without tools for automation.

\subsection{Related work}
Data-driven approaches for constitutive modeling \citep{ddreview} had emerged multiple decades ago~\cite{gaboussi}, but until recently were limited to simple problems and often only uniaxial responses.
Data-driven constitutive modeling approaches can be broadly classified into two categories:  machine learning (ML)-based and model-free approaches.
Model-free approaches remove the need for an explicit constitutive model by directly embedding material observations into the solution of the governing equations.
The dataset is itself a part of the solution for model-free approaches~\cite{ortiz2022,ddreview}.
On the other hand, supervised ML-based approaches aim to learn the map between input-output pairs (e.g., strain-stress space).
One of the most widely-used approaches to learn the input-output map is neural networks (NNs), where their property as universal approximators~\cite{goodfellow2016deep,cybenko1989ufc1,hornik1989ufc2,park1991ufc3,hornik1991ufc4} is utilized.
This implies that NNs fit even highly non-linear maps, but this often comes with a large number of parameters, which can cause overfitting and limit interpretability.
Several solutions for this problem, pertaining to both interpretability and encoding physics, have been proposed.
A common approach for enabling interpretability in data-driven learning is through symbolic regression~\cite{sym1_schmidt2009,sym2_koza1994, sym4_wang2019} and sparse regression~\cite{sparse1_brunton2016,sparse2_tibshirani1996}.
Physical priors can be encoded as soft constraints, as in the case of physics-informed neural networks (PINNs)~\cite{pinn1_raissi2019}, where a PDE residual is evaluated in the loss, or by choosing an architecture that strongly enforces a physical constraint such as in the case of input convex neural networks (ICNNs)~\cite{icnn_amos2017}. 
Analogous embedding of physical or mechanistic priors in neural networks used for constitutive modeling is commonly termed as physics-augmented neural networks (PANNs).
In the context of constitutive modeling, there has been an array of works focusing on learning from labeled data pairs from experimental and synthetic sources.
They have enabled: ensuring  objectivity and enforcing symmetry classes for anisotropy~\cite{linka, fuhg2022physics, patel2025general}, polyconvexity~\cite{yang2026mixture, pcnn_klein2022,fuhg2024polyconvex} and  monotonicity~\cite{mnn_klein2025} in the context of hyperelasticity,
as well as: enforcing mechanistic constraints\footnote{We refer to mechanistic constraints as a separate class,  as opposed to physical constraints, as they encode specific knowledge for the \rev{material system} rather than enforcing physical laws directly.} for elastoplasticity~\cite{fuhg2023modular} and inelasticity more broadly~\cite{Jones_2022,steinmann_GST,jones2026hierarchy}.
Regarding learning from complex imaging data,  EUCLID~\cite{sparse3_flaschel2021} showed that isotropic hyperelastic constitutive laws can be developed using unsupervised learning within a virtual fields framework using only full-field displacements and global force data quantities readily obtained from DIC and standard mechanical testing, without labeled stress observations.
 By enforcing equilibrium in the bulk and on the loaded boundary, and selecting sparse combinations from a library of candidate free energy density terms, the method recovers compact, interpretable model forms rather than black-box surrogates.
 More recent works instead combine ML with adjoint methods closer to the FEMU paradigm~\cite{tan2026towards,femu,mathieu2015estimation}.

For safe and trustworthy implementations in engineering applications and downstream simulations,  
deterministic point predictions from NNs are not sufficient and uncertainty quantification is required~\cite{uqj1_nemani2023,uqj2_fernandez2022,uqj3_fong2006,uqj4_cheng2018,uqj5_abulawi2025,uqj12_begoli2019}. 
Both aleatoric and epistemic uncertainty~\cite{uqj11_hullermeier2021} are important in the study of engineering systems. 
Epistemic uncertainty arises due to inadequate knowledge, common in model form error.
Aleatoric uncertainty is innate to the data distribution, is irreducible, and commonly emanates from the presence of noise.
Uncertain parameters in most applications are modeled through probability distributions and stochastic fields~\cite{uqj6_ostoja1998,uqj7_ostoja2006,uqj8_der1988,uqj9_stefanou2009}. 
Particularly for deep learning applications, Bayesian Neural Networks (BNNs) are one of the most widely used uncertainty quantification methods~\cite{mackay1992practical,neal1996bayesian,blundell2015weight,zhu2018bayesian}.
They provide a probabilistic treatment of model parameters rather than estimating a single deterministic set of weights.
BNNs follow a Bayesian formulation~\cite{gelman1995bayesian} where prior distributions are assigned to the network weights and biases, and training updates these priors into posterior distributions conditioned on the observed data.
Predictions are then obtained by marginalizing over the posterior distribution of the model parameters, which enables the predictive distribution to reflect both the expected response and the associated uncertainty~\cite{neal1996bayesian,jospin2022hands}.
This framework is especially useful for quantifying epistemic uncertainty, since uncertainty in the weights naturally captures the effect of limited or sparse training data, while appropriate likelihood formulations can also account for aleatoric uncertainty~\cite{kendall2017uncertainties}.
In practice, exact Bayesian inference in deep neural networks is generally intractable due to the number of parameters, so approximate inference techniques such as variational inference, Monte Carlo dropout, and related sampling-based methods are commonly employed~\cite{graves2011practical,blundell2015weight,gal2016dropout}.
Other methods which combine ensembles and optimization, such as Stein variational descent~\cite{liu2016stein}, have also shown to be effective with neural networks~\cite{padmanabha2025concurrent}.
A detailed review of uncertainty quantification methods in deep learning can be found in  Ref.~\cite{uqj15_abdar2021}.

Much of the work on uncertainty quantification for constitutive laws focuses on
epistemic uncertainty~\cite{padmanabha2024improving,padmanabha2025concurrent}
and model-form error~\cite{modelform}.
In Bayesian inverse problems arising from parameter identification, Ref.~\cite{tonini_uq} developed uncertainty-quantification variational autoencoders  (UQ-VAEs) that approximate the posterior mean and covariance of model parameters from noisy observational data,  with a reformulated training objective that reduces cost in high-dimensional parameter spaces and supports forward uncertainty propagation.
For history-dependent responses, Ref.~\cite{bessa1_uq} extended data-driven constitutive modeling to single- and multi-fidelity datasets while jointly estimating mean predictions, epistemic model uncertainty, and aleatoric noise through hierarchical Bayesian recurrent networks.
By disentangling these uncertainty sources, such approaches aim to improve the trustworthiness of learned constitutive maps when training data differ in fidelity, cost, and noise content.
Complementing fully Bayesian treatments, Ref.~\cite{quantile_bahador} introduced a distribution-free probabilistic framework for physics-constrained anisotropic constitutive models based on conformalized quantile regression over tensor-valued stress fields.
Within a polyconvex, thermodynamically consistent invariant formulation, the method learns conditional stress quantiles directly from data and can be applied in a plug-and-play manner to endow existing deterministic models with prediction intervals, without Monte  Carlo sampling at inference.
Unlike quantile-based probabilistic intervals, interval- and fuzzy-set formulations provide deterministic enclosures on admissible constitutive response and are particularly suited when bounds, rather than conditional distributions, are the quantity of interest.
Other probabilistic approaches generally require knowledge about the underlying probability distribution that is not always easy to obtain in mechanics applications due to data sparsity.
For this reason, non-probabilistic approaches like interval analysis~\cite{inta1_alefeld2011,inta2_neumaier1990} and fuzzy-set theory~\cite{fuzzy1_klir1995,fuzzy2_zimmermann2011} remain an alternative to Bayesian approaches.
\rev{(A motivating one-dimensional illustrative comparison of Bayesian credible bands with deterministic interval enclosures is given in Section~\ref{sec:simple_bayesian_example}.)}
Interval analysis is particularly useful when deterministic bounds on the inputs are available, a case commonly utilized in engineering practice when developing for worst case scenarios.
Furthermore, one typically appeals to fuzzy set theory in cases where uncertainty bounds are prescribed in lexical terms.

Parallel to these non-probabilistic uncertainty models, interval and fuzzy information has also been propagated directly within 
finite element frameworks.
Ref.~\cite{chen2000intervalfem} formulated an interval finite element method for structures with bounded uncertain parameters, 
combining conventional assembly and solution procedures with interval arithmetic to obtain enclosures on structural responses 
when probabilistic input descriptions are difficult to validate from limited data.
For fuzzy input and model parameters, M{\"o}ller et al.~\cite{moller2000fuzzy} developed a general fuzzy structural analysis strategy
based on $\alpha$-level optimization: at each membership level, fuzzy inputs are reduced to crisp feasible sets, and extrema of 
the structural response are obtained by optimization around an arbitrary deterministic solution, including geometrically and physically nonlinear finite element algorithms.
This nesting of deterministic solves at successive $\alpha$-levels provides a practical route to fuzzy finite element propagation 
and anticipates later extensions in which uncertain material laws themselves are represented by data-driven or physics-augmented neural network models~\cite{graf2012fuzzy}.

Early work by Freitag et al.~\cite{freitag2011rnn} introduced recurrent neural networks for fuzzy data to identify and 
predict uncertain structural response dependencies from monitoring records represented as fuzzy time series. 
By discretizing membership functions through $\alpha$-cuts and training networks to map deterministic or fuzzy inputs onto fuzzy outputs, 
this model-free approach enables data-driven identification of time-dependent stress-strain behavior from imprecise experimental or 
numerically generated sequences without prescribing a closed-form constitutive law.
More recently, physics-augmented neural networks have also been extended to polymorphic uncertainty
quantification in multiscale hyperelastic analyses~\cite{harazin2026multiscale}.
In that framework, uncertain mesoscale descriptors enter as additional inputs to a hyperelastic 
PANN surrogate of representative volume element response, while the deformation invariants remain 
the mechanical inputs. 
This separates the uncertain-parameter domain from the deformation domain and enables 
efficient pre-training and sampling-based propagation of aleatoric, epistemic, and interval-valued 
quantities without repeated nested finite element solvers.
These developments illustrate that encoding mechanistic structure in the network 
architecture can reduce data demand while keeping uncertainty propagation compatible with downstream 
finite element deployment.

\subsection{\rev{Bayesian baselines under constitutive constraints}}
\label{sec:simple_bayesian_example}

\rev{Bayesian regression places a prior $p(\boldsymbol\theta)$ on model parameters (or directly on functions) and updates it with observations
$\mathcal D$ through $p(\boldsymbol\theta\mid\mathcal D)\propto p(\mathcal D\mid\boldsymbol\theta)p(\boldsymbol\theta)$; predictions then
marginalize the posterior, $p(x_*\mid y_*,\mathcal D)=\int p(x_*\mid y_*,\boldsymbol\theta)p(\boldsymbol\theta\mid\mathcal D)\,d\boldsymbol\theta$
\cite{gelman1995bayesian}. A pointwise credible band therefore summarizes posterior belief under the assumed prior and likelihood,
rather than a mechanically admissible constitutive interval.}

\rev{To illustrate the characteristics of the common Bayesian methods and motivate the present work,} 
\rev{Fig.~\ref{fig:bayesian_sparse_example} considers the convex potential $x=y^2$ on $y\in[-1,1]$, with noisy and sparse observations
available only on the regions $y\in[-0.5,0]$ and $y\in[0.3,0.7]$.
As Bayesian baselines we employ a Gaussian process regression (GPR) with a radial basis function (RBF) kernel with
correlation length $\ell=0.1$~\cite{williams2006gaussian},
together with Gaussian observation noise of zero mean and standard deviation $0.03$, and a Bayesian polynomial regression of degree~5
on the same data. Owing to the short correlation length, the $95\%$ credible band of the GPR remains narrow near the observed intervals
and widens in data-scarce regions (Fig.~\ref{fig:bayesian_sparse_example}a); the degree-5 Bayesian polynomial regression exhibits the
same qualitative behavior (Fig.~\ref{fig:bayesian_sparse_example}b).}

\rev{Interpreting these pointwise limits directly as lower and upper bounds of the potential to be individually deployed for uncertainty
propagation, it is clear that these potentials do not preserve by construction convexity, growth, reference-state, or other properties
that might be necessary for robust forward deployment (this would be the case of utilizing a free energy density function in FEM
\cite{ball,linden2023guide}). Even though physics-constrained and ML-enabled Bayesian formulations are possible, the key point is
that a credible band alone does not provide these constitutive guarantees for uncertainty propagation.}

\begin{figure}[!ht]
    \centering
    \begin{subfigure}[t]{0.49\textwidth}
        \centering
        \maybeincludegraphics[width=\linewidth]{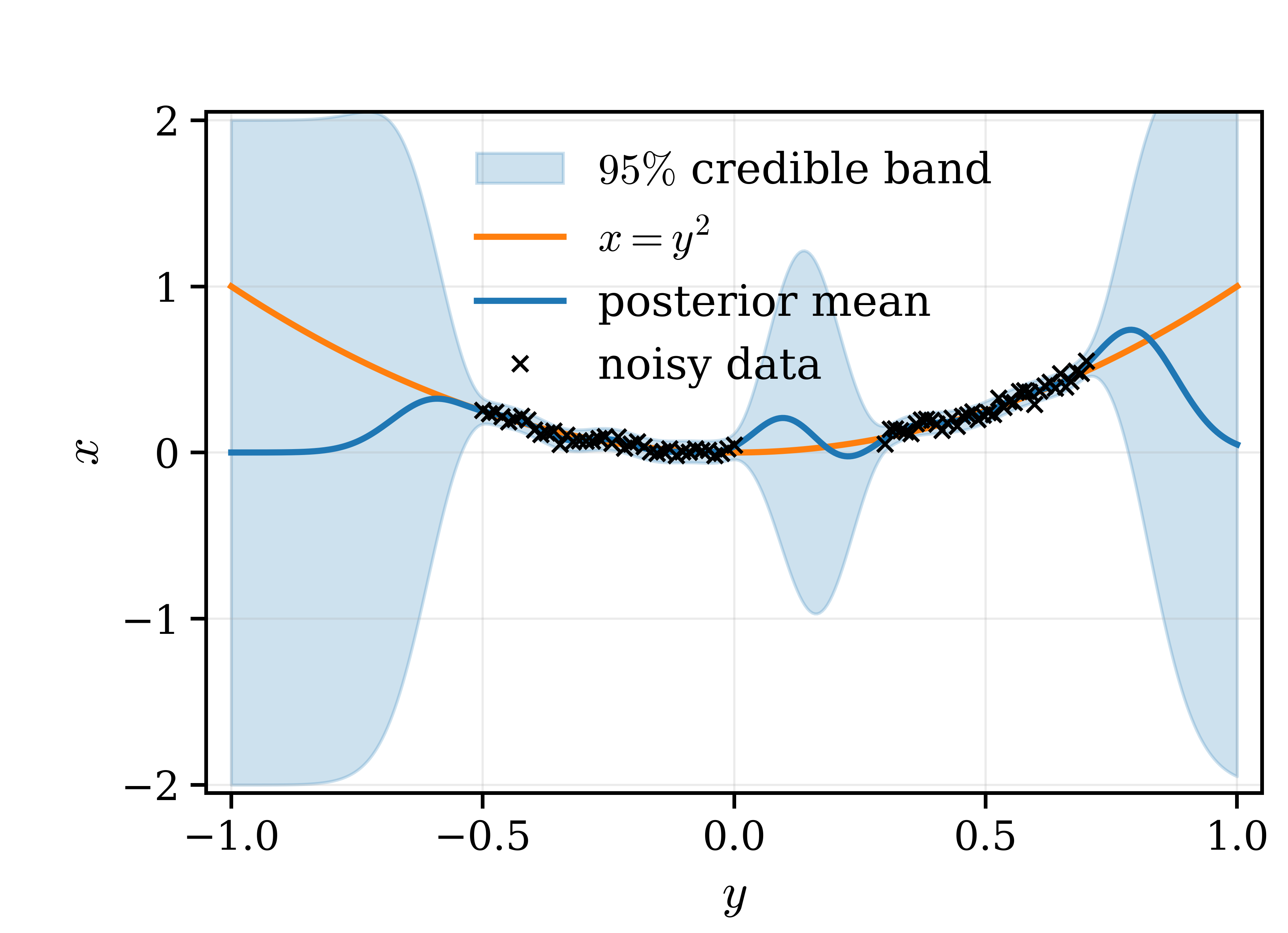}
        \caption{}
    \end{subfigure}\hfill
    \begin{subfigure}[t]{0.49\textwidth}
        \centering
        \maybeincludegraphics[width=\linewidth]{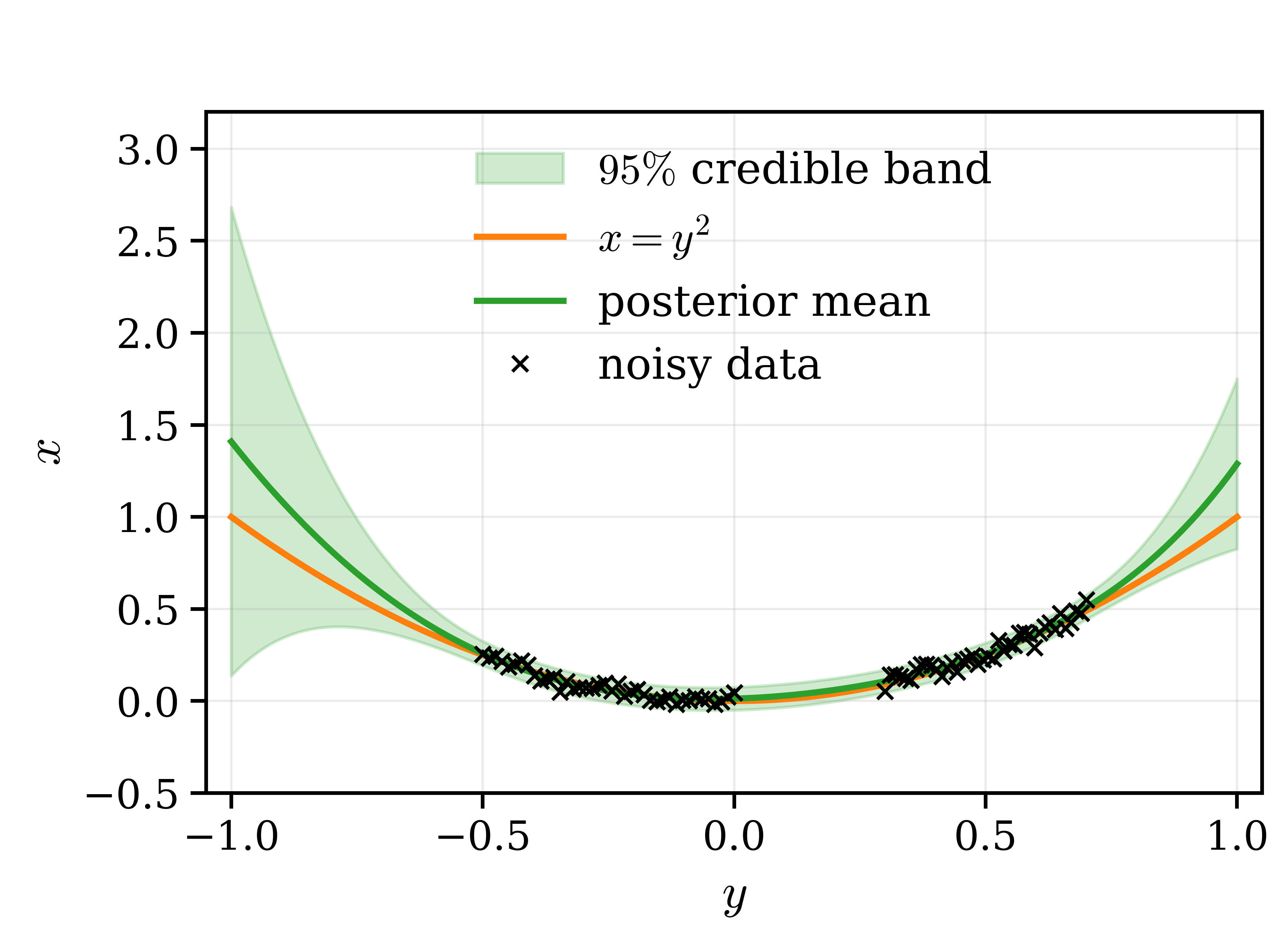}
        \caption{}
    \end{subfigure}
    \caption{\rev{Bayesian regression for the convex potential $x=y^2$ using a sparse and noisy dataset:
    (a) GPR with an RBF kernel with correlation length $\ell=0.1$ and (b) Bayesian polynomial regression of degree~5.}}
    \label{fig:bayesian_sparse_example}
\end{figure}

\subsection{Contributions}
In this work, we propose a framework for uncertainty quantification on isotropic
hyperelastic constitutive responses from noisy stress-deformation data. 
We first construct \emph{interval physics-augmented neural networks} (iPANNs), which extend physics-augmented learning and sparse input-convex representations~\cite{linka,fuhg2024extreme,yang2025} to learn separate lower, mean, and upper free energy density branches together with bound-aware losses that certify stress enclosures of the observables.
\rev{The three branches are linked by the interval and principal-stress ordering conditions introduced 
in Section~\ref{sec:interval_hyperelasticity}; the bound ICNNs inherit the weights of the trained mean model 
which are subsequently fine-tuned to enclose the noisy stress observations (Sections~\ref{sec:loss} and~\ref{sec:three_step}).}
Building on these iPANN endpoints, we then introduce \emph{fuzzy physics-augmented neural networks} (fPANNs), which interpolate the three learned branches through convex $\alpha$-cut constructions~\cite{fuzzy3_fuhg2022,fem1_fish2007}.
Whereas iPANNs provide a deterministic interval model anchored at the extreme lower and upper bounds, fPANNs furnish a graded, membership-based description that allows less conservative enclosures when a practitioner specifies an intermediate $\alpha$.
Both representations retain compact analytic forms and can be incorporated seamlessly in a finite element setting for uncertainty propagation. 
We evaluate our proposed approach on synthetic stress--deformation data corrupted by 
multiplicative heteroscedastic noise as a baseline case and controlled perturbations of 
the noise seed, mean, and standard deviation on top of the baseline case.

\rev{The specific contributions of this work are:}
\begin{enumerate}
\item \rev{an interval physics-augmented neural network (iPANN) formulation in which lower, mean, and upper
isotropic hyperelastic free energy densities are represented by sparse input-convex networks and trained with
bound-aware losses, so that the induced second Piola--Kirchhoff stresses enclose noisy observations by construction;}
\item \rev{interval and principal-stress ordering conditions that make the three learned energy branches mutually
consistent, together with a two-stage warm-start procedure in which the bound networks inherit the trained mean weights
before being fine-tuned;}
\item \rev{a fuzzy physics-augmented neural network (fPANN) construction that interpolates the learned branches through
convex $\alpha$-cuts, yielding graded and less conservative enclosures at intermediate membership levels;}
\item \rev{closed-form interval and fuzzy constitutive expressions that are deployed directly in a finite element
solver to propagate the learned bounds to structural-level quantities of interest.}
\end{enumerate}

\subsection{Outline}
The remainder of this paper is organized as follows. Section~2 introduces the theoretical and computational preliminaries required for the proposed framework, including interval analysis, fuzzy set theory and $\alpha$-cuts, hyperelastic constitutive modeling, and interval concepts in hyperelasticity. Section~3 presents the proposed fuzzy physics-augmented neural network framework for isotropic hyperelasticity, starting from preliminaries for input-convex neural networks and smoothed $L_0$ sparsification. We first formulate interval-valued free energy density functions and the associated stress enclosures, then introduce interval physics-augmented neural networks for learning lower and upper constitutive bounds, and finally construct fuzzy constitutive representations through membership-based interpolation of the learned energy branches. Section~4 describes the numerical experiments used to illustrate and evaluate the approach, including the synthetic data-generation procedure, the two-stage transfer-learning strategy, stress and energy-bound predictions under different noise models, and the resulting $\alpha$-cut membership analysis. Section~5 concludes the paper by summarizing the main findings and discussing the implications of the proposed framework for uncertainty-aware constitutive modeling and downstream finite element simulations for uncertainty propagation.

\section{Preliminaries}\label{sec:FITheory}
\subsection{Interval analysis}\label{sec:interval_theory}

An interval $\mathbf{x}$ is defined by its lower and upper bounds,
\begin{equation}
    \mathbf{x} = [\underline{x}, \overline{x}], 
    \qquad \underline{x} \le \overline{x},
\end{equation}
and represents the set of all real numbers $x$ such that $x \in [\underline{x}, \overline{x}]$.
For two intervals $\mathbf{x} = [\underline{x}, \overline{x}]$ and $\mathbf{y} = [\underline{y}, \overline{y}]$, 
interval addition is defined as
\begin{equation}
    \mathbf{x} + \mathbf{y} 
    = [\underline{x} + \underline{y},\; \overline{x} + \overline{y}],
\end{equation}
Scalar multiplication by a real number $\alpha\in\mathbb{R}$ is defined as:
\begin{equation}
    \alpha \mathbf{x}
    =
    \begin{cases}
        [\alpha \underline{x},\; \alpha \overline{x}], & \alpha \ge 0,\\[1mm]
        [\alpha \overline{x},\; \alpha \underline{x}], & \alpha < 0.
    \end{cases}
\end{equation}
Furthermore, if a function $g$ is a strictly increasing function, then its interval extension can be 
obtained by evaluating $g$ at the interval endpoints.

Using this interval arithmetic,  one can construct the output from a neural network layer using 
\begin{equation}\label{inn_neuron}
  y = g(wx + b)
\end{equation}
for a scalar neuron with weight $w$, bias $b$, and activation function $g$~\cite{inta1_alefeld2011,inta2_neumaier1990, garczarczyk2000interval}.

\subsection{Fuzzy set theory and $\alpha$-cuts}\label{sec:fuzzy_theory}

Fuzzy set theory provides a further generalization of interval analysis by relaxing the binary notion of set membership.
Instead of declaring whether a value belongs to a set or not, a \emph{membership function} $\mu$ assigns each element $x\in X$ a degree of membership between 0 and 1~\cite{fuzzy1_klir1995,fuzzy2_zimmermann2011}:
\begin{equation}
    \mu:X\to[0,1],
    \qquad x\mapsto\mu(x).
\end{equation}
When $\mu(x)=1$, the element $x$ has full membership in the fuzzy set, whereas $\mu(x)=0$ indicates no membership. A fuzzy set on the domain $X$ can therefore be represented by the pair $(X,\mu)$.
This framework is well suited to describing ``informal'' or linguistic uncertainty, for which only qualitative statements such as ``likely'', ``typical'', or ``extreme'' are available.

The link between fuzzy sets and interval analysis is made through the concept of \emph{$\alpha$-cuts}.
For a given level $\alpha\in(0,1]$, the $\alpha$-cut of the fuzzy set $(X,\mu)$ is defined as
\begin{equation}
    X_f^\alpha=\{x\in X\mid\mu(x)\ge\alpha\}.
\end{equation}
If the fuzzy set is convex, each $\alpha$-cut reduces to a closed interval
\begin{equation}
    X_f^\alpha=[X_L^{\alpha},X_U^{\alpha}],
\end{equation}
where $X_L^{\alpha}$ and $X_U^{\alpha}$ denote the lower and upper bounds of the $\alpha$-cut, respectively.
The full fuzzy set can be reconstructed from a nested family of intervals
\begin{equation}
   \{[X_L^\alpha,X_U^\alpha]\}\quad\forall\,\alpha\in(0,1].
\end{equation}
In many applications one works with a small number of representative $\alpha$-levels. A common choice is a triangular membership function whose peak ($\alpha=1$) corresponds to the most likely (or nominal) state and whose base ($\alpha\to 0^+$) spans the full range of admissible uncertainty~\cite{fuzzy3_fuhg2022}.

\subsection{Preliminaries for Hyperelasticity}\label{sec:hyperelasticity}

For a body occupying the reference configuration $\Omega_0$, kinematics define its motion from the reference configuration to the current configuration $\Omega$. Let $\mathbf{X} \in \Omega_0$ denote the material position vector and $\mathbf{x} \in \Omega$ its spatial counterpart. The deformation gradient tensor is then defined as $\mathbf{F}(\mathbf{X}) = \partial \mathbf{x} / \partial\mathbf{X}$,
which locally characterizes the mapping of material line elements from the reference configuration to their deformed counterparts in the current configuration.

A hyperelastic material possesses a Helmholtz free energy density function $\Psi$, defined per unit reference volume.
If $\Psi$ depends only on the deformation gradient tensor $\mathbf{F}$ or on some derived strain tensor, then 
the Helmholtz free energy function is called a free energy density. For homogeneous hyperelastic materials, 
the free energy density depends only upon the
deformation gradient tensor $\mathbf{F}$, i.e.\ $\Psi = \Psi(\mathbf{F})$. 
The associated response function of a hyperelastic material is of the form:
\begin{equation}
    \mathbf{P} = \frac{\partial \Psi}{\partial \mathbf{F}}
\end{equation}
where $\mathbf{P}$ denotes the first Piola–Kirchhoff stress tensor.
The following are some common assumptions and restrictions  on the free energy density $\Psi$: 
(a) $\Psi$ is a continuous function of $\mathbf{F}$ that vanishes in the reference configuration, $\Psi(\mathbf{I}) = 0$ (\rev{here $\mathbf{I}$ denotes the second-order
 identity tensor})
 and increases with deformation, $\Psi(\mathbf{F}) \geq 0$; (b) $\Psi$ has no other stationary points in strain space, and satisfies the growth conditions i.e.
  $\Psi(\mathbf{F}) \to \infty$ as $\det(\mathbf{F}) \to 0^+$ or $\det(\mathbf{F})\ \to \infty$ and (c)
\rev{$\Psi$ is objective: writing $\mathbf{x}=\chi(\mathbf{X},t)$ for the motion of the body, one has $\Psi(\mathbf{F})=\Psi(\mathbf{QF})$ for every
 proper orthogonal tensor $\mathbf{Q}$.} If the material is isotropic, one can combine objectivity with polar decomposition to 
show that the scalar-valued isotropic tensor function $\Psi$ is invariant with the right Cauchy-Green tensor $\mathbf{C} = \mathbf{F}^T \mathbf{F}$. Using a representation theorem for invariants, 
one can express this $\Psi$ in terms of principal invariants $I_1, I_2, I_3$ of $\mathbf{C}$~\cite{truesdell_noll,holzapfel}\rev{, defined by
\begin{equation}
I_1=\operatorname{tr}\mathbf{C},\qquad
I_2=\tfrac{1}{2}\bigl((\operatorname{tr}\mathbf{C})^2-\operatorname{tr}(\mathbf{C}^2)\bigr),\qquad
I_3=\det\mathbf{C},
\label{eq:principal_invariants}
\end{equation}
and we write $J=\sqrt{I_3}=\det\mathbf{F}$, where $\det\mathbf{F}$ denotes the determinant of $\mathbf{F}$.}

\subsection{Interval Hyperelasticity and Stress Enclosure}\label{sec:interval_hyperelasticity}
It is important to note that experimental observables are mostly deformation- and force-like quantities (with stresses and strains being tensor-valued), but hyperelasticity is based on assuming a scalar-valued free energy density function. As such, it is convenient to develop an interval representation in terms of the free energy density function, but the difficulty lies in connecting it to the stress observables and the corresponding stress enclosure. Effectively, the interval representation is developed for an unseen quantity.

We introduce an \emph{interval-valued free energy density representation}:
\begin{equation}
\mathbf{\Psi}(\mathbf{F}) = [\underline{\Psi}(\mathbf{F}), \overline{\Psi}(\mathbf{F})]\label{eq:PsiInterval}
\end{equation}
where $\mathbf{\Psi}(\mathbf{F})$ denotes the set of admissible free energy densities at deformation $\mathbf{F}$.
\rev{The endpoints $\underline{\Psi}(\mathbf{F})$ and $\overline{\Psi}(\mathbf{F})$ are not arbitrary functions of 
$\mathbf{F}$: each is an isotropic hyperelastic free energy density, expressible in terms of the eigenvalues of the strain as $\widehat{\Psi}(E_1,E_2,E_3)$ or, 
equivalently, $\Psi(I_1,I_2,J)$, with associated principal second Piola--Kirchhoff stresses obtained from Eq.~\eqref{s_lb_ub2}.
The deformation $\mathbf{F}$ is deterministic in this formulation; 
uncertainty is carried by the interval on $\Psi$ and the induced stress enclosure.}

The interval structure on $\Psi$ induces a corresponding \emph{stress enclosure} for the second Piola-Kirchhoff stress tensor:
\begin{equation}
\mathbf{S}(\mathbf{F}) \in \mathcal{S}(\mathbf{F})
:=
\left\{
2 \frac{\partial \mathbf{\Psi} (\mathbf{F})}{\partial \mathbf{C}} 
\;\mid\;
\mathbf{\Psi} = [\underline{\Psi}, \overline{\Psi}]
\right\},
\end{equation}
\rev{where $\mathcal{S}(\mathbf{F})$ denotes the set of admissible second Piola--Kirchhoff stresses induced by the energy interval at deformation $\mathbf{F}$.}

The second Piola-Kirchhoff stress corresponding to $\underline{\Psi}$ and $\overline{\Psi}$  
is an interval quantity,
\begin{equation}
\underline{\mathbf{S}}=2\frac{\partial \underline{\Psi}}{\partial \mathbf{C}},
\qquad
\overline{\mathbf{S}}=2\frac{\partial \overline{\Psi}}{\partial \mathbf{C}}.
\label{s_lb_ub2}
\end{equation}
The overline and underline for the stress tensor do not yet refer to an
ordering for the stress tensor itself, but for now just indicate that the stress tensor
 is defined corresponding to the lower and upper free energy density branches.

The Green Lagrange strain tensor is defined as $\mathbf{E}=\frac{1}{2}(\mathbf{C}-\mathbf{I})$, has 
principal values $E_i=\frac{1}{2}(\lambda_i^2-1)$, $i=1,2,3$  (in terms of the principal stretches $\lambda_i$)and 
\begin{equation}
\bm{E} = \sum_{i=1}^{3} E_i \bm{A}_i \otimes \bm{A}_i
\end{equation}
is its spectral decomposition, where 
\(\bm{A}_i\) are the corresponding orthonormal eigenvectors. For an isotropic
hyperelastic material, the stored energy may be written as $\widehat{\Psi}(E_1,E_2,E_3)$.
Moreover, the second Piola stress tensor $\bm{S}$ is coaxial with \(\bm{E}\). Hence
\begin{equation}
\bm{S}
=
\frac{\partial \Psi}{\partial \bm{E}}
=
\sum_{i=1}^{3} S_i \bm{A}_i \otimes \bm{A}_i,
\end{equation}
where
\begin{equation} \label{eq:Sprin}
S_i = \frac{\partial \widehat{\Psi}}{\partial E_i}.
\end{equation}
\rev{In the numerical results of Section~\ref{sec:results}, where $\mathbf{C}$ and $\mathbf{S}$ are diagonal in the canonical 
principal basis, the diagonal components $S_{ii}$ reported there correspond to the principal values $S_i$ ($i=1,2,3$) defined in Eq. \ref{eq:Sprin}.}

We assume that the reference configuration $\mathbf{E}=\bm{0}$ is stress-free for the lower-bound,
mean, and upper-bound responses. That is,
\begin{equation}
    \underline{\mathbf{S}}(\bm{0})=\mathbf{S}(\bm{0})=\overline{\mathbf{S}}(\bm{0})=\bm{0}
\end{equation}
underlying the fact that, in this construction, the stress enclosure collapses at that deformation state.
Equivalently, in principal stress and strain notation,
\begin{equation}
\underline{S}_i(\{ E_j=0 \})=S_i(\{ E_j=0 \})=\overline{S}_i(\{ E_j=0 \})=0,\qquad i,j=1,2,3.
\end{equation}
We also choose arbitrary additive constants in the strain-energy densities
so that all three energies vanish at the reference configuration:
\begin{equation}
    \underline{\Psi}(\bm{0})=\Psi(\bm{0})=\overline{\Psi}(\bm{0})=0.
    \label{eq::energydatum}
\end{equation}
This energy normalization is independent of the stress-free reference
assumption, since stresses are derivatives of the energy and are unchanged by
additive constants.
\rev{These reference-state conditions apply identically to $\underline{\Psi}$, $\Psi$, and $\overline{\Psi}$. 
At any admissible strain state $\mathbf{E}$, the three branches must further satisfy the principal-stress ordering of
 Eq.~\eqref{eq:SiOrdering}; when this holds, Proposition~1 guarantees $\underline{\Psi}(\mathbf{E})\le\Psi(\mathbf{E})\le\overline{\Psi}(\mathbf{E})$. 
 Thus, within the interval of Eq.~\eqref{eq:PsiInterval}, $\underline{\Psi}$ and $\overline{\Psi}$ act as lower and upper bounding energy surfaces,
  and $\Psi$ as the mean response between them.}

\begin{proposition}
In the context of isotropic hyperelasticity, let $\underline{\Psi},\Psi,\overline{\Psi}$ be 
free energy density functions with associated principal second Piola-Kirchhoff 
stresses $\underline{S}_i,S_i,\overline{S}_i$. 
For any given admissible strain state $\mathbf{E}$, if the following ordering holds for the principal stresses
\begin{equation}
\operatorname{sign}(E_i)\underline{S}_i\le \operatorname{sign}(E_i)S_i\le \operatorname{sign}(E_i)\overline{S}_i, \,i=1,2,3, \, (\text{no}\,\operatorname{sum})\label{eq:SiOrdering}
\end{equation}
where $\operatorname{sign}(\,)$ passes the sign of the scalar quantity in parentheses, then the following ordering holds for the free energy density branches
\begin{equation}
\underline{\Psi}(\mathbf{E})\le \Psi(\mathbf{E})\le \overline{\Psi}(\mathbf{E}).\label{eq:PsiOrdering}
\end{equation}
\end{proposition}

\begin{proof}
Consider a loading path from the zero stress reference configuration $\mathbf{E}=\mathbf{0}$ to any strain state $\mathbf{E}$, following a ray in strain space as 
\begin{equation}
\hat{\bm{E}}(\bm{\gamma})=\gamma \bm{E},
\qquad
\gamma \in [0,1],
\end{equation}
with $\gamma$ monotonically non-decreasing. The increment of the principal strain along the ray is
\begin{equation}
    d_{\gamma}E_i=\frac{d E_i}{d \gamma} d\gamma
\end{equation}
Along this ray, the following holds for the principal strain values
and their corresponding increments
\begin{equation}
 \mathrm{sign}(d_{\gamma}E_i)=\mathrm{sign}(E_i).
\end{equation}

The corresponding increment of the free energy density along the ray is
\begin{equation}
d_{\gamma}\Psi =  \sum_{i=1}^3 S_i\, d_{\gamma}E_i,
\end{equation}
and similarly for $\underline{\Psi}$ and $\overline{\Psi}$.
Assuming that $\operatorname{sign}(E_i)\underline{S}_i\le \operatorname{sign}(E_i)S_i\le \operatorname{sign}(E_i)\overline{S}_i$,  $i=1,2,3$ $(\text{no}\,\operatorname{sum})$ for any admissible strain state $\bm{E}$, two cases follow:

i) If $\lambda_i \ge 1$, then $E_i \ge 0$ and, by assumption, $0 \le \underline{S}_i \le S_i \le \overline{S}_i$ with $d_{\gamma}E_i \ge 0$, so
\begin{equation}
\underline{S}_i\, d_{\gamma}E_i \le S_i\, d_{\gamma}E_i \le \overline{S}_i\, d_{\gamma}E_i. \label{eq::workineq1}
\end{equation}

ii) If $\lambda_i \le 1$, then $E_i \le 0$ and $\overline{S}_i \le S_i \le \underline{S}_i \le 0$ with $d_{\gamma}E_i \le 0$, which again yields
\begin{equation}
\underline{S}_i\, d_{\gamma}E_i \le S_i\, d_{\gamma}E_i \le \overline{S}_i\, d_{\gamma}E_i. \label{eq::workineq2}
\end{equation}
As the inequalities from Eqs.~\eqref{eq::workineq1} and~\eqref{eq::workineq2} align, summing over $i$ for any admissible path gives
\begin{equation}
\label{eq:diffinequalities}
d_{\gamma}\underline{\Psi} \le d_{\gamma}\Psi \le d_{\gamma}\overline{\Psi}.
\end{equation}
Integration along the path from $\gamma=0$ (at the reference $\mathbf{E}=\mathbf{0}$) to $\gamma=1$,  and invoking path independence yields
\begin{equation}
\underline{\Psi}(\bm{E})-\underline{\Psi}(\bm{0})
\le
\Psi(\bm{E})-\Psi(\bm{0})
\le
\overline{\Psi}(\bm{E})-\overline{\Psi}(\bm{0}).
\end{equation}
As a common energy datum was assumed at the reference state per Eq.~\eqref{eq::energydatum}, then the energy ordering is retrieved
\begin{equation}
\underline{\Psi}(\mathbf{E})\le \Psi(\mathbf{E})\le \overline{\Psi}(\mathbf{E}).
\end{equation}
\end{proof}

\rev{The proposition and proof are central to the scope and feasibility of the present work and illuminate why extension to anisotropic or inelastic materials is non-trivial.}
\rev{Proposition~1 establishes the implication \emph{principal-stress ordering} $\Rightarrow$ \emph{energy ordering}, not the converse.
Pointwise ordering of $\underline{\Psi}$, $\Psi$, and $\overline{\Psi}$ at a fixed strain state does not, in general, imply ordering of 
the corresponding principal stresses.
In the present framework, the reverse implication is therefore neither assumed nor required: stress enclosures on the noisy observations
are enforced directly through the bound-aware loss (Section~\ref{sec:loss}), while Proposition~1 guarantees that, whenever the learned 
models satisfy Eq.~\eqref{eq:SiOrdering}, the three energy branches define a consistent hyperelastic interval with the common datum of 
Eq.~\eqref{eq::energydatum}.} 
\rev{Furthermore, this establishes a \emph{foliation} of strain space where $\underline{\Psi}(\mathbf{E})$, $\Psi(\mathbf{E})$, 
and $\overline{\Psi}(\mathbf{E})$ are the \emph{leaves} (the dividing surfaces) \cite{candel2000foliations}.}

The differential inequalities in Eq.~\eqref{eq:diffinequalities} are established along a particular path in strain space, but the resulting energy bounds follow from path independence for any path. Since the response is hyperelastic, $d\Psi=\mathbf{S}:d\mathbf{E}$ is an exact differential. Thus, to bound the energy at a given state $\mathbf{E}$, it is sufficient to choose any convenient admissible path from the reference state to $\mathbf{E}$. 
Along this comparison path, the stress--strain sign assumptions preserve the ordering of the principal incremental work terms. Integrating along this path then yields bounds on $\Psi(\mathbf{E})$, which hold for the endpoint state by path independence.
As such, an interval enclosure for the principal values of the second Piola-Kirchhoff stress tensor is defined as:
\begin{equation}
S_i \in \mathcal{S}_i
= [\operatorname{sign}(E_i)\underline{S}_i,\operatorname{sign}(E_i)\overline{S}_i], \quad i=1,2,3, \quad \text{no}\,\operatorname{sum}\label{eq:SInterval}
\end{equation}
following the interval of the free energy densities (with a common datum), as defined in Eq.~\eqref{eq:PsiInterval}. 

It is important to establish the function of the stress ordering in Eq.~\eqref{eq:SInterval} in this interval construction. In this setting, 
the stresses are noisy observables paired with a corresponding strain state. As such, Eq.~\eqref{eq:SInterval} establishes that for every 
principal second Piola-Kirchhoff stress the upper and lower bounds of the interval enclosure flip as the corresponding principal 
Green-Lagrange strain changes sign\footnote{For positive $E_i$ the upper bound $\overline{S}_i$ bounds the data from above, whereas 
for negative $E_i$ it bounds the data from below, and the dual statement holds for the lower bound $\underline{S}_i$.}. 
The corresponding free energy ordering of Eq.~\eqref{eq:PsiOrdering} enables the consistent interval construction of Eq.~\eqref{eq:PsiInterval} 
 that matches the ordering of stress observations. Matching interval bounds of observed stresses to free energy density interval bounds is 
 central to obtaining an interval hyperelasticity formulation that learns the uncertainty of the free energy and propagates it to the observable, 
 which in this case is the stress-strain pair.
Furthermore, we propose a data-driven scheme for learning the interval free energy density and its extension to the fuzzy setting

\section{ Interval and Fuzzy Physics Augmented Neural Networks for Hyperelasticity}\label{sec:kinematics_thermo}
Towards an ML-enabled construction for interval and fuzzy hyperelasticity, this section summarizes constructions for enhancing implicit biases for structure and sparsity of neural networks. In the following, Interval and Fuzzy Physics Augmented Neural Networks (fPANN and iPANN) are introduced.

\subsection{Polyconvexity and Input Convex Neural Networks (ICNN)}\label{sec:icnn}

Polyconvexity is a sufficient (not necessary) condition on the free energy density function $\Psi(\mathbf{F})$  which guarantees weak lower semicontinuity of the energy functional and thus the existence of minimizers under appropriate growth conditions.
The potential $\Psi(\mathbf{F})$ is polyconvex if it can be written as a convex function of $\mathbf{F}$ and all of its minors~\cite{ball}.
Polyconvexity is a convenient restriction to pose so that the material exhibits a stable response, thus avoiding the issues that might manifest during finite element analysis~\cite{tan2026towards}. 
ICNNs are a class of neural networks that are guaranteed to be convex in their inputs by construction~\cite{icnn_amos2017}.
In the context of isotropic hyperelasticity, ICNNs can represent a potential which is convex in the inputs $I_1$, $I_2$ and $J$. Achieving polyconvexity in an NN setting requires additional structural considerations beyond input convexity.
\rev{For $\Psi(\mathbf{F})$ to be polyconvex~\cite{ball}, it must be expressible as a convex function of $(\mathbf{F},\operatorname{cof}\mathbf{F},\det\mathbf{F})$.
Ball's sufficient criterion \cite{ball} gives polyconvexity of $\Psi(\mathbf{F})=\widehat{\Psi}(I_1,I_2,J)$ whenever $\widehat{\Psi}$ is jointly convex \emph{and} non-decreasing in $I_1$ and $I_2$; no monotonicity is required in $J$.
In the ICNN of Eq.~\eqref{icnn_formulation}, convexity of $f$ in $\bm{y}=(I_1,I_2,J)$ already requires $\bm{W}^{(z)}_{i}\ge 0$ and convex non-decreasing activations $g_i$; polyconvexity further requires the pass-through weights $\bm{W}^{(y)}_i$ acting on $I_1$ and $I_2$ to be non-negative~\cite{linden2023guide,pcnn_klein2022}.
Our networks enforce only the convexity part: the pass-through weights are unconstrained, so the learned potentials are convex but not monotone in $I_1$ and $I_2$, and we accordingly do not claim polyconvexity.
This is deliberate: the Gent--Gent model used to generate the data is itself non-polyconvex.}
More extensive discussions are included in several works regarding the translation of this condition in an
NN setting~\cite{pcnn_klein2022,mnn_klein2025,schroder2003invariant}.

For a $k$-layer fully connected ICNN, the architecture is defined as
\begin{equation}
\label{icnn_formulation}
\bm{z}_{i+1}
=
g_i\left(
\bm{W}^{(z)}_i\bm{z}_i
+
\bm{W}^{(y)}_i\bm{y}
+
\bm{b}_i
\right),
\qquad
f(\bm{y};\bm{\theta})=\bm{z}_k,
\end{equation}
for $i=0,\ldots,k-1$, where $\bm{z}_i$ denotes the layer  activations (with $\bm{z}_0,\bm{W}^{(z)}_0\equiv 0$), and
$\bm{\theta}=\{\bm{W}^{(y)}_{0:k-1},\bm{W}^{(z)}_{1:k-1},\bm{b}_{0:k-1}\}$
denotes the set of trainable parameters. The function $f$ is convex in $\bm{y}$ provided that the weights $\bm{W}^{(z)}_{1:k-1}$ are constrained to be non-negative and the activation functions $g_i$ are convex and non-decreasing.

Pass-through layers, which directly connect the input $y$ to hidden units in deeper layers via the $\bm{W}^{(y)}_i$ terms, are 
essential for ICNNs. In standard feedforward networks, previous hidden units can always be mapped to subsequent layers with 
identity mappings. However, for ICNNs, the non-negativity constraint on $\bm{W}^{(z)}$ weights restricts the use of hidden units 
that mirror the identity mapping. Passthrough layers explicitly provide direct connections from input to deeper layers, 
enabling the network to learn more expressive convex functions~\cite{icnn_amos2017}.

\subsection{Smoothed $L_0$ sparsification}\label{sec:l0}

Regularization is commonly used to improve the robustness of neural network models and mitigate overfitting. A widely used choice is $L_2$ regularization, which penalizes the squared magnitude of the parameters, thereby suppressing large parameter values and encouraging smoother learned responses. However, because this penalty continuously shrinks parameters rather than driving them exactly to zero, 
it does not typically produce sparse models. While $L_1$ regularization is often used to encourage sparsity by 
penalizing the absolute magnitude of the parameters, it penalizes parameter magnitudes rather than explicitly 
counting the number of nonzero parameters, as in $L_0$ regularization. For neural network-based constitutive 
laws, explicit sparsification is important because a large number of active trainable parameters can reduce 
the interpretability of the learned material response, even when the model achieves high predictive accuracy.
We therefore employ a smoothed $L_0$ regularization strategy to obtain a compact representation with fewer
active parameters while preserving predictive accuracy~\cite{louizos2017learning,mcculloch2024sparse,fuhg2024extreme}. 
Because the exact $L_0$ penalty is non-differentiable, a differentiable relaxation is required for gradient-based training.

We adopt the hard-concrete gating framework of Louizos et al.~\cite{louizos2017learning},
where each trainable parameter is re-parameterized as
\begin{equation}
\bm{\theta}=\bar{\bm{\theta}}\odot \bm{\mathsf{z}},
\end{equation}
\rev{where $\bar{\bm{\theta}}$ denotes the ungated trainable network parameters and
$\bm{\mathsf{z}}$ is a gate vector of matching dimension.} 
Here, $\odot$ denotes the element-wise Hadamard product.
A stochastic realization of the hard-concrete gate is obtained as
\begin{equation} 
\bm{\mathsf{z}} = \min\!\bigl(\bm{1},\max(\bm{0},\overline{\bm{s}})\bigr), \qquad \overline{\bm{s}} = \bm{s}(\zeta-\gamma)+\gamma\bm{1}, 
\end{equation} 
with 
\begin{equation} 
\bm{s} = \operatorname{sig}\!\left( \frac{\log\bm{u}-\log(1-\bm{u})+\log\bm{\alpha}}{\beta} \right). 
\end{equation}

\rev{Here, $\bm{u}\sim\mathcal{U}(0,1)$ is an auxiliary random variable used element-wise to generate a stochastic gate  at the start of training. Also, $\log\bm{\alpha}$ denotes trainable gate parameters. These parameters are randomly initialized once according to $\log\bm{\alpha}\sim\mathcal{N}(0,\sigma)$, with $\sigma=0.01$, and are subsequently optimized together with $\bar{\bm{\theta}}$. In the present work, we employ a hybrid gating strategy. Stochastic hard-concrete sampling is used only for the first training forward pass. For all subsequent training iterations and during inference, the gates are evaluated deterministically as follows}
\begin{equation}
\widehat{\bm{\mathsf{z}}}=\min\!\left(\bm{1},\max\!\left[\bm{0},\operatorname{sig}(\log\bm{\alpha})(\zeta-\gamma)+\gamma\bm{1}\right]\right).\label{eq:deterministic_gate}
\end{equation}
\rev{Consequently, after the initial forward pass the effective parameters $\bm{\theta}=\bar{\bm{\theta}}\odot\widehat{\bm{\mathsf{z}}}$ are deterministic.}
The resulting sparsity-regularized training objective is
\begin{equation}
\mathcal{L}_{\mathrm{total}}(\bar{\bm{\theta}},\bm{\alpha})= \rev{\mathcal{L}_{\mathrm{mse}}\!\left(\bar{\bm{\theta}}\odot\widetilde{\bm{\mathsf{z}}}\right)}+\lambda_{0}\sum_{j=1}^{|\bm{\theta}|}\operatorname{sig}\!\left(\log\alpha_j-\beta\log\frac{-\gamma}{\zeta}\right),
\label{eq:l0}
\end{equation}
\rev{where $\widetilde{\bm{\mathsf{z}}}$ denotes the stochastic gate $\bm{\mathsf{z}}$ in the first forward pass and the deterministic gate $\widehat{\bm{\mathsf{z}}}$ thereafter.} The second term is the differentiable approximation of the expected $L_0$ norm and penalizes the expected number of active gates, while $\lambda_0$ controls the sparsification strength. \rev{The data-fitting term $\mathcal{L}_{\mathrm{mse}}$ is the stress-based mean squared error defined in Section~\ref{sec:loss}.} The constants are set to $\gamma=-0.1$, $\zeta=1.1$, and $\beta=\frac{2}{3}$, following~\cite{fuhg2024extreme,padmanabha2024improving}. 
The deterministic gate evaluation after the initial stochastic pass reduces sampling-induced fluctuations during optimization and improves training stability while retaining the sparsity-promoting $L_0$ penalty. \rev{Further details and applications of this sparsification strategy to physics-augmented constitutive models are provided in~\cite{fuhg2024extreme,padmanabha2024improving,yang2025}.}

\subsection{iPANNs}\label{sec:loss}
iPANNs are constructed to learn free energy density bounds from limited noisy stress-strain observations and aim to propagate the uncertainty bounds in unseen loading states. 
Starting from the definition of the free energy density interval of Eq.~\eqref{eq:PsiInterval} and 
principal stress interval of Eq.~\eqref{eq:SInterval}, and the consistent physics augmentation of 
neural networks for hyperelasticity as discussed in Section~\ref{sec:icnn}, we introduce 
iPANNs in the following as an extension of interval neural networks~\cite{garczarczyk2000interval}.
iPANNs are networks that, given a deformation state, can predict the free energy 
density interval of Eq.~\eqref{eq:PsiInterval} and the corresponding principal stress interval of Eq.~\eqref{eq:SInterval}.
The learned interval-valued free energy density representation is denoted as:
\begin{equation}
\Psi^*(\mathbf{F}) \in \mathbf{\Psi}^*(\mathbf{F}) = [\underline{\Psi}^*(\mathbf{F}), \overline{\Psi}^*(\mathbf{F})]\label{eq:learnedPsiInterval}
\end{equation}
where $(\,)^*$ denotes a learned quantity. The corresponding enclosure for the principal values of 
the second Piola-Kirchhoff stress tensor can be calculated through automatic differentiation~\cite{paszke2017automatic}:
\begin{equation}
S^*_i \in \mathcal{S}^*_i
= [\operatorname{sign}(E_i)\underline{S}^*_i,\operatorname{sign}(E_i)\overline{S}^*_i], \quad i=1,2,3.. \quad (\text{no\,sum})\label{eq:learnedSInterval}
\end{equation}

Following common approaches for training PANNs~\cite{ddreview},  the learned model would aim to minimize MSE across observations, leading to a prediction that is an approximation to the mean response.
As the noisy observables here are stresses,  the learned interval models should provide predictions that enforce the physics augmentations as discussed in Section~\ref{sec:icnn}, and at the same time approximate the bounds of the interval from Eq.~\eqref{eq:learnedSInterval}.
Additionally,  to enable interpretability, we also aim to promote sparsity in network parameters~\cite{fuhg2024extreme} as discussed in Section~\ref{sec:l0}.
This is achieved, separately for the upper and the lower bound of the energy, through the following loss function structure:
\begin{equation}
\label{eq:loss}
\mathcal{L} =
\lambda_{\text{mse}}\mathcal{L}_{\text{mse}} +
\lambda_{\text{sparse}}\mathcal{L}_{\text{sparse}} + 
\lambda_{\text{bound}}\mathcal{L}_{\text{bound}}
\end{equation}
Here $\mathcal{L}_{\text{mse}}$ is the mean squared error (MSE) evaluated on \emph{stresses},  computed over the training dataset $\{(\mathbf{x}_i,\mathbf{y}_i)\}_{i=1}^{N}$, where  $\mathbf{x}_i$ collects the deformation state (in terms of the invariants) and  $\mathbf{y}_i$ is the corresponding observed second Piola-Kirchhoff stress tensor.
It is important to note that the network $f(\mathbf{x}_i; \bar{\bm{\theta}} \odot \bm{\mathsf{z}})$ does not output the stress directly; its scalar output is a free energy density.
Following the normalization used for physics-augmented constitutive networks in Fuhg et al.~\cite{fuhg2024extreme}, we enforce the stress-free reference condition and energy normalization introduced in Section~\ref{sec:interval_hyperelasticity} by correcting each learned energy branch according to
\begin{equation}
\Psi_{\bm{\theta}}(I_1,I_2,J)
=
f(I_1,I_2,J;\bm{\theta})
-n_{\mathrm{ref}}(J-1)
-f(3,3,1;\bm{\theta}),
\label{eq:stress_free_correction}
\end{equation}
where
\begin{equation}
n_{\mathrm{ref}}
=
2\left(
\frac{\partial f}{\partial I_1}
+2\frac{\partial f}{\partial I_2}
+\frac{1}{2}\frac{\partial f}{\partial J}
\right)_{(I_1,I_2,J)=(3,3,1)} .
\end{equation}
The constant subtraction sets $\Psi_{\bm{\theta}}(3,3,1)=0$, while the term linear in $J$ 
ensures $\mathbf{S}(\mathbf{F}=\mathbf{I})=\bm{0}$.
The stress prediction is obtained from this corrected energy by automatic differentiation, 
consistent with Eq.~\eqref{eq:learnedSInterval}. The derived stress from the corrected energy in Eq.~\eqref{eq:stress_free_correction} enters the loss.
Denoting the energy-to-stress operator by $\mathbf{S}[\cdot]=2\,\partial(\cdot)/\partial\mathbf{C}$, the MSE loss reads 
\begin{equation} 
\mathcal{L}_{\text{mse}} = \frac{1}{N}\sum_{i=1}^{N} \left\|  \mathbf{S}\!\big[\Psi_{\bm{\theta}}(\mathbf{x}_i)\big]  - \mathbf{y}_i \right\|_2^2, 
\end{equation} 
where $\Psi_{\bm{\theta}}(\mathbf{x}_i)$ is obtained from the gated parameters $\bm{\theta} = \bar{\bm{\theta}} \odot \bm{\mathsf{z}}$ as defined in Section~\ref{sec:l0}, and $\|\cdot\|_2$ is the Frobenius norm over the stress components.

The second term in the loss of Eq.~\eqref{eq:loss}, $\mathcal{L}_{\text{sparse}}$, is identified with the smoothed $L_0$ regularization of Eq.~\eqref{eq:l0}, i.e.\ $\mathcal{L}_{\text{sparse}}\equiv\mathcal{L}_{0}$ and $\lambda_{\text{sparse}}\equiv\lambda_{0}$, and acts by penalizing the number of active parameters.

As the aim is to learn the free energy density interval that covers the observables, penalization for violations of the learned interval bounds in terms of principal stresses is required for cases where the data does not comply with the learned bounds. 
To achieve this, in the third loss term $\mathcal{L}_{\text{bound}}$ we utilize the Macaulay bracket
\[
\langle x \rangle =
\begin{cases}
0, & x < 0, \\
x, & x \ge 0 ,
\end{cases}
\]
which activates only when the constraint is violated. Since this function is non-differentiable at $x=0$, we replace it with a smooth, differentiable approximation
\begin{equation}
\label{eq:smooth_macaulay}
g(x) = \operatorname{sig}(x/\varepsilon)\,\bigl(2\operatorname{sig}(mx)-1\bigr),
\end{equation}
where $\operatorname{sig}(x)=(1+e^{-x})^{-1}$ denotes the sigmoid function. The parameter $\varepsilon$ 
controls the sharpness of the gating function $\operatorname{sig}(x/\varepsilon)$, which suppresses contributions
for $x<0$ and approaches unity for $x>0$, while $m$ controls the steepness of the ramp $2\operatorname{sig}(mx)-1$ as
the violation increases. The influence of these parameters on the shape of $g(x)$ is illustrated in 
Fig.~\ref{fig:boundloss}.
\rev{This smoothing is analogous in spirit to the Huber regularization of the L1 loss~\cite{huber1964robust}.}

\begin{figure}[t]
\centering
\begin{subfigure}[t]{0.48\textwidth}
    \centering
    \maybeincludegraphics[width=\linewidth]{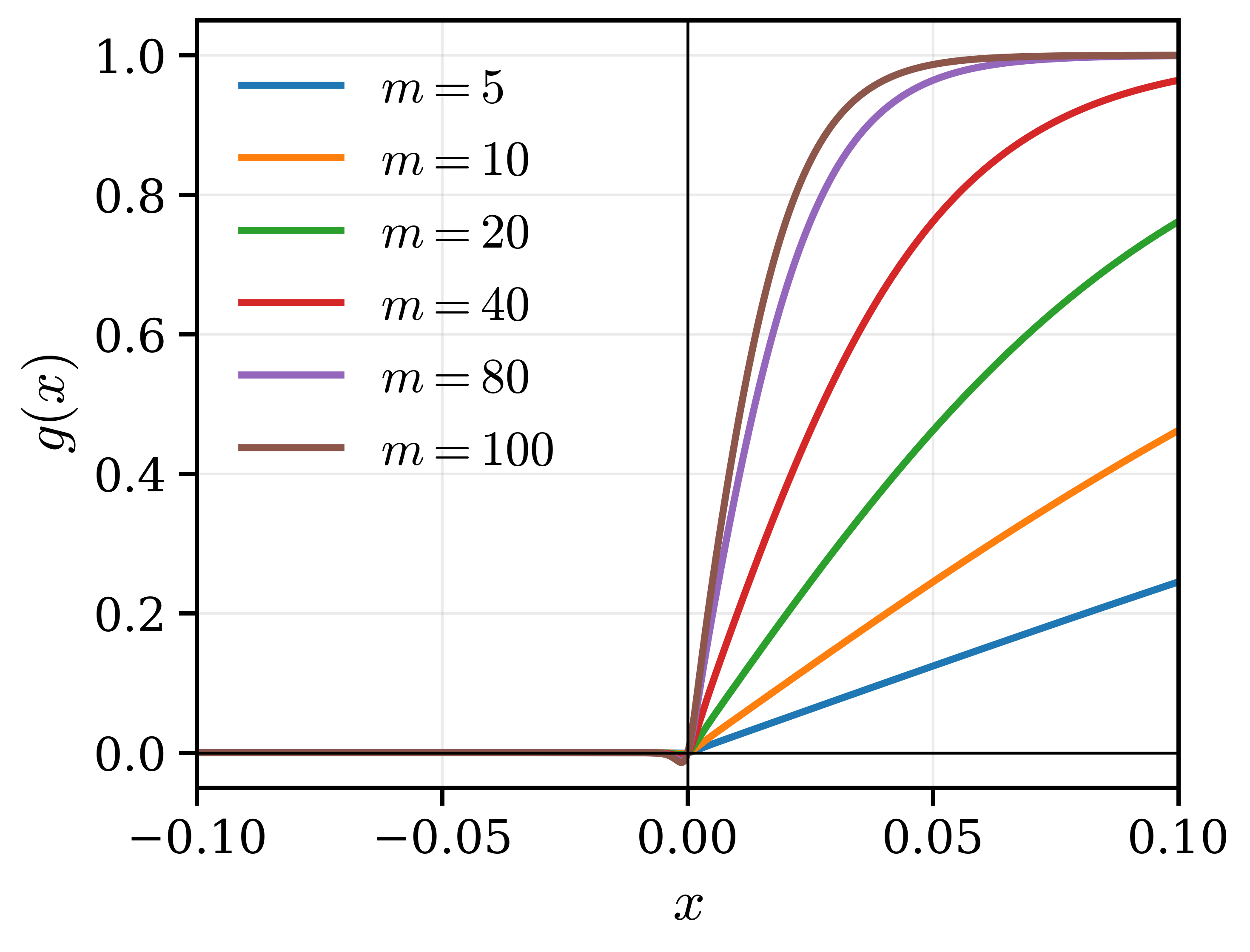}
    \caption{}
\end{subfigure}
\hfill
\begin{subfigure}[t]{0.48\textwidth}
    \centering
    \maybeincludegraphics[width=\linewidth]{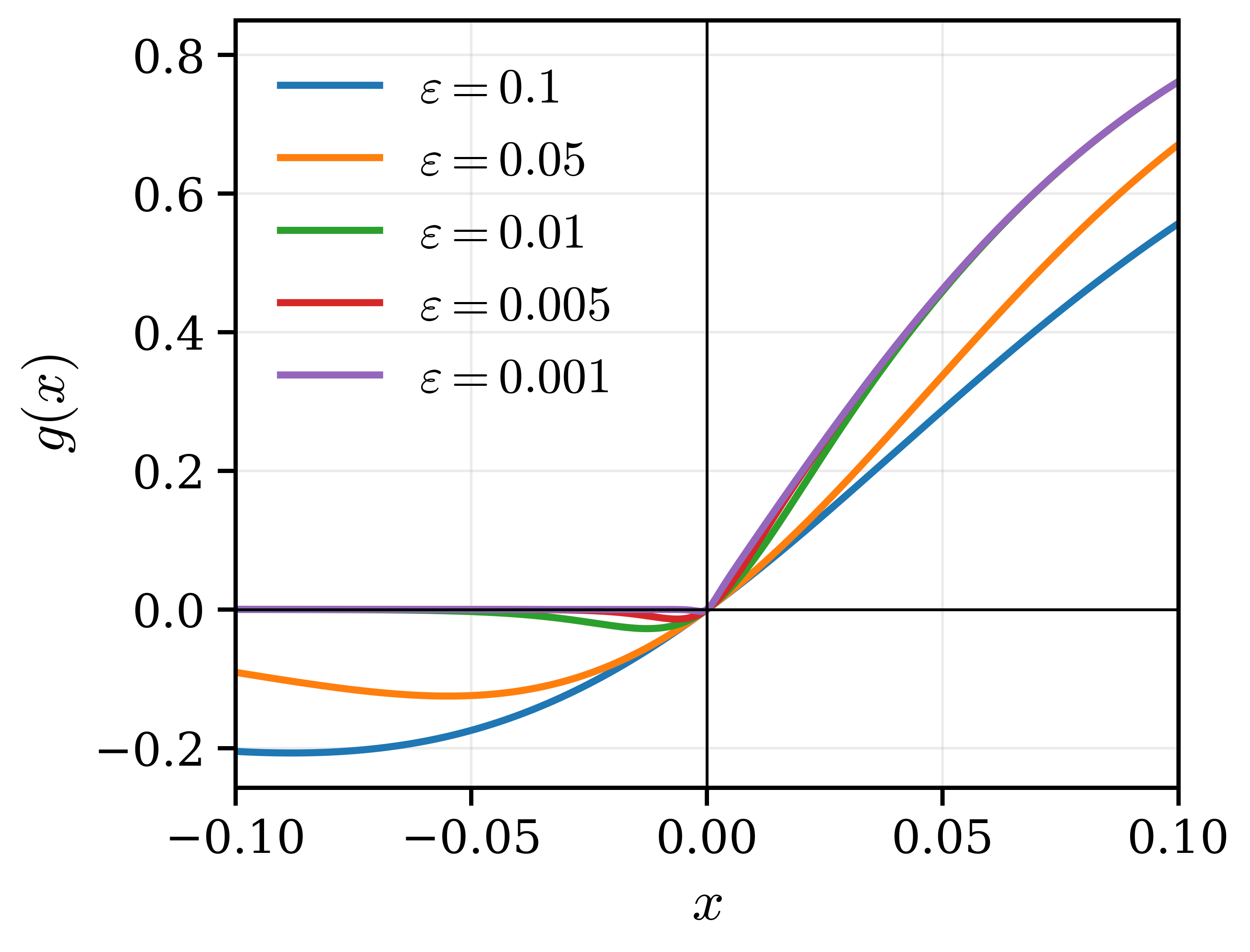}
    \caption{}
\end{subfigure}
\caption{Smooth approximation of the Macaulay bracket used in the bound loss.
(a)~Effect of the ramp parameter~$m$ at fixed $\varepsilon=10^{-3}$.
(b)~Effect of the gating parameter~$\varepsilon$ at fixed $m=20$.}
\label{fig:boundloss}
\end{figure}

We formulate the constraint in terms of the magnitude of the observables $|S_i|$ and of the learned representation of the interval from Eq.~\eqref{eq:learnedSInterval}. The violation variable $x$ is defined for the upper and lower bounds as
\begin{equation}
    x =
\begin{cases}
\mathrm{sign}(E_i)(S_i - \overline{S}^*_i), & \text{upper bound},\\[2mm]
\mathrm{sign}(E_i)(\underline{S}^*_i - S_i), & \text{lower bound},
\end{cases}
\label{boundloss}
\end{equation}
and the smooth Macaulay approximation $g(x)$ is applied in both cases. The bound term for the loss function can be written as 
\begin{equation} 
  \mathcal{L}_{\text{bound}} = 
  \frac{1}{N}\sum_{i=1}^{N} g(x_i). 
\end{equation}

\rev{The interval conditions of Section~\ref{sec:interval_hyperelasticity} are realized within the ICNN-based parametric class of Section~\ref{sec:icnn}.
Let $\mathcal{H}_{\mathrm{ICNN}}$ denote the class of free energy densities representable by the ICNN architecture of Eq.~\eqref{icnn_formulation}.
Formally, the learned branches may be viewed as
\[
\underline{\Psi},\Psi,\overline{\Psi}\in\mathcal{H}_{\mathrm{ICNN}},
\]
subject to the reference-state conditions and the required stress ordering.
The ICNN architecture therefore defines the representational approximation being made; it is not part of the underlying definition of the interval itself.
Likewise, initializing the lower and upper networks from the trained mean is a numerical warm-start strategy and does not impose an additional 
constitutive assumption on the admissible interval.}

Ultimately, the iPANN construction enables a data-driven approximation of the upper and 
lower free energy bounds that enclose all the stress observations. Predictively, it enables 
propagating the aleatoric uncertainty in unseen strain states.

\subsection{fPANNs and membership}\label{sec:fem_membership}
Furthermore, fPANNs are constructed based on iPANNs, enabling approximation of free energy density bounds for different levels of membership. 
Following the discussion in Section~\ref{sec:fuzzy_theory}, in order to construct a fuzzy set for the free energy density, a membership function $\mu$ assigns to each element $\Psi(\mathbf{F})$ a degree of membership between 0 and 1:
\begin{equation}
    \mu : \Psi(\mathbf{F}) \to [0,1], 
    \qquad \Psi(\mathbf{F}) \mapsto \mu(\Psi(\mathbf{F})).
\end{equation}
The $\alpha$-cut of a fuzzy set of free energy densities, for a given level $\alpha \in (0,1]$, is defined as
\begin{equation}
    \Psi_f^\alpha(\mathbf{F}) = \{ \Psi(\mathbf{F}) \in \mathbf{\Psi}(\mathbf{F}) \mid \mu(\Psi) \ge \alpha \}.
\end{equation}
Assuming the construction of a fuzzy set of free energy densities is convex, each $\alpha$-cut reduces to a closed interval
\begin{equation}
    \Psi_f^\alpha(\mathbf{F}) = [\underline{\Psi}_\alpha(\mathbf{F}), \overline{\Psi}_\alpha(\mathbf{F})].
\end{equation}
so that the full fuzzy set can be reconstructed from a nested family of intervals 
\begin{equation}
   \{[\underline{\Psi}_\alpha(\mathbf{F}), \overline{\Psi}_\alpha(\mathbf{F})]\}\quad\forall\,\alpha\in(0,1].
\end{equation}
The concepts introduced in this subsection are illustrated schematically in Fig.~\ref{fig:fuzzy_triangle}.

\begin{figure}[t]
\centering
\maybeincludegraphics[width=0.72\linewidth]{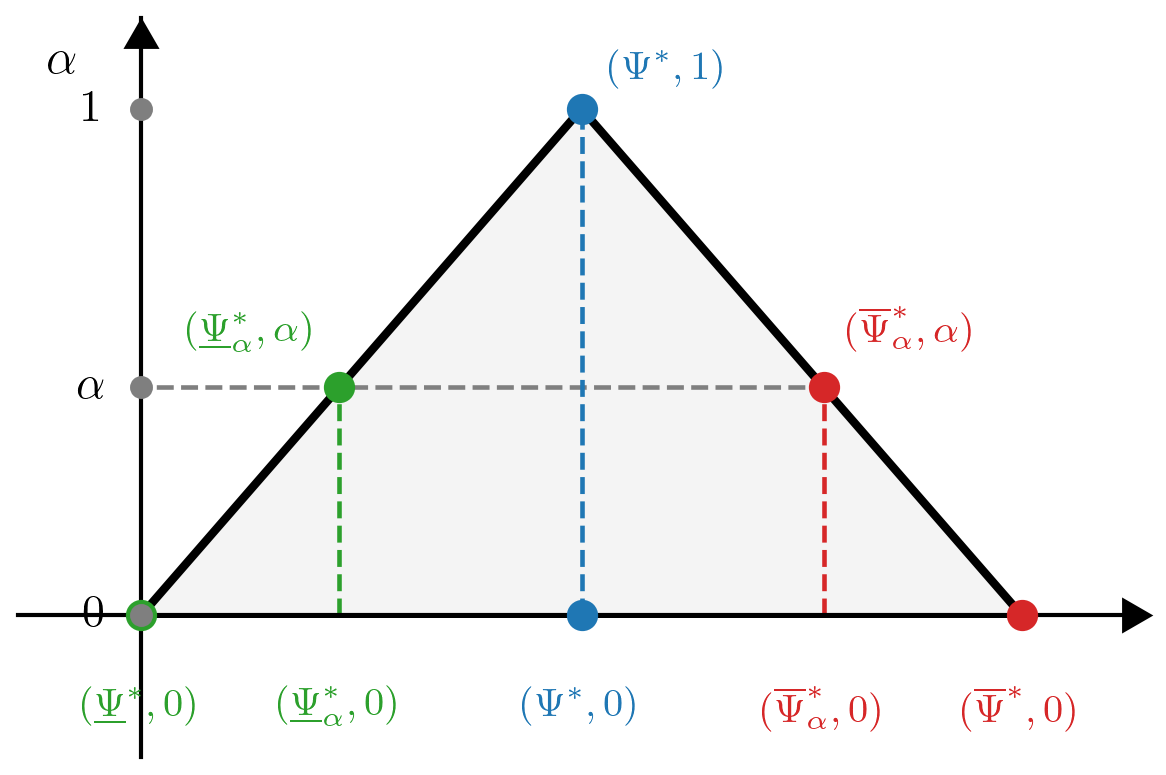}
\caption{Triangular fuzzy membership function for the learned free energy densities.}
\label{fig:fuzzy_triangle}
\end{figure}

Furthermore, we construct a convex fuzzy set 
based on the learned iPANN for the bounds $\underline{\Psi}^*(\mathbf{F}), \overline{\Psi}^*(\mathbf{F})$ 
and PANN for the mean $ {\Psi}^*(\mathbf{F})$.
Through Eq.~\eqref{eq:loss}, removing the term $\mathcal{L}_{\text{bound}}$ from the loss will result in learning an approximation of the mean response, and corresponding mean free energy density, for which the following inequality holds $\underline{\Psi}^*(\mathbf{F})\le {\Psi}^*(\mathbf{F}) \le \overline{\Psi}^*(\mathbf{F})$. To maintain convexity with respect to membership, we employ a piecewise linear interpolation between the  three models following
\begin{equation}
  \underline{\Psi}^*_\alpha(\mathbf{F})  = (1-\alpha)\underline{\Psi}^*(\mathbf{F}) + \alpha\Psi^*(\mathbf{F}),
  \label{eq:interpolation1}
\end{equation}
\begin{equation}
  \overline{\Psi}^*_\alpha(\mathbf{F})  = \alpha\Psi^*(\mathbf{F}) + (1-\alpha)\overline{\Psi}^*(\mathbf{F})  .
  \label{eq:interpolation2}
\end{equation}
Following Eqs.~\eqref{eq:PsiInterval} and~\eqref{eq:learnedPsiInterval}, an enclosure for the principal values of 
the second Piola-Kirchhoff stress tensor, the corresponding $\alpha-$cut is given as:
\begin{equation}
S^*_{\alpha,i} \in \mathcal{S}^*_{\alpha,i}
= [\operatorname{sign}(E_i)\underline{S}^*_{\alpha,i},\operatorname{sign}(E_i)\overline{S}^*_{\alpha,i}], \quad i=1,2,3, \quad (\text{no}\,\operatorname{sum}).\label{eq:alphaSIcut}
\end{equation}
The second Piola-Kirchhoff stress for a given $\alpha-$cut can be calculated using
Eqs.~\eqref{eq:interpolation1} and~\eqref{eq:interpolation2} as
\begin{equation}
\underline{\mathbf{S}}^*_\alpha=2\frac{\partial \underline{\Psi}^*_\alpha}{\partial \mathbf{C}},
\qquad
\overline{\mathbf{S}}^*_\alpha=2\frac{\partial \overline{\Psi}^*_\alpha}{\partial \mathbf{C}}.
\label{eq:S_alpha_cut}
\end{equation}

\section{Numerical results and illustration}\label{sec:results}

\subsection{Data generation}\label{sec:data_generation}

While the free energy density
is a fundamental quantity in hyperelasticity, it is not directly
observable in experiments. Our training data consists of 
input-output pairs where the inputs are the principal invariants of the right Cauchy--Green tensor, 
and the outputs are the corresponding stress components of the second Piola-Kirchhoff stress tensor.
These labels are generated from the Gent-Gent hyperelastic model, with free energy density given by Eq.~\eqref{eq:gent_gent}:
\begin{equation}
\Psi(I_1,I_2,J)
=
-\frac{\mu}{2}\,J_m\log\!\left(1-\frac{I_1-3}{J_m}\right)
+C_2\log\!\left(\frac{I_2}{3}\right)
+k\left(\frac{1}{2}(J^2-1)-\log J\right),
\label{eq:gent_gent}
\end{equation}
where the material parameters are taken as $\mu=2.4195\,\rev{\mathrm{MPa}}$, $J_m=77.931$, $C_2=-0.75\mu$, and $k=\mu/2$.
\rev{Accordingly, second Piola--Kirchhoff stresses are reported in $\mathrm{MPa}$ and free energy densities in $\mathrm{MPa}$ 
(equivalently $\mathrm{MJ}/\mathrm{m}^3$).}
It is noted that this model is not polyconvex, a purposeful choice to create some 
discrepancy from polyconvex representations in the learning scheme~\cite{tan2026towards, boyce_arruda1993three,gent1996new}.

Following Fuhg et al.~\cite{fuhg2022physics,fuhg2024tensorbasis,tan2026towards}, we generate data on complex heterogeneous stress states.
Starting from the undeformed configuration $\mathbf{F}_0 = \mathbf{I}$,  we sample deformation gradients around this reference state using Latin Hypercube Sampling (LHS) over the full 9-dimensional space of components, with
\begin{equation}
F_{ij} \in
\begin{cases}
1 + \mathcal{U}(-\delta,\delta), & i=j,\\[2pt]
\mathcal{U}(-\delta,\delta), & i\neq j,
\end{cases}
\label{eq:F_sampling}
\end{equation}
and a bound of $\delta=0.2$, where $\mathcal{U}(-\delta,\delta)$ denotes the uniform distribution on $(-\delta,\delta)$.

For each sampled $\mathbf{F}$, we compute the right Cauchy--Green tensor  $\mathbf{C} = \mathbf{F}^T\mathbf{F}$ and its principal invariants, forming a point cloud of $250{,}000$ physically valid samples (i.e., $\det(\mathbf{F}) > 0$) that approximates the convex hull of physically feasible invariant states.
From this convex hull, we then run a hybrid farthest-point sampling (FPS) and simulated  
annealing (SA)~\cite{fuhg2022physics,tan2026towards} optimization scheme to select a representative 
subset of $500$ well-spaced invariant triplets $(I_1^i, I_2^i, I_3^i)$ that maximizes spatial coverage 
while maintaining diversity. The origin of the invariant space $(3,3,1)$, corresponding to the undeformed state, 
is manually enforced to be included in the selection.
The resulting $500$-point subset and the underlying invariant convex hull are shown in Appendix~\ref{app:invariant_sampling} (Fig.~\ref{fig:invariant_sampling}).
\rev{Since the material is isotropic, the free energy depends on $\mathbf{F}$ only through the principal invariants of $\mathbf{C}$.
Hence a diagonal representative $\mathbf{C}=\mathrm{diag}(\lambda_1^2,\lambda_2^2,\lambda_3^2)$ realizing each selected invariant triplet is a sufficient choice of training state; the induced second Piola--Kirchhoff stress is then diagonal in the same basis, with entries equal to the principal stresses $S_i$.
The selection of invariant triplets over the admissible hull and the recovery of the corresponding diagonal $\mathbf{C}$ follow the procedure of Ref.~\cite{tan2026towards}, to which we refer the reader for algorithmic details.}
We denote by \emph{clean stress} the corresponding noise-free Gent-Gent response evaluated at each 
deformation state. 
This quantity serves as a \emph{synthetic reference trajectory}; it is not intended to represent a 
directly measurable experimental but rather provides controlled ground truth against 
which the mean branch learned in Section~\ref{sec:loss} and the interval and fuzzy bounds 
constructed in Section~\ref{sec:fem_membership} can be assessed. 
To emulate measurement noise, each clean stress 
component $S_{ij}^{\mathrm{clean}}$ is corrupted with \emph{multiplicative heteroscedastic} 
Gaussian noise,
\begin{equation}
  \tilde{S}_{ij} = \bigl(1 + \mu + \sigma\,\varepsilon_{ij}\bigr)\,S_{ij}^{\mathrm{clean}},
  \qquad \varepsilon_{ij}\stackrel{\text{i.i.d.}}{\sim}\mathcal{N}(0,1),
  \label{eq:hetero_noise}
\end{equation}
where $\mu$ and $\sigma$ control the mean and standard deviation of the perturbation, 
respectively. The specific values of $(\mu,\sigma)$ and the random seed used to instantiate 
Eq.~\eqref{eq:hetero_noise} for each experiment are given in Sections~\ref{sec:results_500} 
and~\ref{sec:results_noise_types}.
The resulting $\tilde{S}_{ij}$ are the \emph{noisy observations} presented to the network during 
training.
Throughout Section~\ref{sec:results}, models are fit on the training set and evaluated on the 
independent test set.

Whereas the training data form a point cloud sampled over the full nine-dimensional
deformation-gradient space, the test data are generated along a single coherent loading path so as
to probe how the learned models generalize along a physically interpretable deformation mode. We use a 
constrained uniaxial deformation $\mathbf{F} = \mathrm{diag}(\lambda, 1, 1)$, in
which the axial stretch $\lambda$ is varied while the two transverse stretches are held fixed at
unity. We take $1000$ equally spaced values of $\lambda \in [0.6, 1.4]$, i.e. a $\pm 40\%$ axial
stretch about the undeformed state, and, for each, form $\mathbf{C} = \mathbf{F}^{T}\mathbf{F}$ and
evaluate the corresponding second Piola--Kirchhoff stress and free energy density from the same Gent-Gent
model, yielding diagonal stress components as in the training set. We note that the axial-stretch range of this test path exceeds the per-component training
interval $[1-\delta, 1+\delta] = [0.8, 1.2]$; consequently the test set probes both interpolation
(for $\lambda \in [0.8, 1.2]$) and mild extrapolation,
providing a more strict assessment of the learned bounds than the training distribution alone.

\subsection{Two-stage warm-start training procedure}\label{sec:three_step}

Motivated by transfer learning and two-stage approaches in which a pretrained point-prediction model provides the initialization for learning prediction bounds~\cite{tan2026towards, pan2009survey,ws1_patel}, we employ a sequential two-stage warm-start strategy, rather than training the mean, lower, 
and upper branches simultaneously.
Joint optimization of all three networks from random initialization is less efficient because 
the objectives compete: the stress MSE term drives each branch toward the central noisy trend, 
whereas the bound penalties in Eq.~\eqref{boundloss} require the lower and upper responses to 
shift and widen so as to enclose the noisy observations.
Training the mean first with $\lambda_{\text{bound}}=0$ and a comparatively strong $L_0$ penalty 
isolates the task of learning a sparse, physics-consistent approximation of the mean response.

\textbf{Step 1: Train mean response.} In the first stage, we train an ICNN to learn the mean free energy density $\Psi$ using the stress-based mean squared error loss together with the smoothed $L_0$ sparsification penalty, while setting $\lambda_{\text{bound}} = 0$. This stage therefore produces a physics-consistent baseline constitutive model that fits the central trend of the stress observations while promoting a compact and interpretable representation. 
The resulting mean model serves as the reference branch of the fuzzy representation and also provides the initialization for the subsequent bound-learning stage.

\textbf{Step 2: Train lower and upper bounds.} In the second stage, two additional ICNNs are initialized from the trained mean network and then optimized to represent the lower and upper free energy density branches of the iPANN. The bound loss in Eq.~\eqref{boundloss} is activated by setting $\lambda_{\text{bound}} \neq 0$. In this stage, the sparsification penalty is substantially reduced relative to the first stage. This choice preserves the sparse constitutive structure inherited from the mean model, while allowing the lower and upper networks sufficient flexibility to adjust their active parameters and widen or shift the response as needed to enclose the noisy stress data. Thus, the second stage does not relearn the entire constitutive law from scratch, it fine-tunes the mean representation into two physically consistent bounding models.
\rev{As noted in Section~\ref{sec:loss}, this inheritance of the mean weights is a numerical warm-start strategy only; it does not restrict the 
admissible interval defined in Section~\ref{sec:interval_hyperelasticity}.}

This transfer learning procedure improves optimization stability and reduces the tendency of 
the bound models to develop unnecessarily complex representations. 
It also ensures that the 
learned interval is organized around a meaningful central response. The resulting 
three learned free energy density functions from this iPANN, are then used to construct
 the fuzzy constitutive representation through the membership-based interpolation described in Section~\ref{sec:fem_membership}.

\subsection{Results}\label{sec:results_main}

We illustrate the framework on the synthetic data of Section~\ref{sec:data_generation}. The baseline case is presented in Section~\ref{sec:results_500}, followed by performance on noise-perturbation illustrated in Section~\ref{sec:results_noise_types}. The learned bounds are subsequently interpreted through the fuzzy membership analysis of Section~\ref{sec:membership_analysis}, and their downstream deployment is demonstrated in the finite element method (FEM) setting of Section~\ref{sec:fem_validation_500}. For convenience, we use the notation \textbf{E1}--\textbf{E4} to refer to the individual numerical experiments presented in this section, as summarized in Table~\ref{tab:datagen}.

\begin{table}[htbp]
\centering
\small
\renewcommand{\arraystretch}{1.5}
\setlength{\tabcolsep}{8pt}
\caption{Data generation modalities for the performed numerical experiments.}
\label{tab:datagen}
\begin{tabular}{cl}
\hline
ID & Data generation cases \\
\hline
\textbf{E1} & Baseline dataset with heteroscedastic noise\\
\textbf{E2} & Repeated realizations by random seed perturbation \\
\textbf{E3} & Repeated realizations by shifting the noise mean \\
\textbf{E4} & Repeated realizations by shifting the standard deviation  \\
\hline
\end{tabular}
\end{table}

\begin{figure}[htbp]
\centering
\begin{subfigure}[t]{0.48\textwidth}
  \centering
  \stresspanel{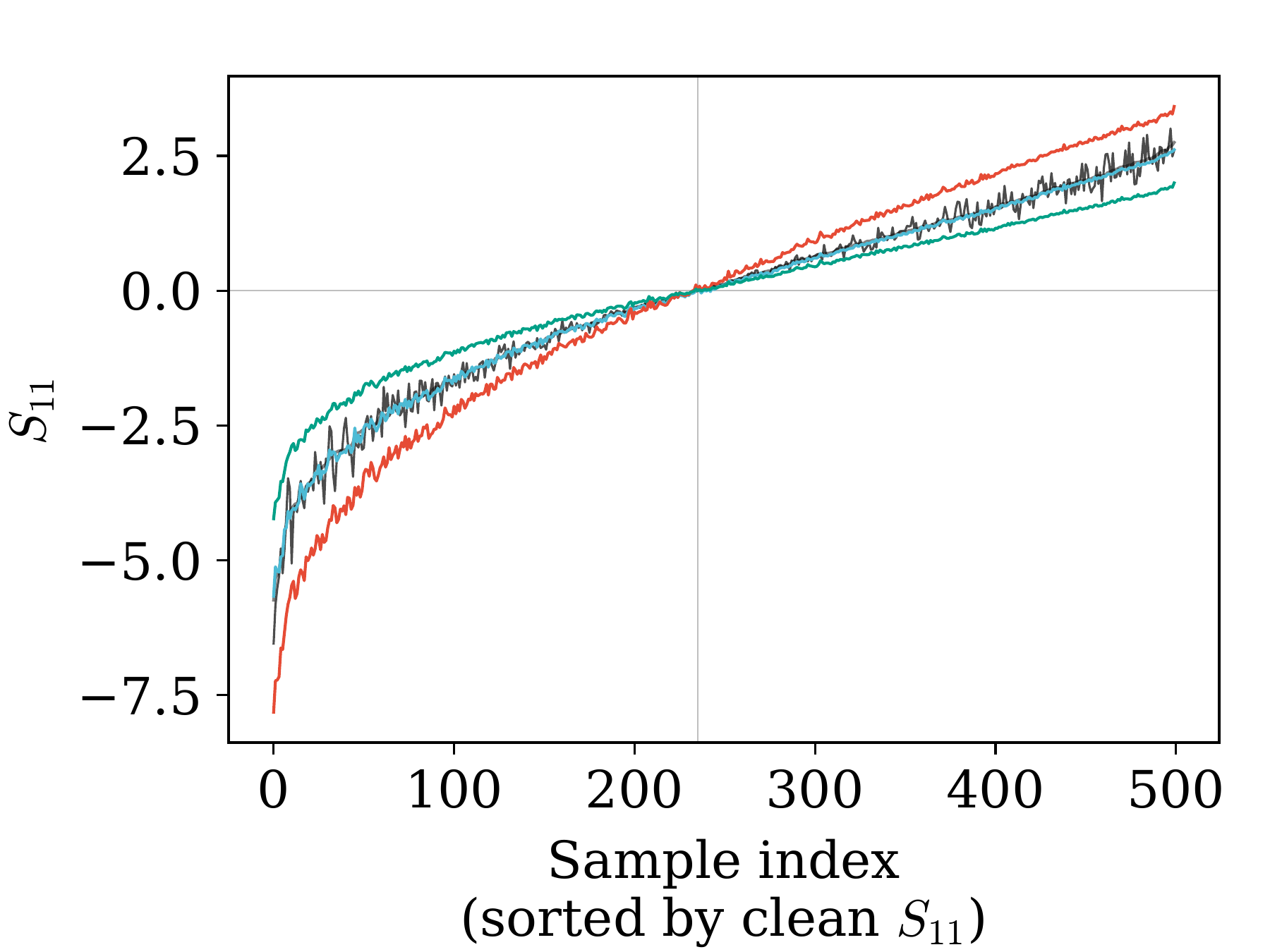}
  \caption{}
\end{subfigure}
\hfill
\begin{subfigure}[t]{0.48\textwidth}
  \centering
  \stresspanel{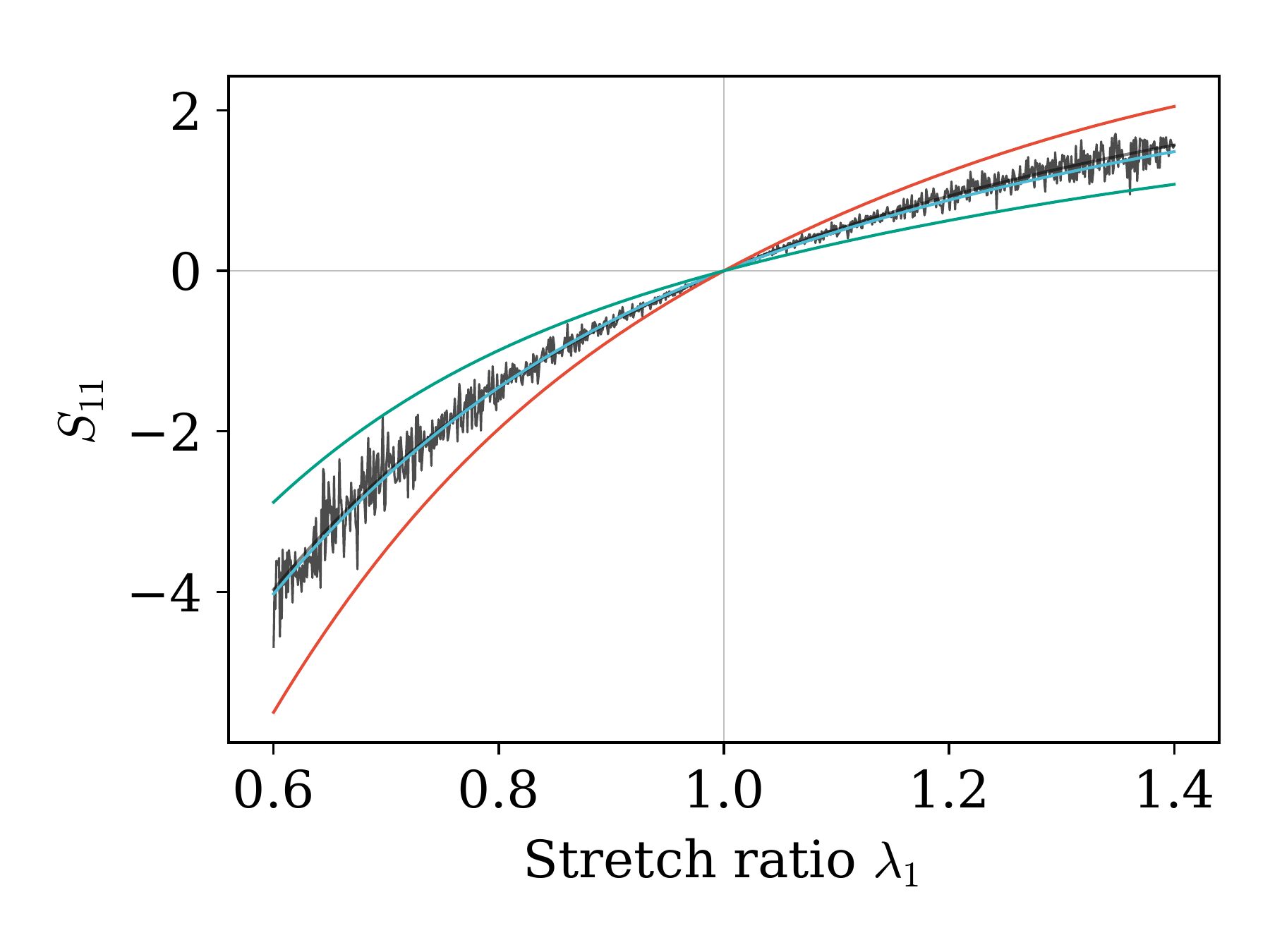}
  \caption{}
\end{subfigure}

\begin{subfigure}[t]{0.48\textwidth}
  \centering
  \stresspanel{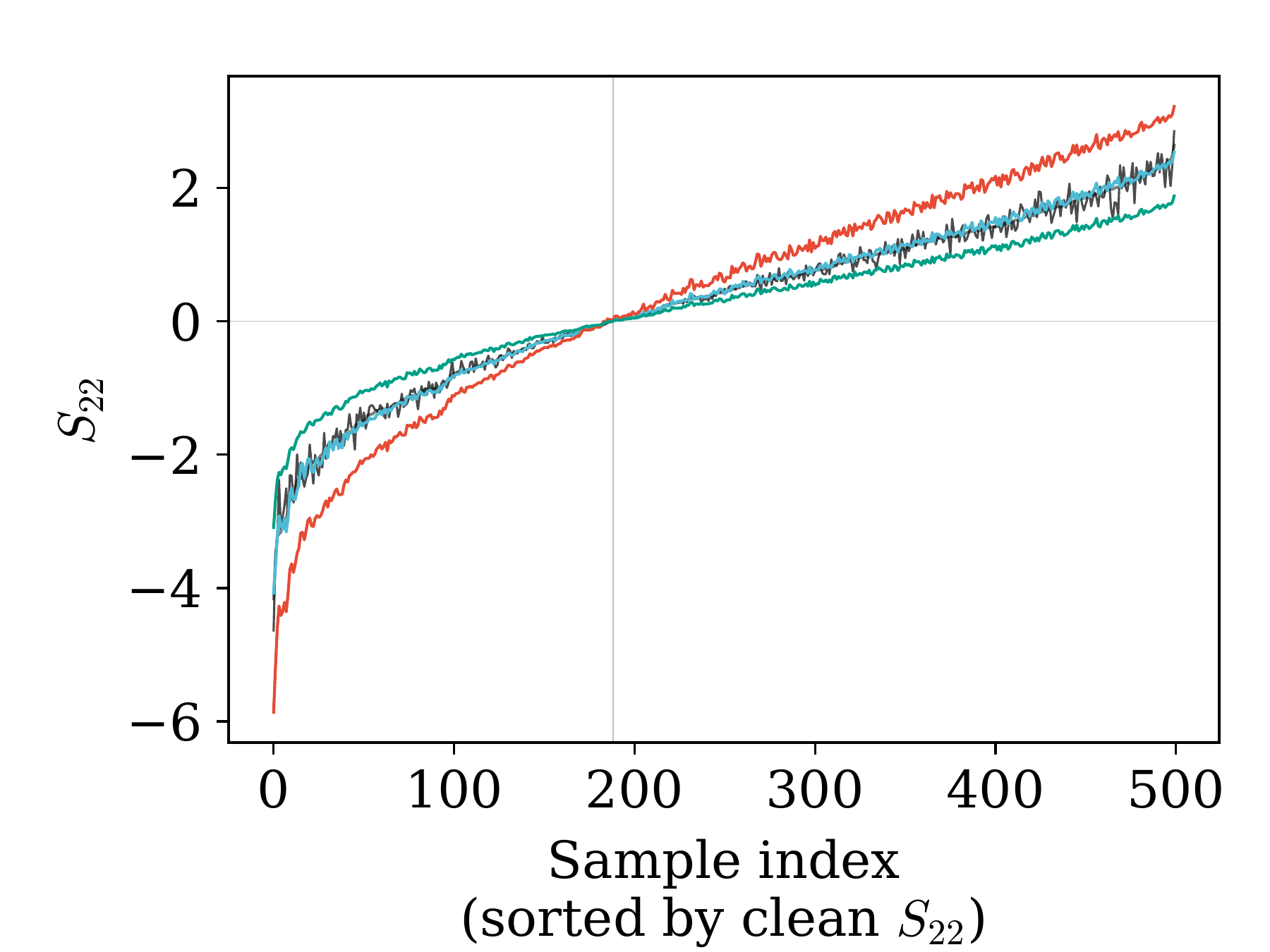}
  \caption{}
\end{subfigure}
\hfill
\begin{subfigure}[t]{0.48\textwidth}
  \centering
  \stresspanel{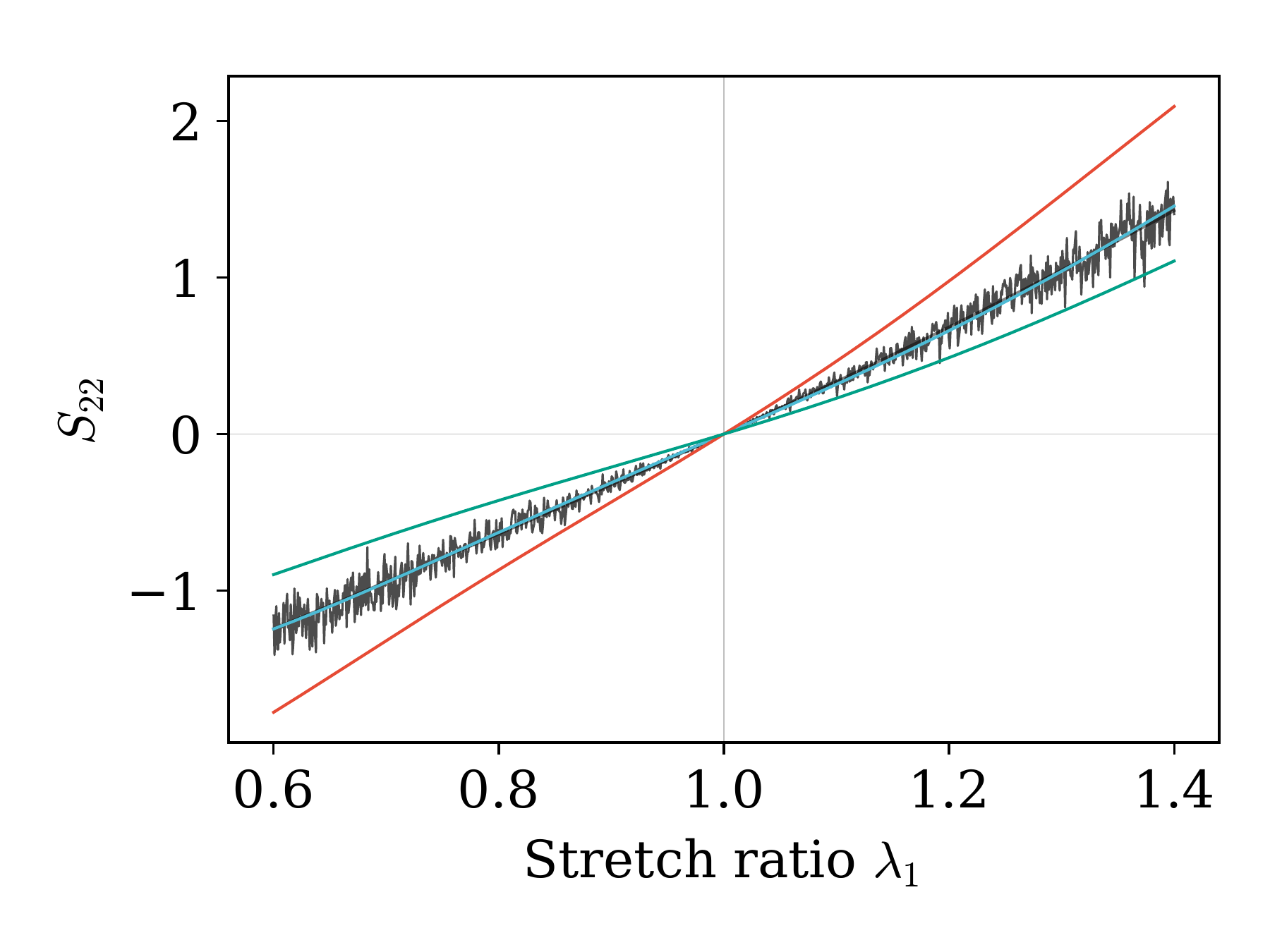}
  \caption{}
\end{subfigure}

\vspace{0.15cm}

\begin{subfigure}[t]{0.48\textwidth}
  \centering
  \stresspanel{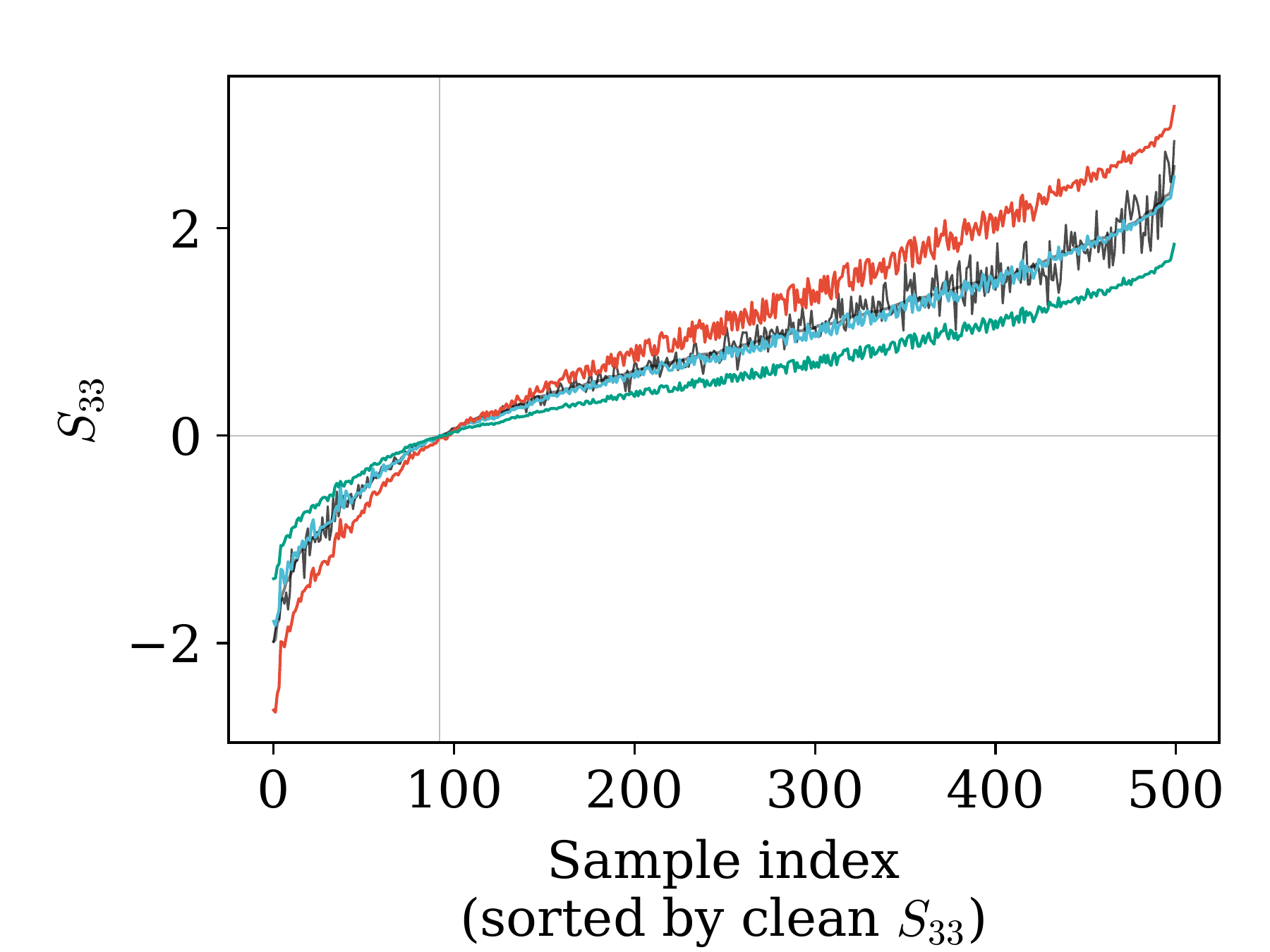}
  \caption{}
\end{subfigure}
\hfill
\begin{subfigure}[t]{0.48\textwidth}
  \centering
  \stresspanel{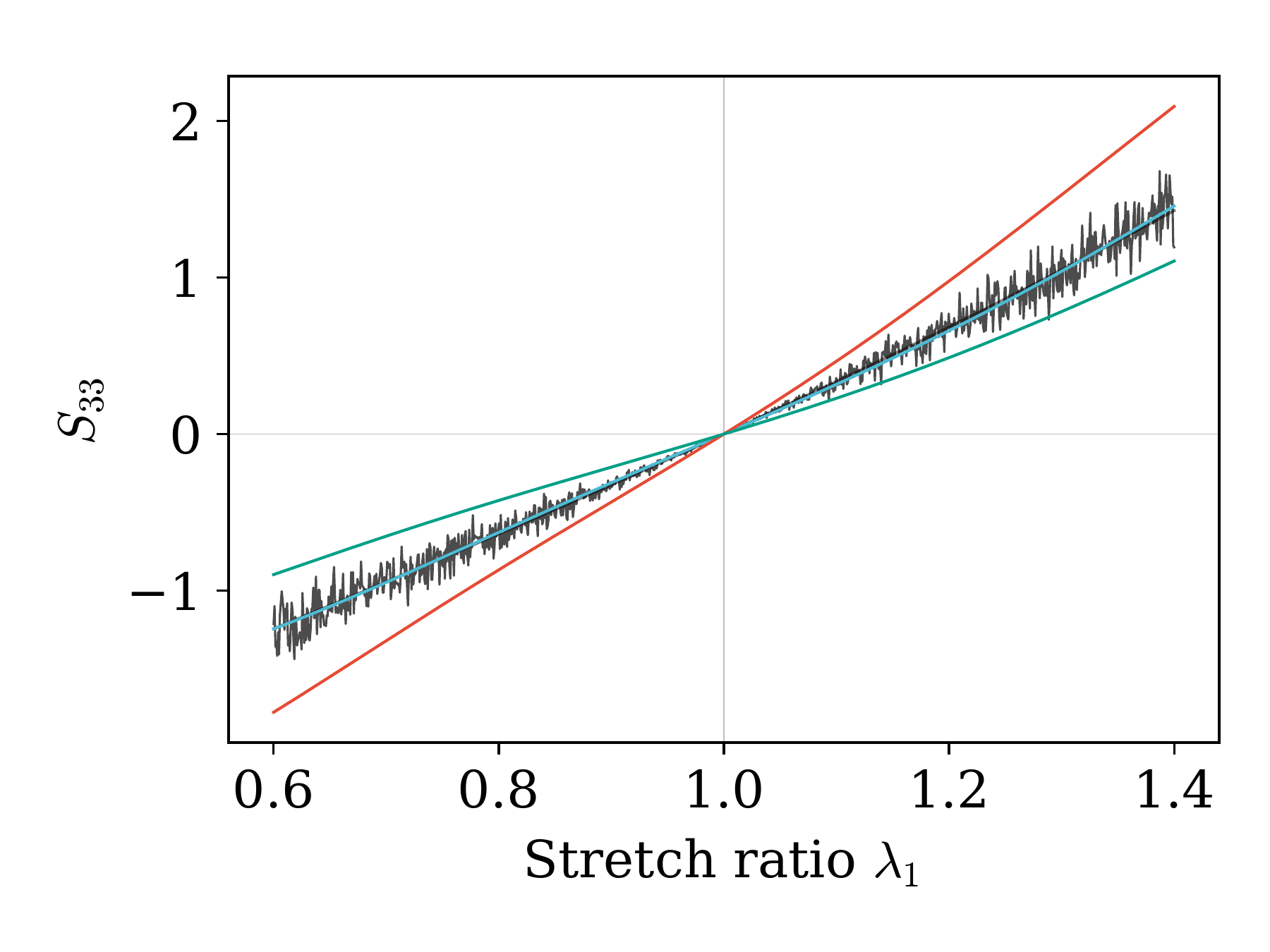}
  \caption{}
\end{subfigure}

\vspace{0.0001in}
\begin{subfigure}[t]{\textwidth}
  \centering
  \stresslegend{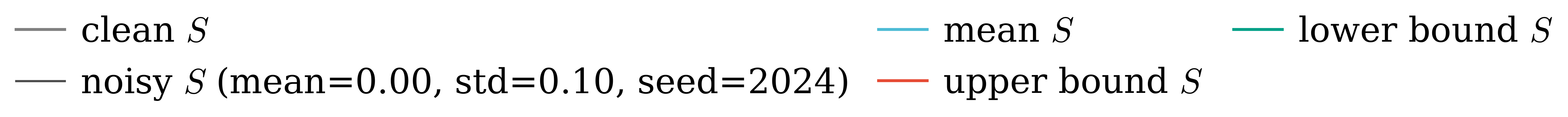}
\end{subfigure}
\stressfigcaption{\textbf{E1}: Stress components ($S_{11}$, $S_{22}$, $S_{33}$) with mean and bounds on training (left) 
and test (right) data.}
\label{fig:exp500_stress}
\end{figure}

\begin{figure}[htbp]
\centering
\begin{subfigure}[t]{0.48\textwidth}
  \centering
  \energypanel{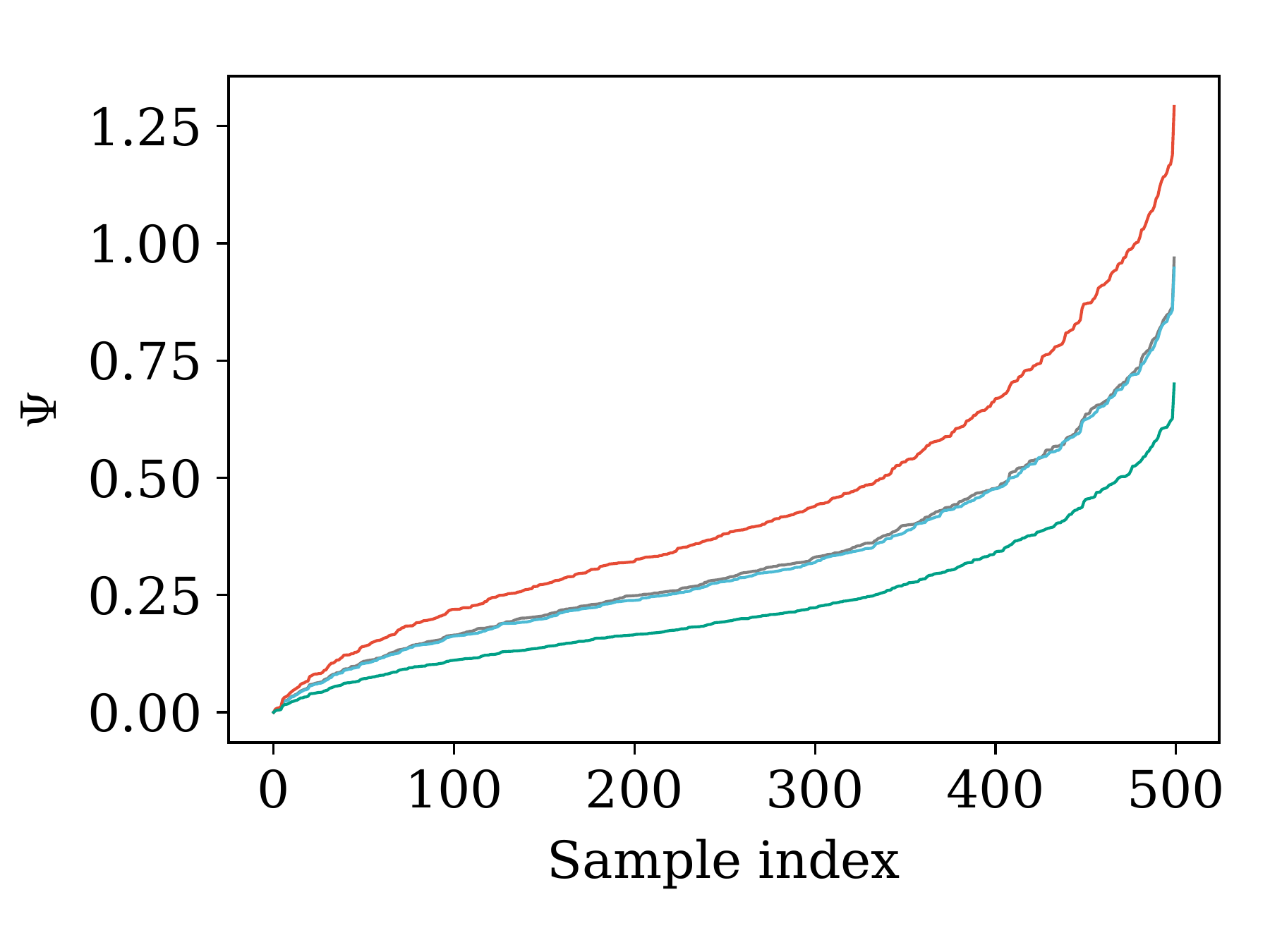}
  \caption{$\Psi$, train.}
\end{subfigure}
\hfill
\begin{subfigure}[t]{0.48\textwidth}
  \centering
  \energypanel{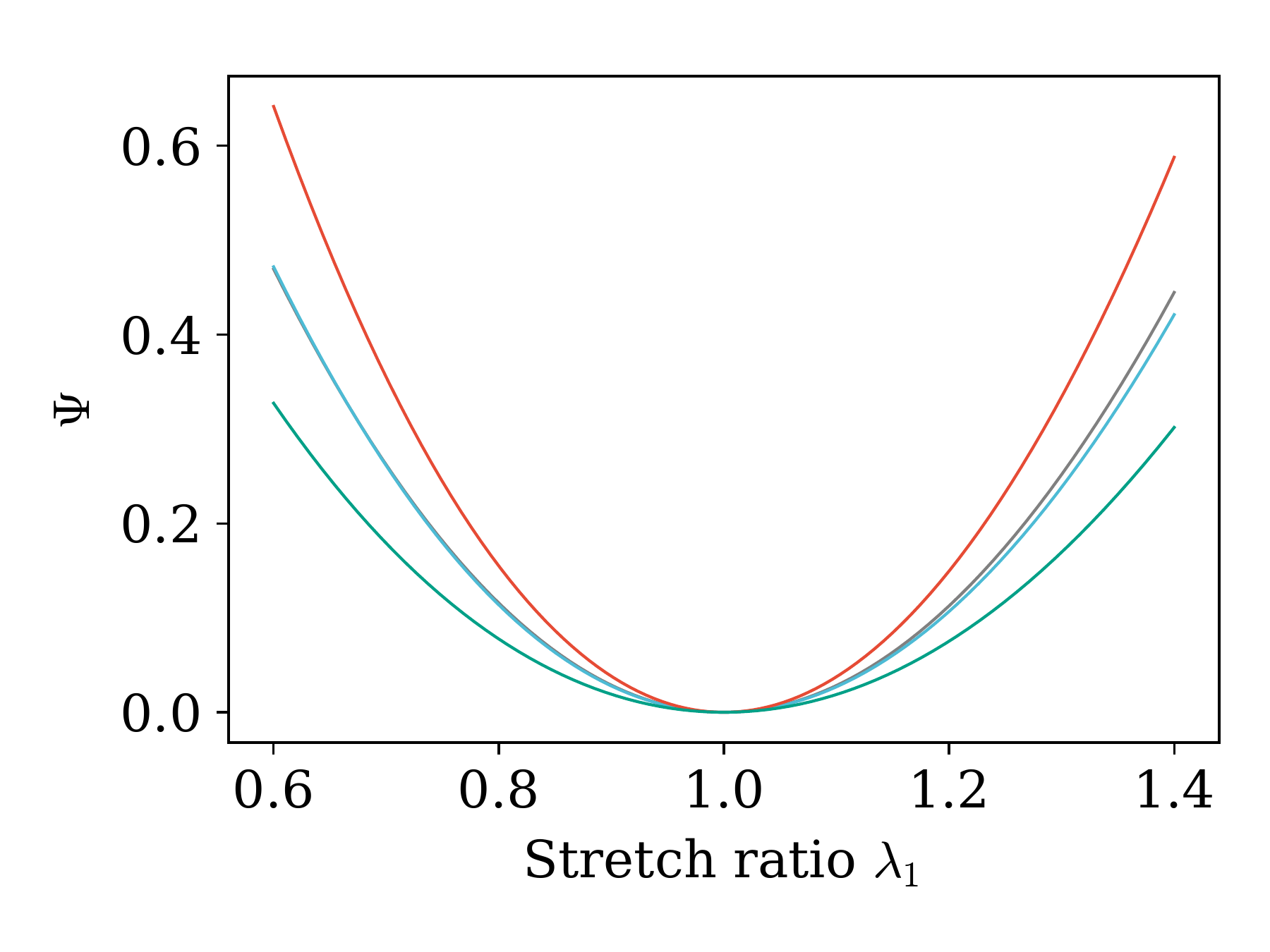}
  \caption{$\Psi$, test.}
\end{subfigure}

\vspace{0.001in}
\begin{subfigure}[t]{\textwidth}
  \centering
  \energylegend{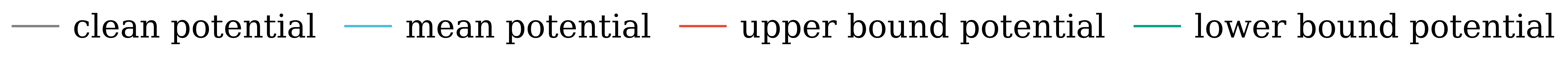}
\end{subfigure}
\caption{\textbf{E1}: Free energy density $\Psi$ with mean and bounds on training (left) 
and test data(right)}
\label{fig:exp500_energy}
\end{figure}

\subsubsection{Baseline case: heteroscedastic noise}\label{sec:results_500}

The baseline experiment~\textbf{E1} evaluates the proposed framework under the multiplicative heteroscedastic noise of Eq.~\eqref{eq:hetero_noise} with $\mu=0$ and $\sigma=0.1$, and serves as the reference case for the subsequent noise-attribute perturbations.
Training and testing use the synthetic data of Section~\ref{sec:data_generation}.
Following the two-stage procedure of Section~\ref{sec:three_step}, the mean model is trained first and then used as a warm start for the upper and lower bound networks of the iPANN.

The learned stress responses are shown in Fig.~\ref{fig:exp500_stress} for the three principal components $S_{11}$, $S_{22}$ and $S_{33}$, on both the training and test sets.
The mean prediction follows the clean Gent-Gent reference closely across all three components, while the lower and upper bounds enclose the scattered noisy observations over the entire deformation range.
The enclosure is not a fixed-width band: it widens where the noisy data are more dispersed and tightens where they concentrate, reflecting the multiplicative nature of the noise.
Close agreement between the training and test panels indicates that the bounds learned on the training data carry over to the test set.
The corresponding free energy density $\Psi$ is reported in Fig.~\ref{fig:exp500_energy}.
Each of the three converged networks reduces to a compact closed-form expression.
The resulting mean and bound energies for~\textbf{E1}, together with their coefficients, are reported in Appendix~\ref{app:learned_expressions}. The iPANN deployment for this case in a finite element solver is demonstrated in Section~\ref{sec:fem_validation_500}.
With this baseline established, Section~\ref{sec:results_noise_types} perturbs one noise attribute at a time: seed, mean, or standard deviation, relative to~\textbf{E1}. 

\subsubsection{Results with different \rev{types} of noise}\label{sec:results_noise_types}

The robustness of the framework is verified by varying one noise attribute at a time while keeping all remaining settings aligned with \textbf{E1}.
Three experiments are considered: multiple independent noise realizations obtained by changing the random seed of the noise generator (\textbf{E2}), a shift in the noise mean (\textbf{E3}), and a change in the noise standard deviation (\textbf{E4}).
These synthetic perturbations are motivated by scenarios encountered in mechanical testing of materials, 
namely (i)~random variations arising from measurement noise, specimen variability, or slightly different initial conditions, (ii)~effects such as load-cell zero offset, calibration drift, strain-gauge bias, or a systematic preprocessing error, 
and (iii)~changes in sensor precision, environmental disturbances, specimen-to-specimen variability, or signal quality.
The corresponding results are collected in Fig.~\ref{fig:noise_stress_seeds} for \textbf{E2}, Fig.~\ref{fig:noise_stress_means} for \textbf{E3}, and Fig.~\ref{fig:noise_stress_std} for \textbf{E4}.

Across all three perturbations the qualitative behavior observed for the baseline is preserved: the
mean tracks the underlying clean response, and the learned bounds enclose the noisy observations on
both the training and test sets. The width of the enclosure adapts to the imposed noise. When
several independent realizations are presented simultaneously (\textbf{E2}), the bounds broaden to contain
the additional spread introduced by the differing seeds; when the noise mean is shifted (\textbf{E3}), the
enclosure follows the displaced data while remaining centered on the mean response; and when the
noise standard deviation is increased (\textbf{E4}), the bounds widen. In every case the
agreement between the training and test panels remains close, confirming that the learned bounds
generalize to the test set and are not unduly sensitive to the specific noise setting used
during training.

\begin{figure}[htbp]
\centering
\begin{subfigure}[t]{0.48\textwidth}
  \centering
  \stresspanel{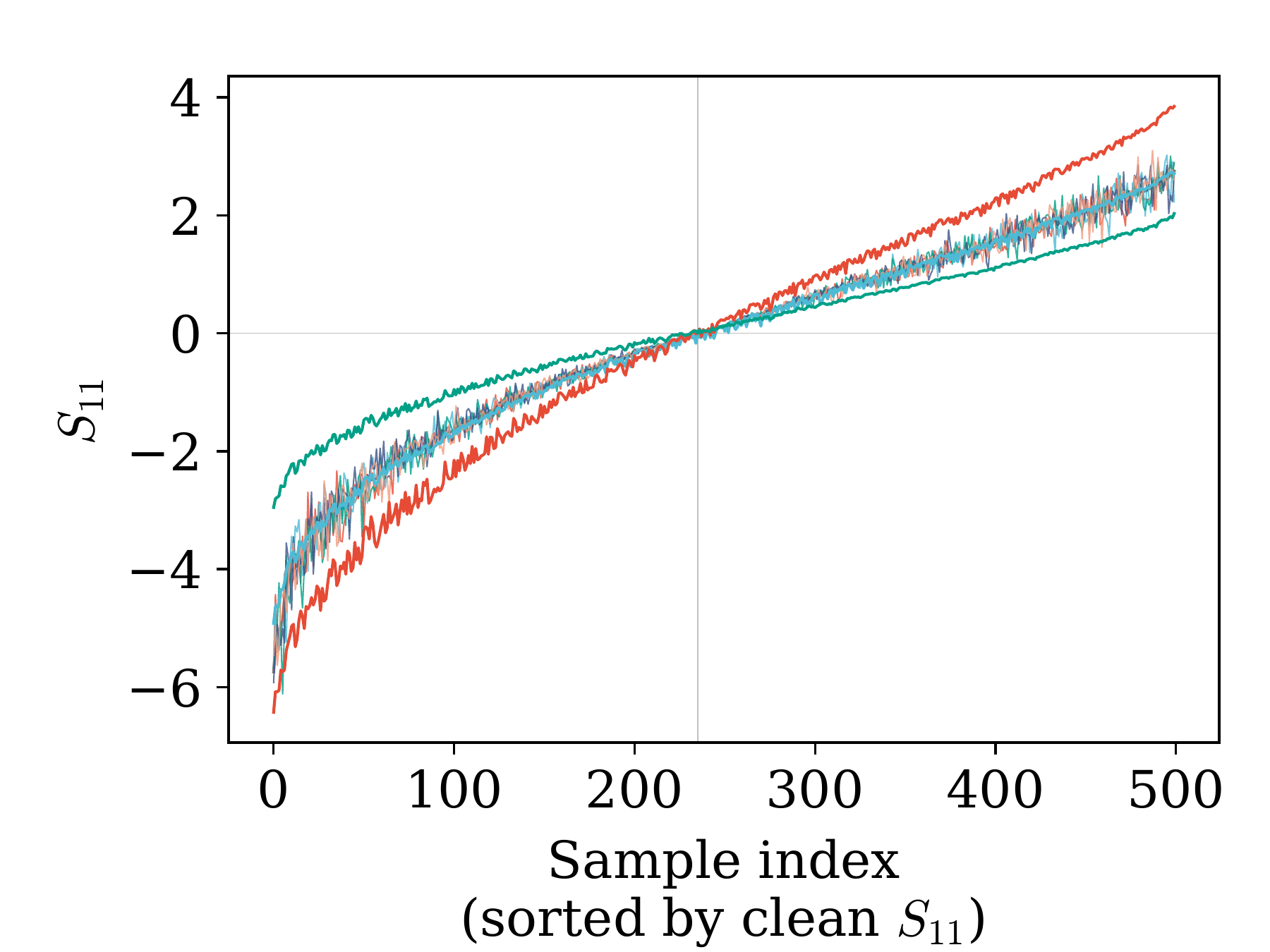}
  \caption{}
\end{subfigure}
\hfill
\begin{subfigure}[t]{0.48\textwidth}
  \centering
  \stresspanel{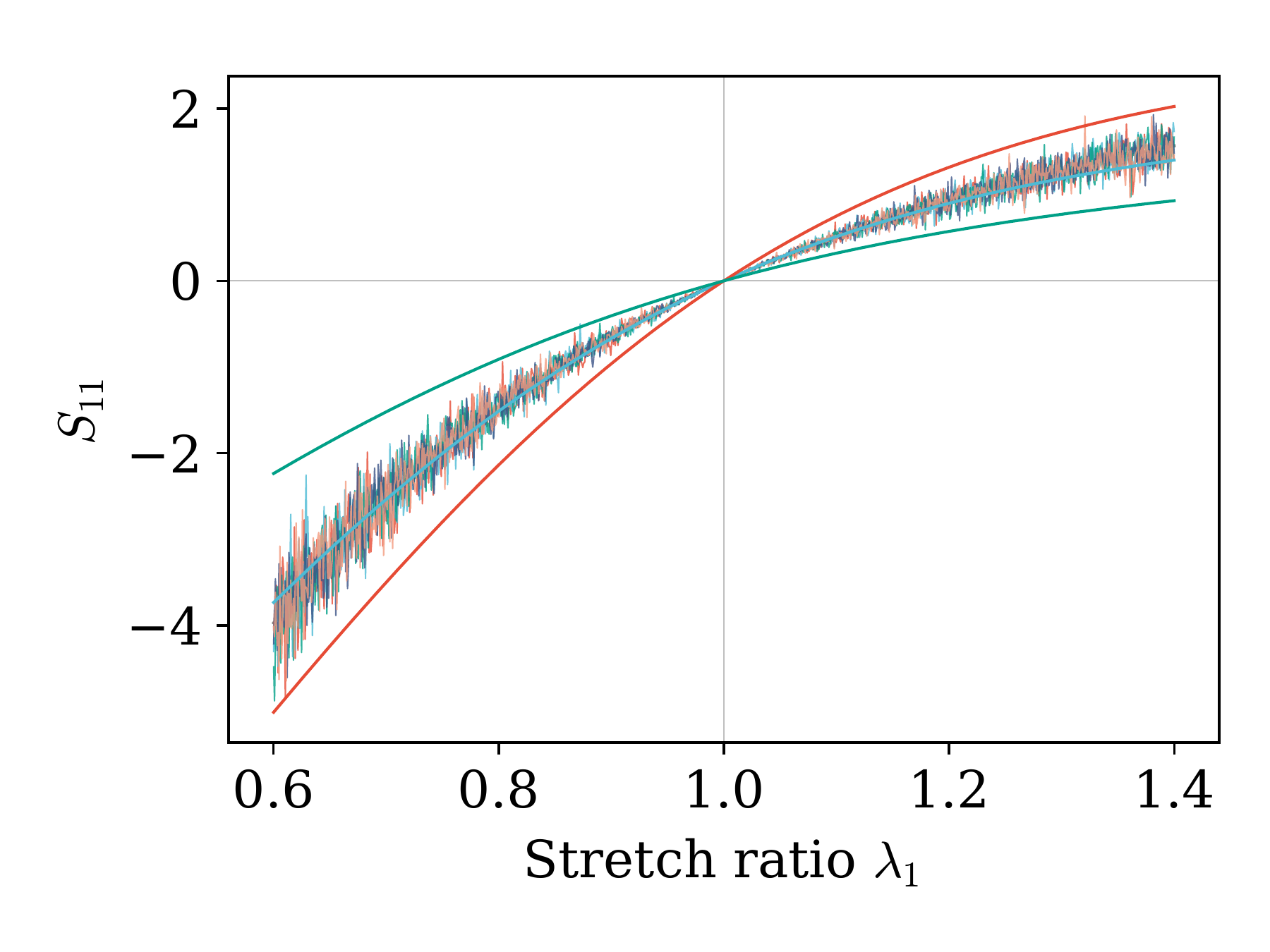}
  \caption{}
\end{subfigure}

\vspace{0.15cm}

\begin{subfigure}[t]{0.48\textwidth}
  \centering
  \stresspanel{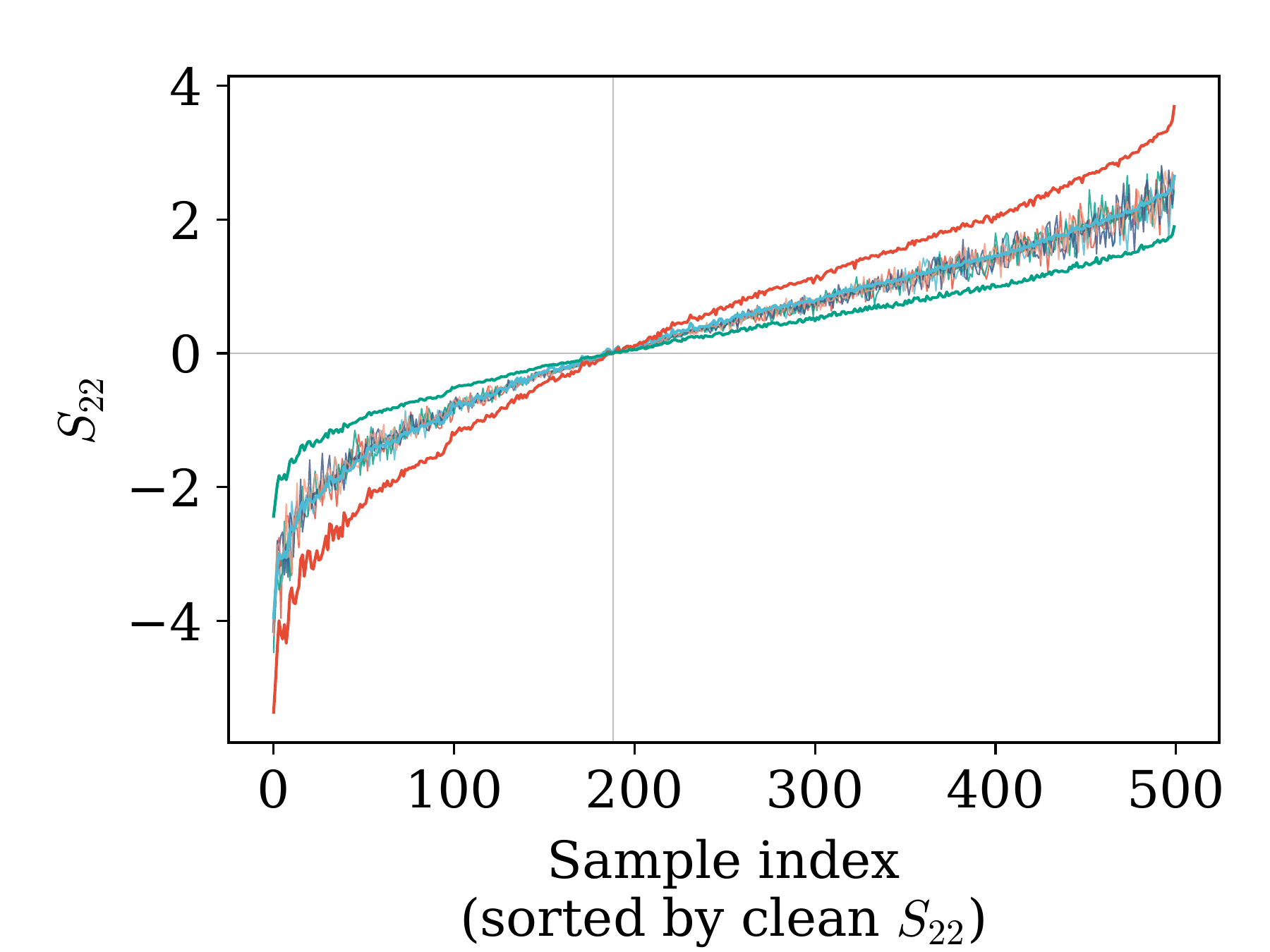}
  \caption{}
\end{subfigure}
\hfill
\begin{subfigure}[t]{0.48\textwidth}
  \centering
  \stresspanel{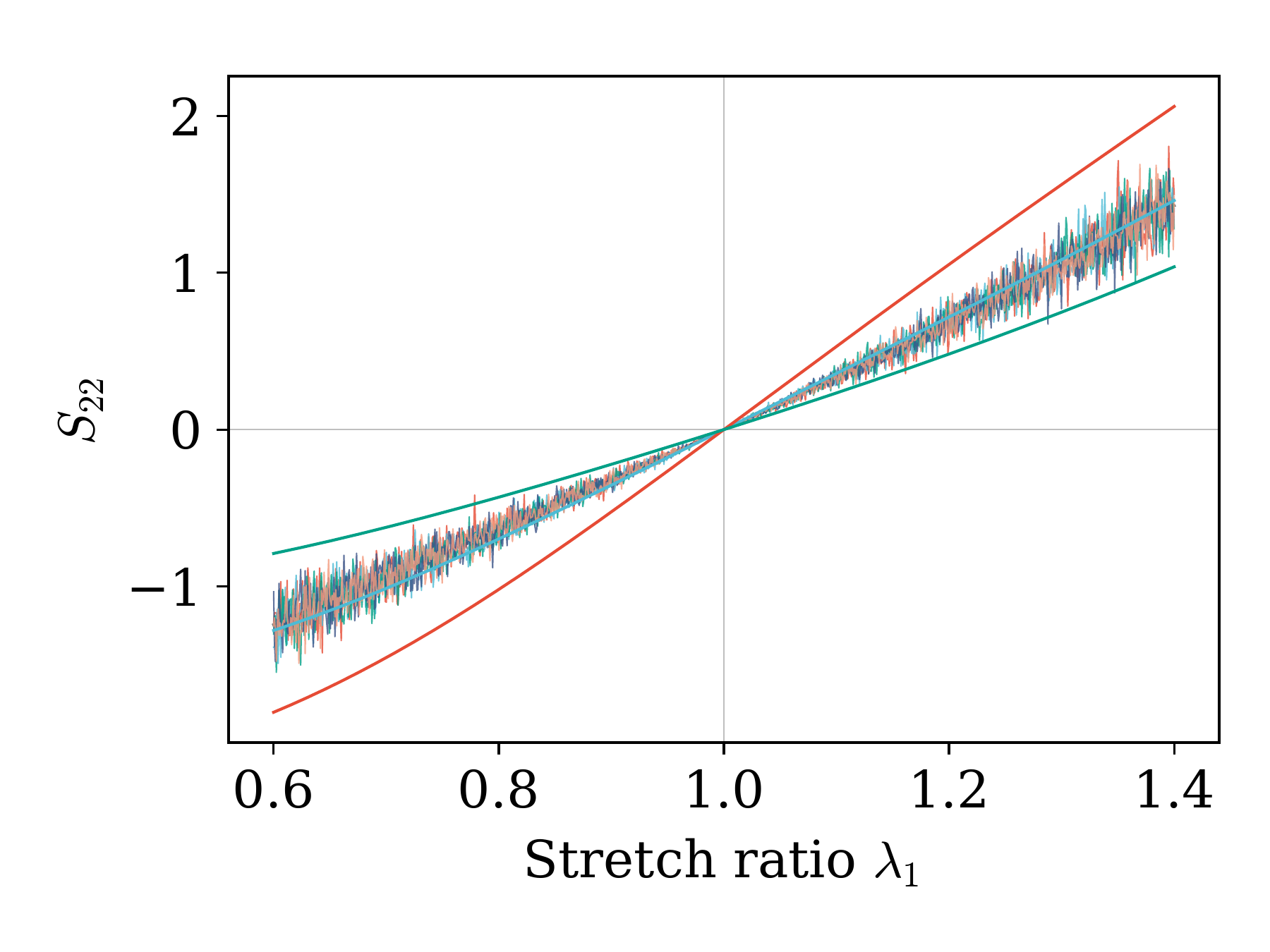}
  \caption{}
\end{subfigure}

\vspace{0.15cm}

\begin{subfigure}[t]{0.48\textwidth}
  \centering
  \stresspanel{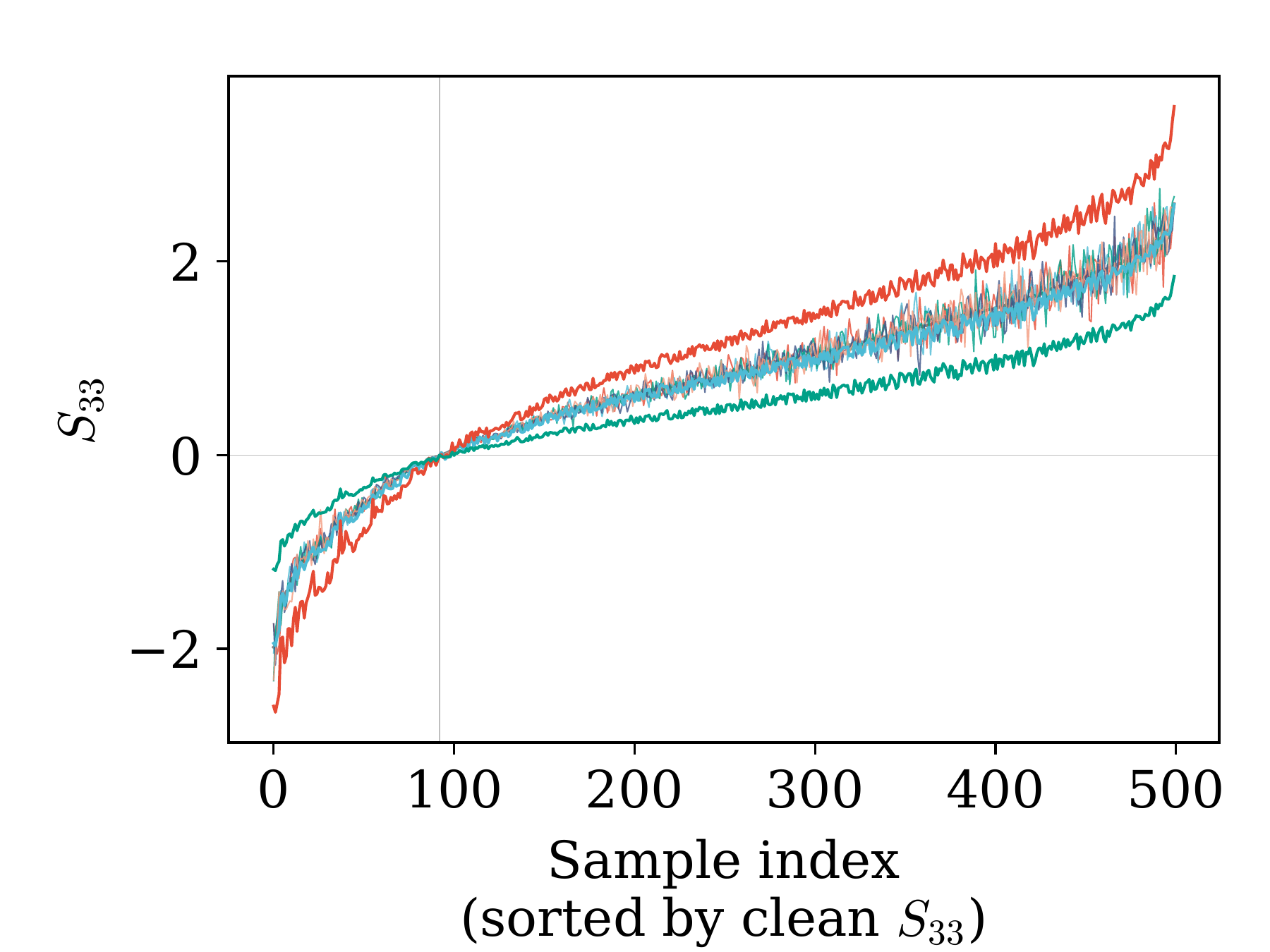}
  \caption{}
\end{subfigure}
\hfill
\begin{subfigure}[t]{0.48\textwidth}
  \centering
  \stresspanel{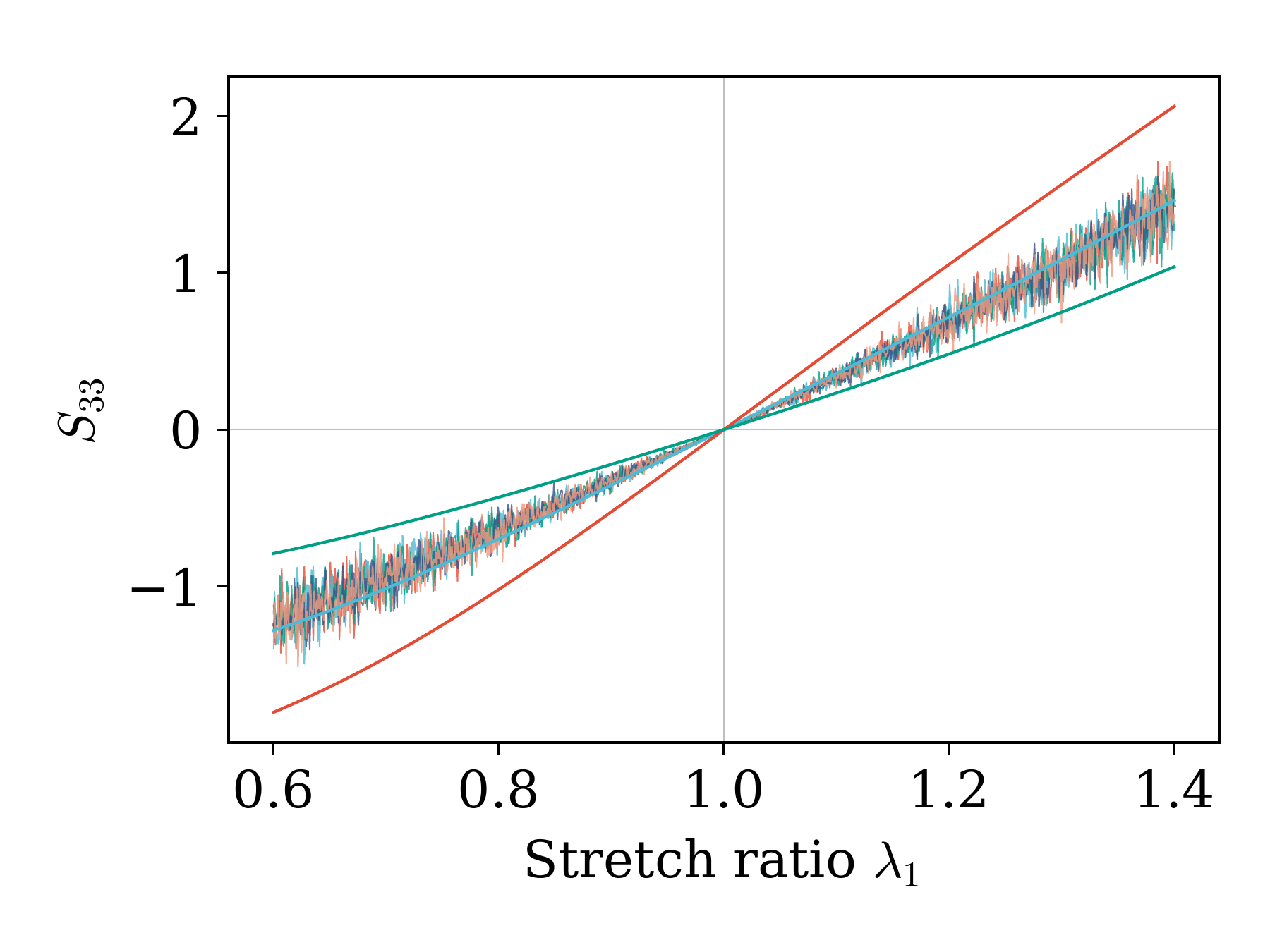}
  \caption{}
\end{subfigure}

\vspace{0.001in}
\begin{subfigure}[t]{\textwidth}
  \centering
  \stresslegend{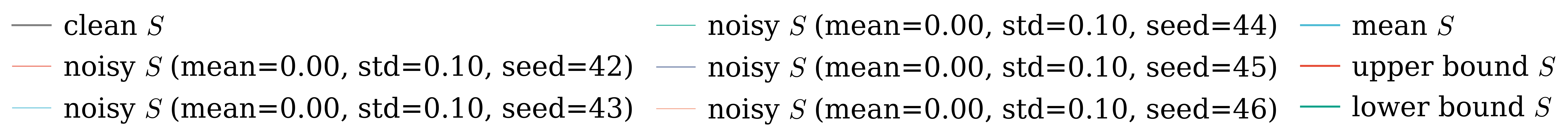}
\end{subfigure}
\stressfigcaption{\textbf{E2}: Stress components ($S_{11}$, $S_{22}$, $S_{33}$) on training (left) and test (right) data.}
\label{fig:noise_stress_seeds}
\end{figure}

\begin{figure}[htbp]
\centering
\begin{subfigure}[t]{0.48\textwidth}
  \centering
  \stresspanel{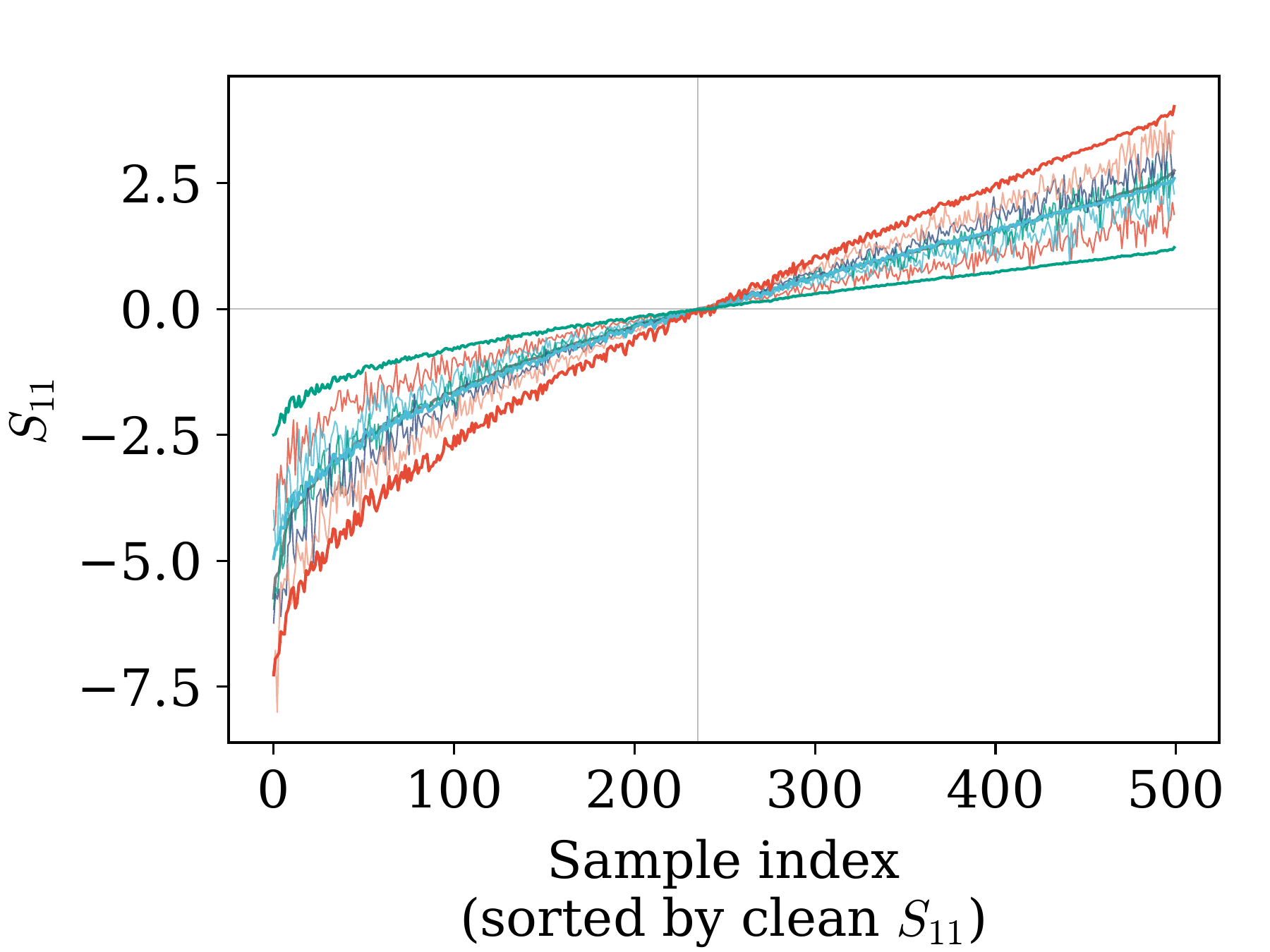}
  \caption{}
\end{subfigure}
\hfill
\begin{subfigure}[t]{0.48\textwidth}
  \centering
  \stresspanel{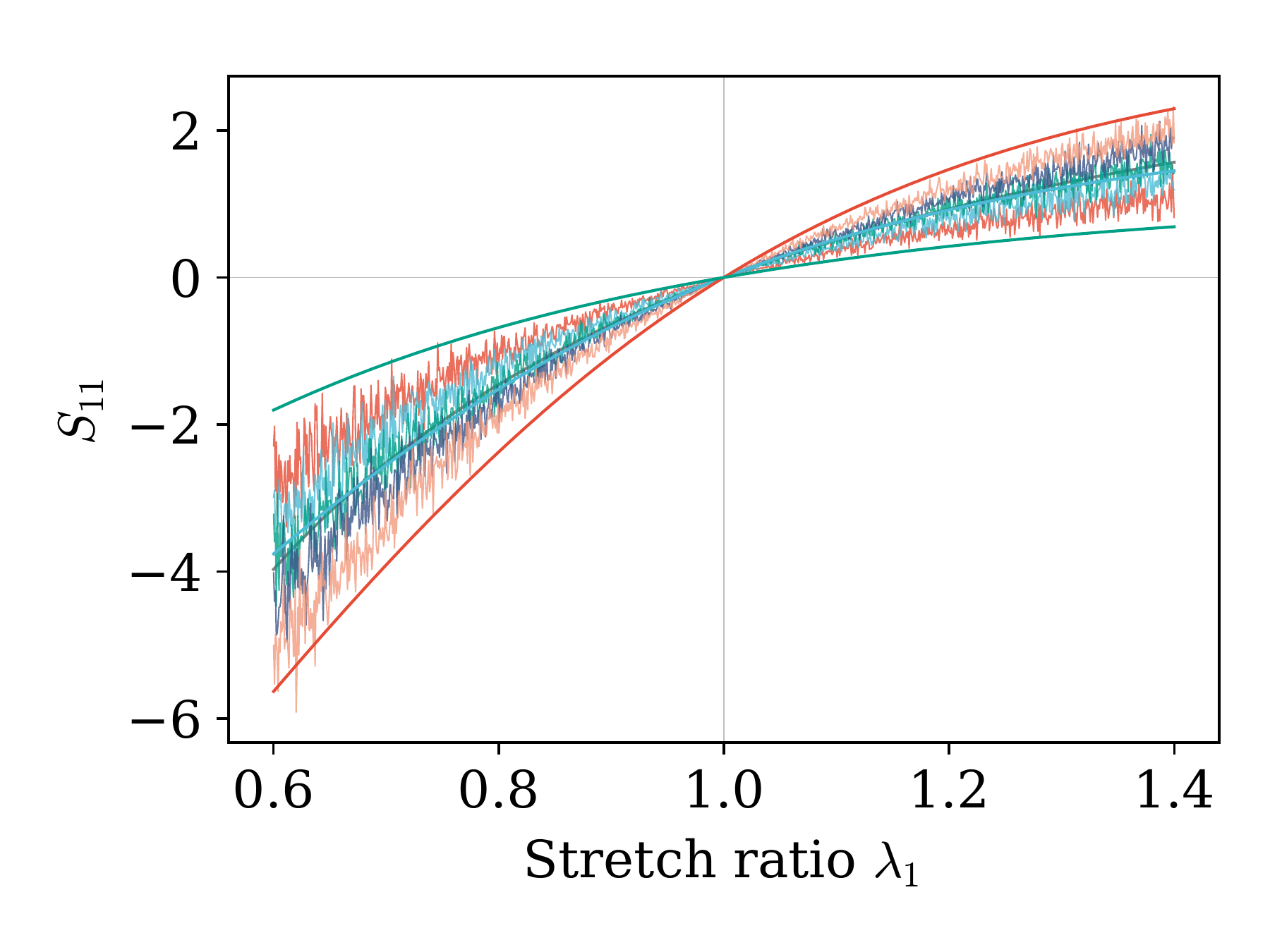}
  \caption{}
\end{subfigure}

\vspace{0.15cm}

\begin{subfigure}[t]{0.48\textwidth}
  \centering
  \stresspanel{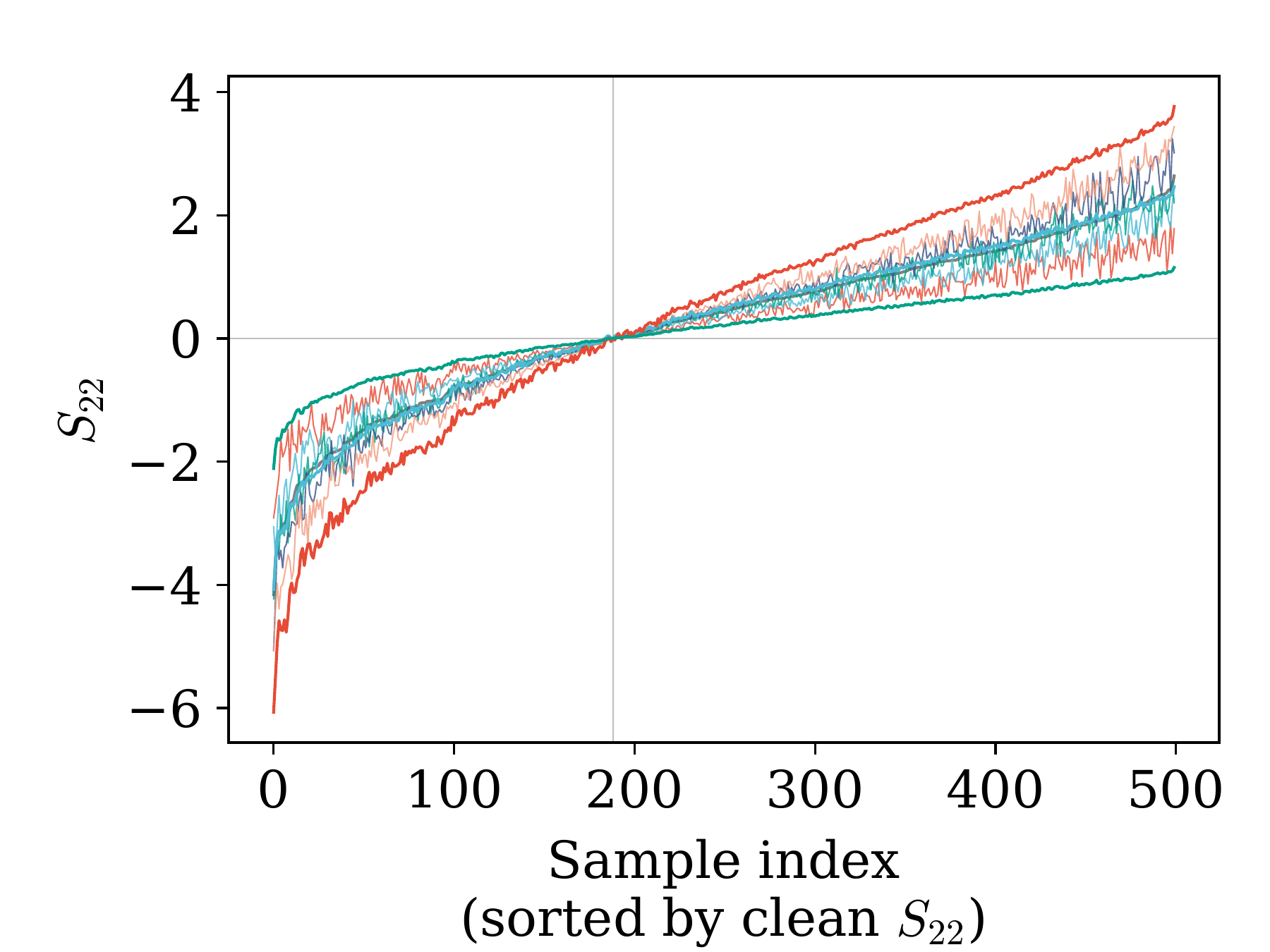}
  \caption{}
\end{subfigure}
\hfill
\begin{subfigure}[t]{0.48\textwidth}
  \centering
  \stresspanel{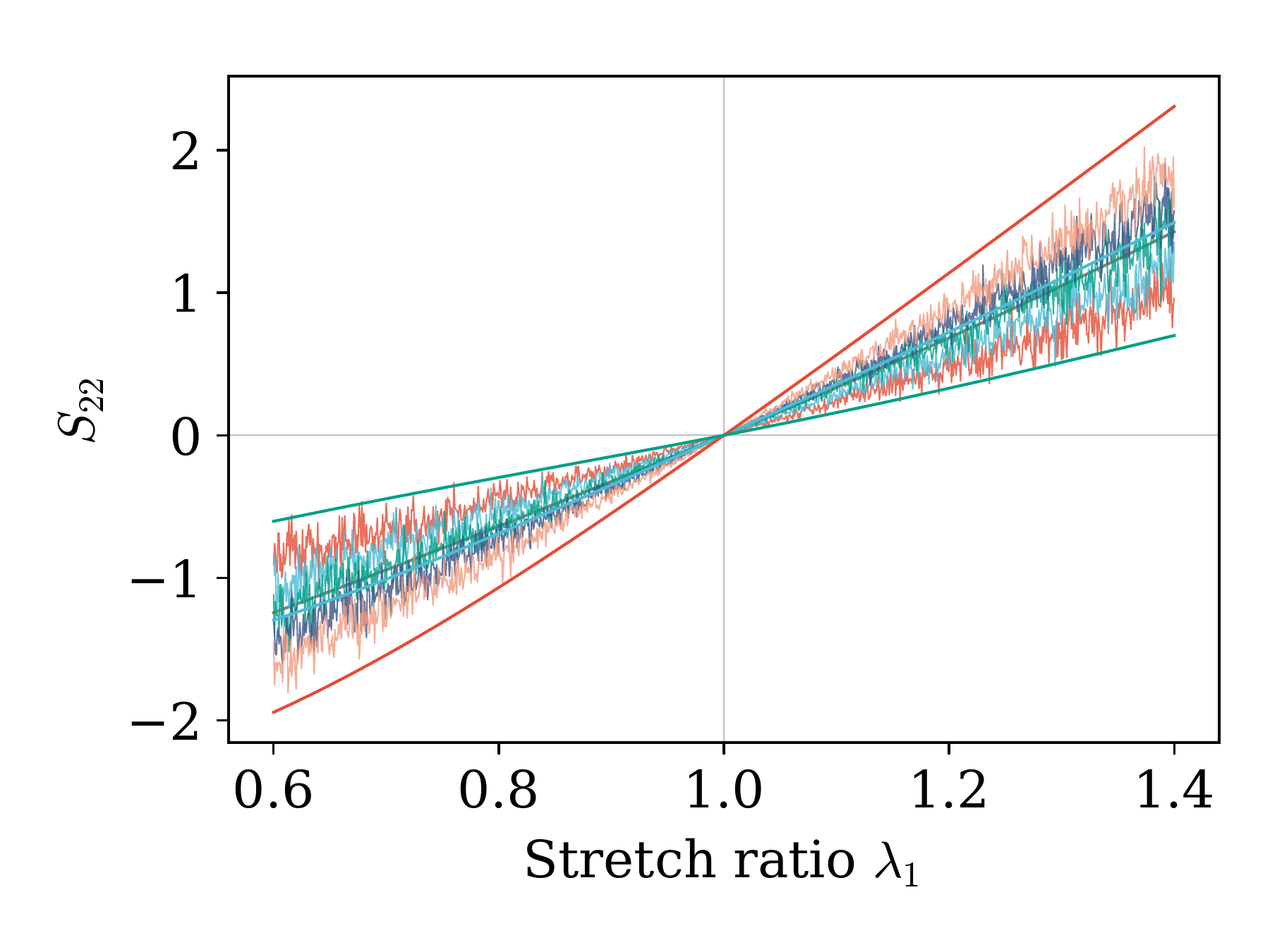}
  \caption{}
\end{subfigure}

\vspace{0.15cm}

\begin{subfigure}[t]{0.48\textwidth}
  \centering
  \stresspanel{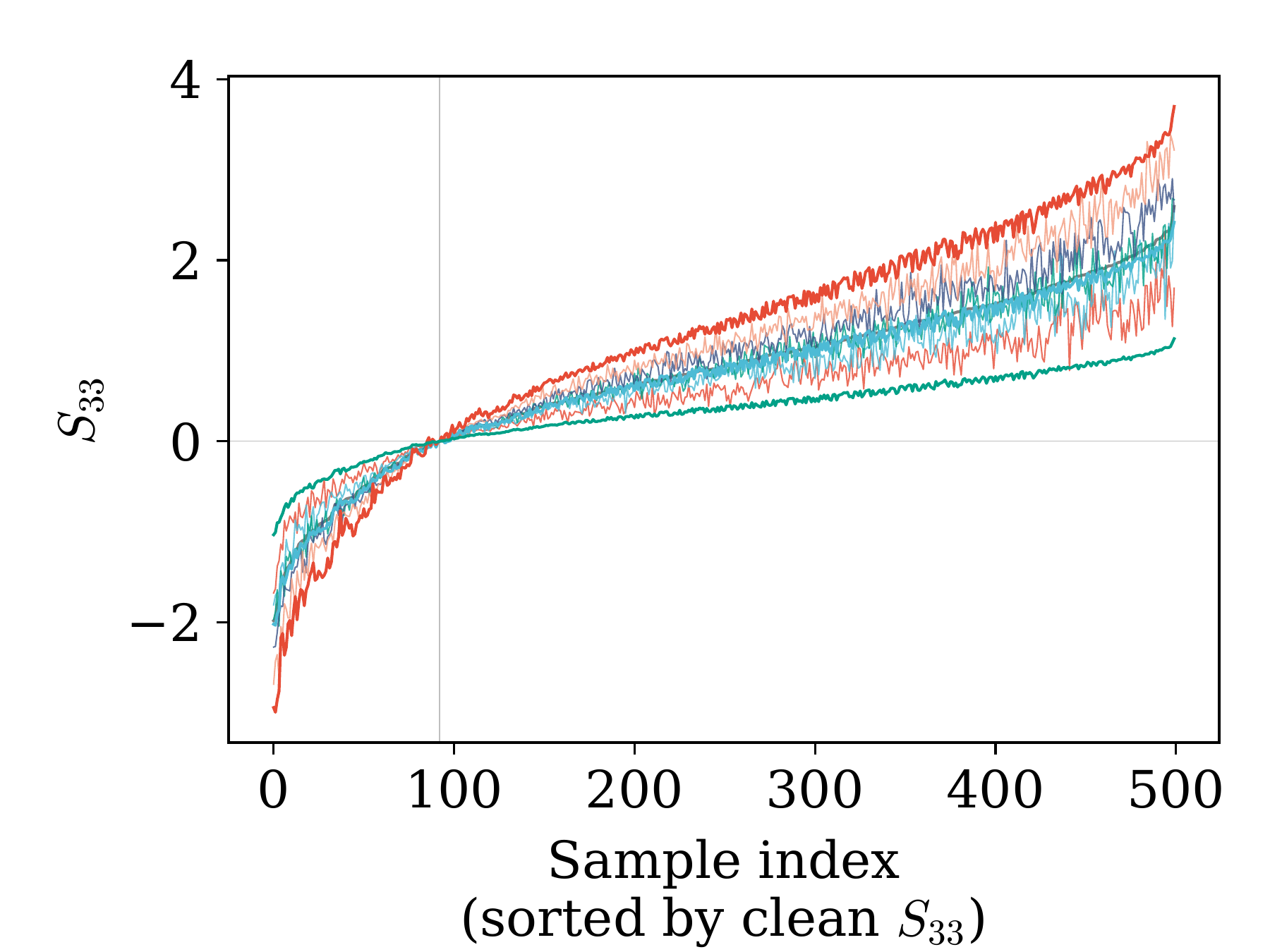}
  \caption{}
\end{subfigure}
\hfill
\begin{subfigure}[t]{0.48\textwidth}
  \centering
  \stresspanel{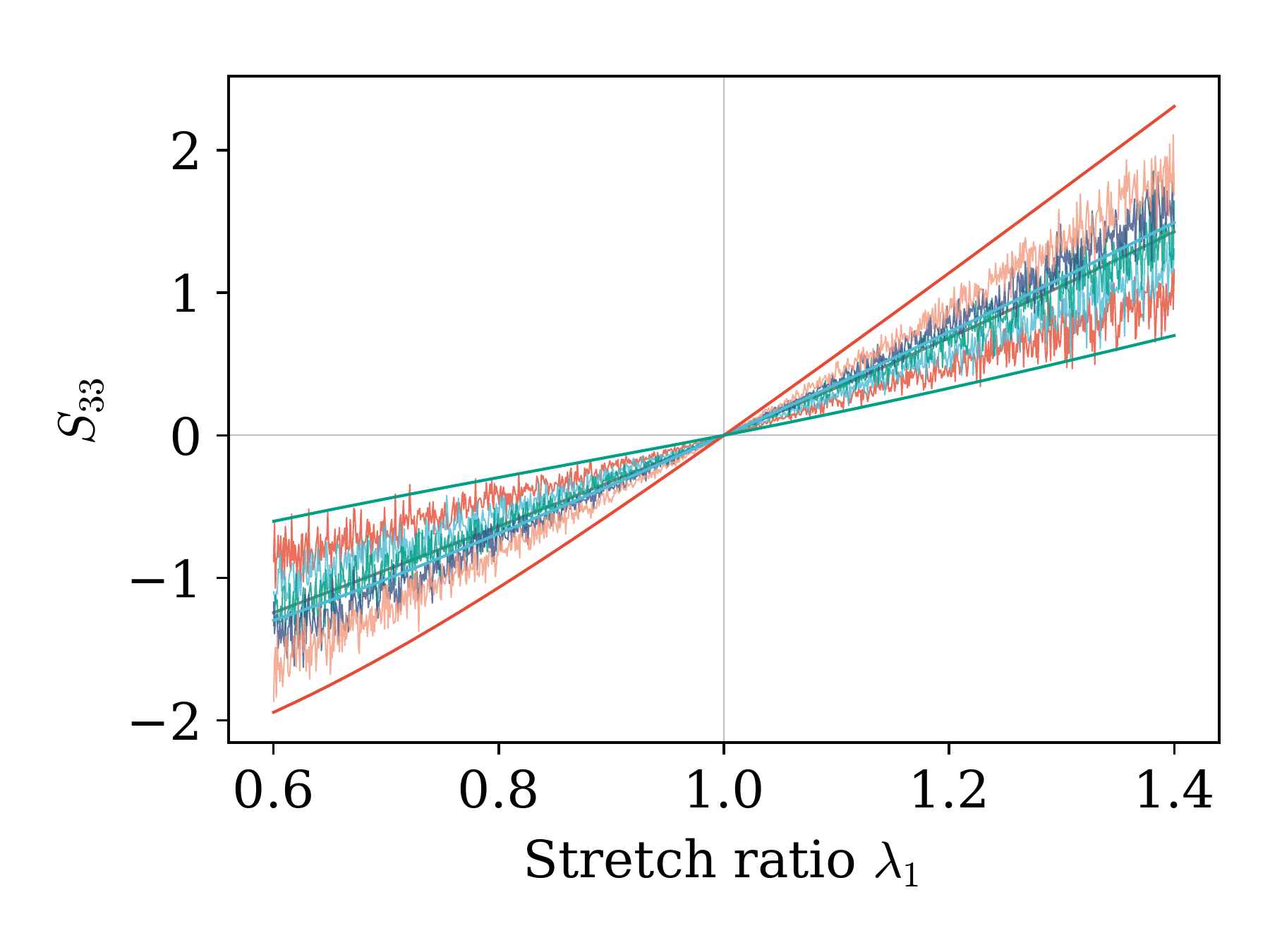}
  \caption{}
\end{subfigure}

\vspace{0.001in}
\begin{subfigure}[t]{\textwidth}
  \centering
  \stresslegend{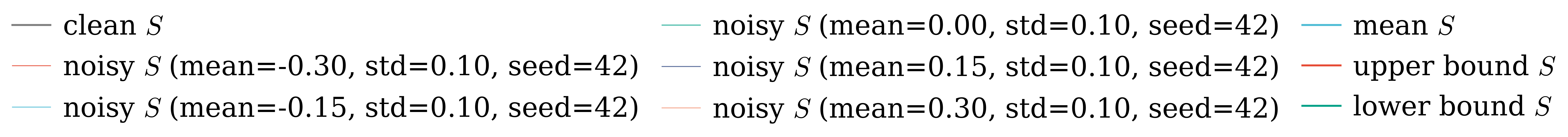}
\end{subfigure}
\stressfigcaption{\textbf{E3}: Stress components ($S_{11}$, $S_{22}$, $S_{33}$) on training (left) and test (right) data.}
\label{fig:noise_stress_means}
\end{figure}

\begin{figure}[htbp]
\centering
\begin{subfigure}[t]{0.48\textwidth}
  \centering
  \stresspanel{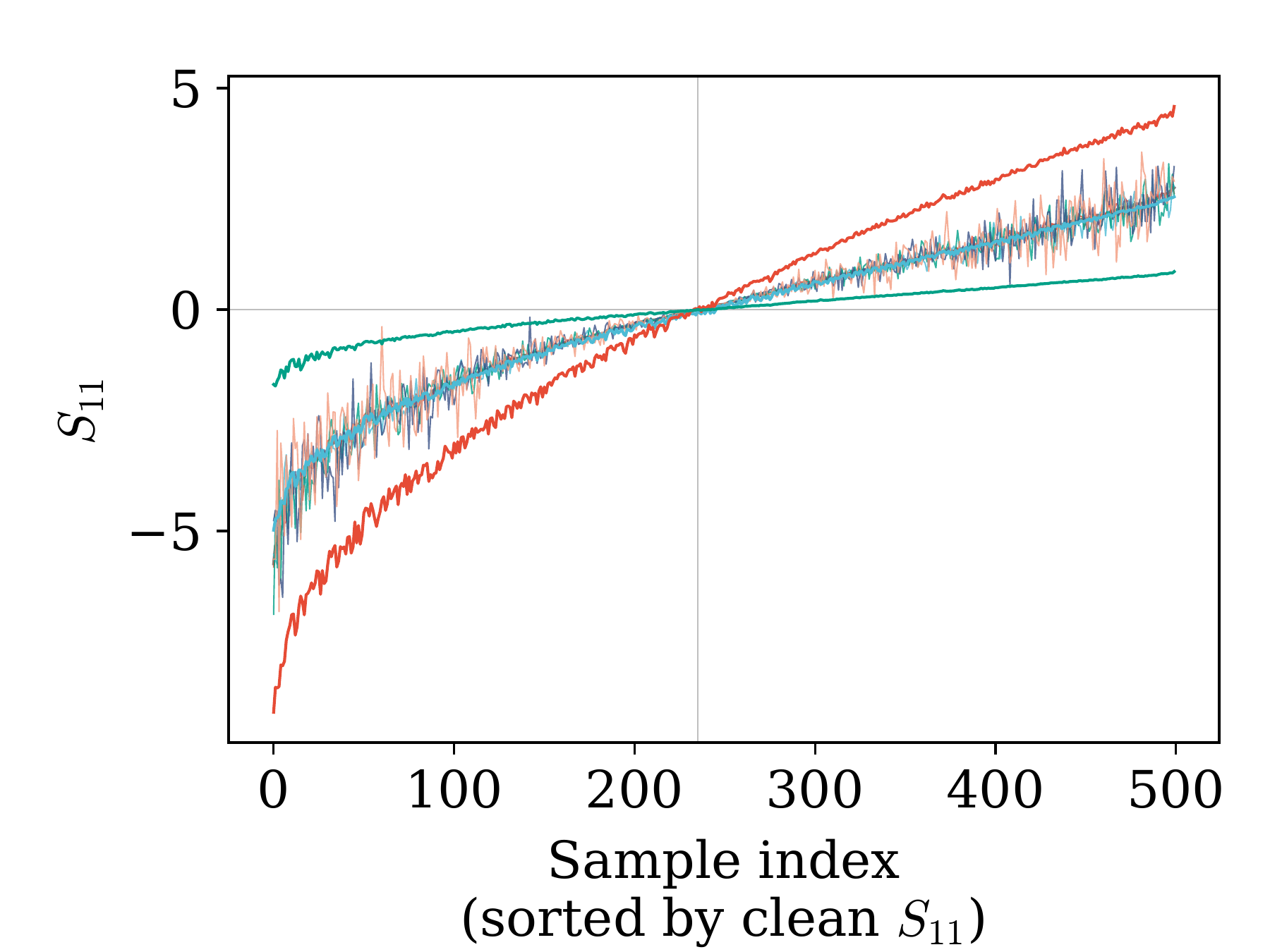}
  \caption{}
\end{subfigure}
\hfill
\begin{subfigure}[t]{0.48\textwidth}
  \centering
  \stresspanel{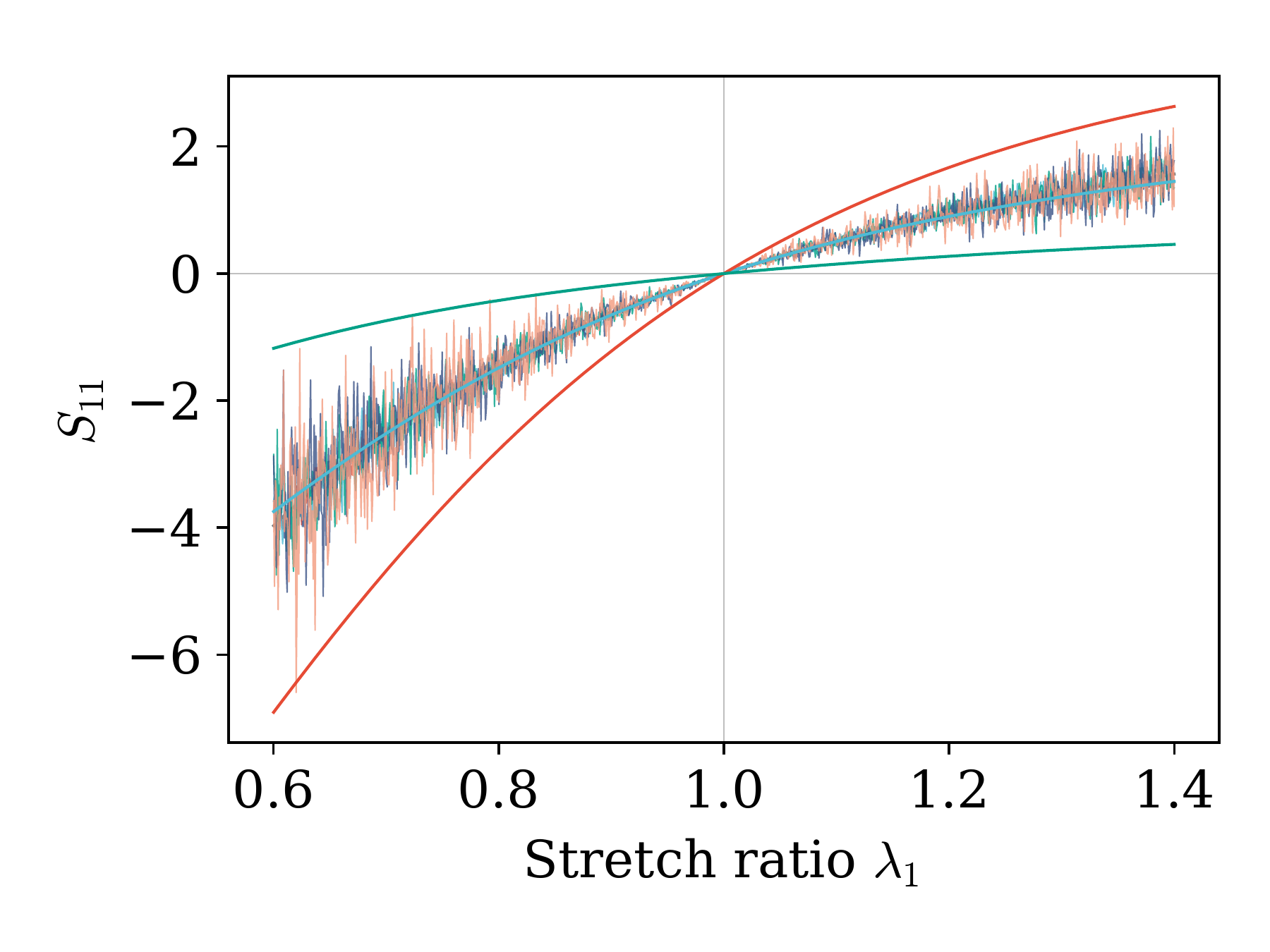}
  \caption{}
\end{subfigure}

\vspace{0.15cm}

\begin{subfigure}[t]{0.48\textwidth}
  \centering
  \stresspanel{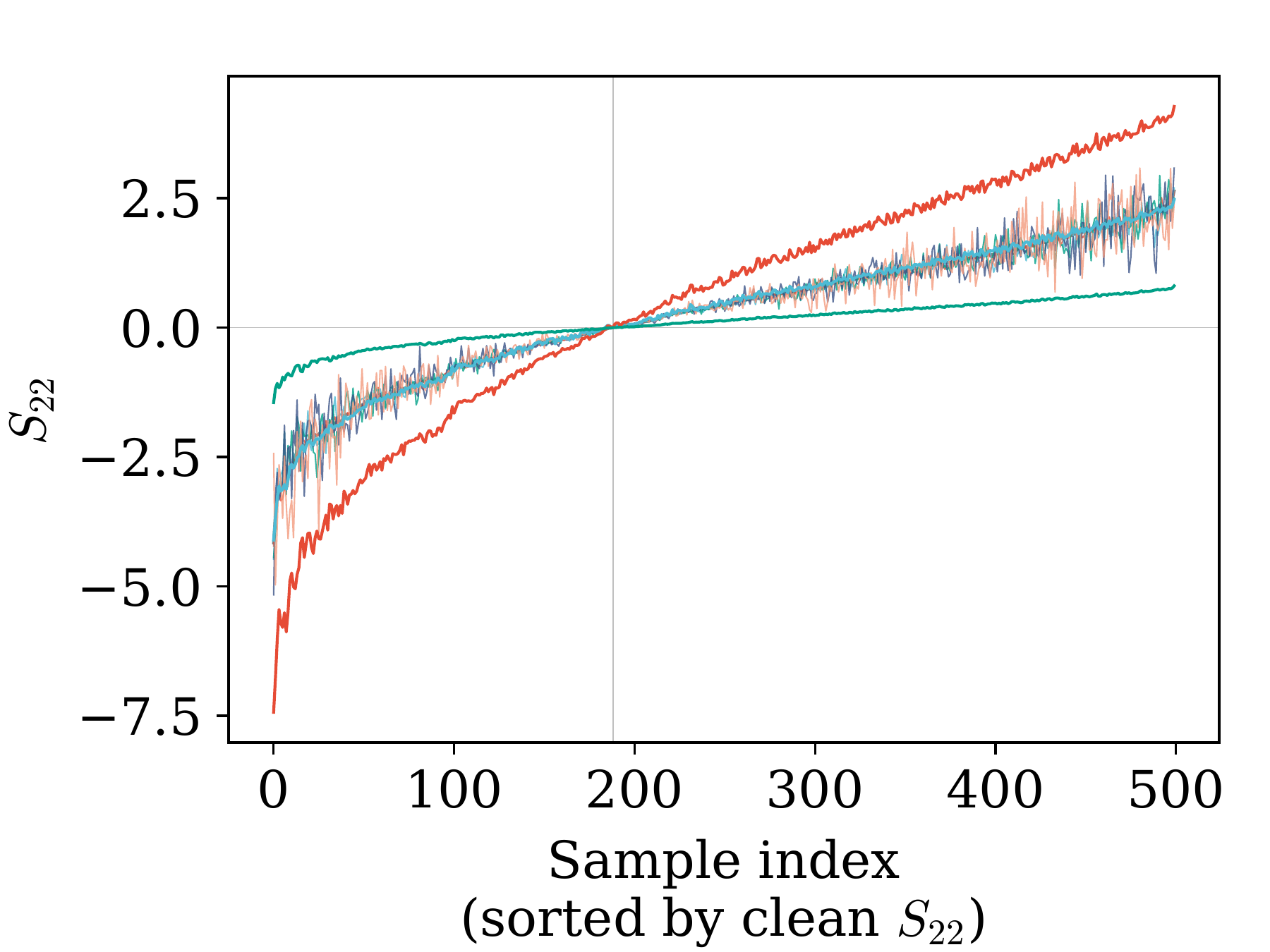}
  \caption{}
\end{subfigure}
\hfill
\begin{subfigure}[t]{0.48\textwidth}
  \centering
  \stresspanel{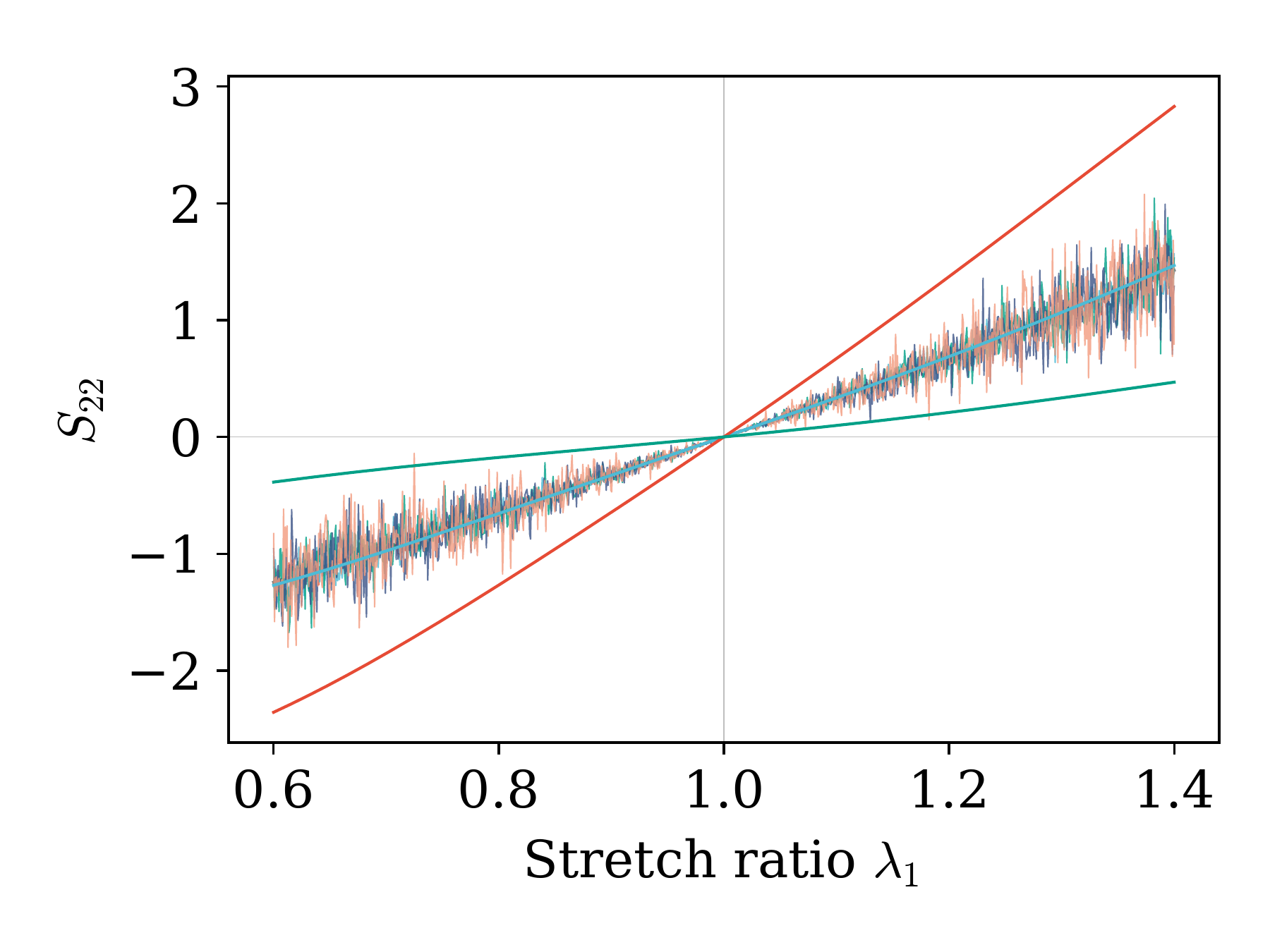}
  \caption{}
\end{subfigure}

\vspace{0.15cm}

\begin{subfigure}[t]{0.48\textwidth}
  \centering
  \stresspanel{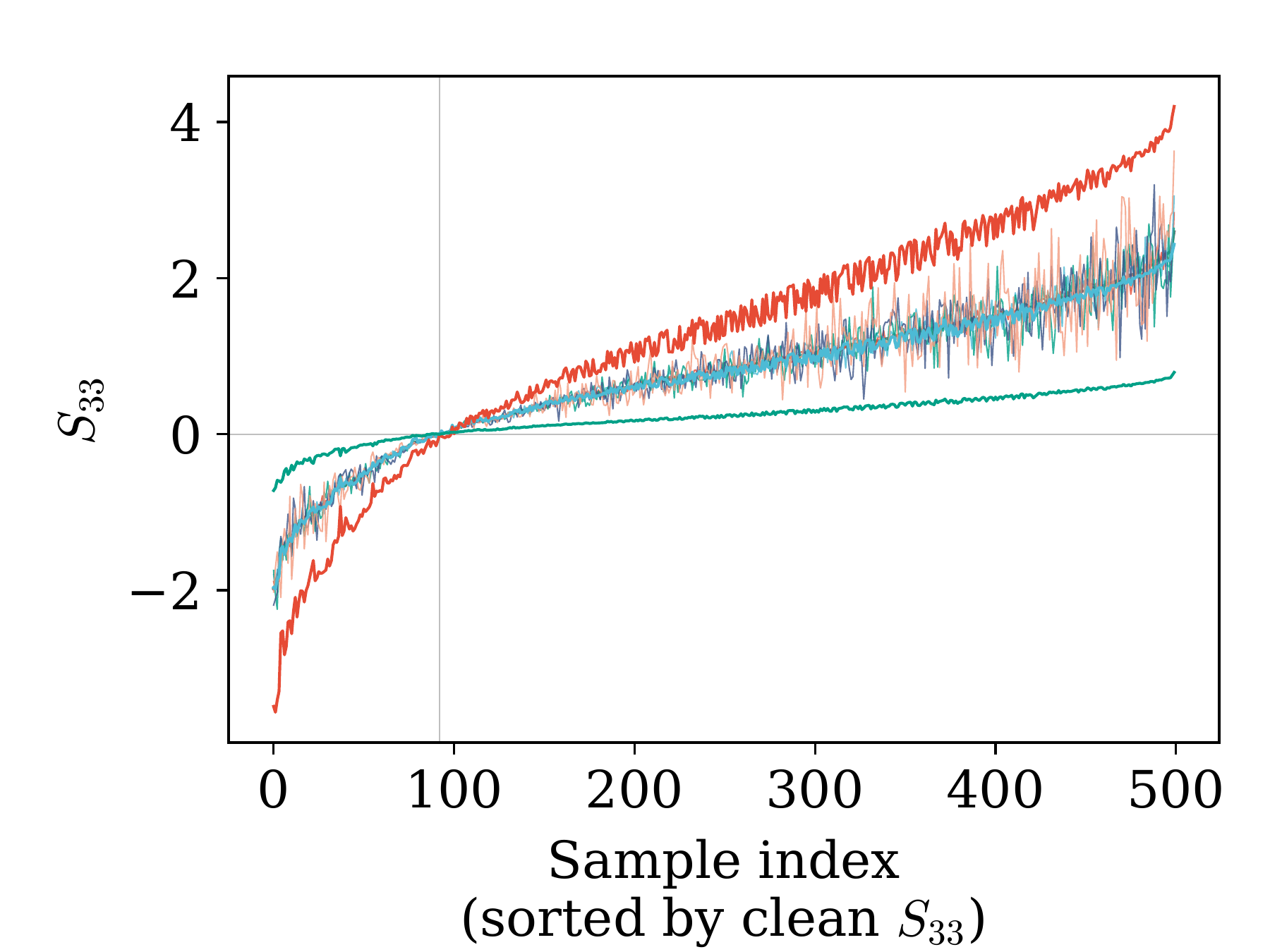}
  \caption{}
\end{subfigure}
\hfill
\begin{subfigure}[t]{0.48\textwidth}
  \centering
  \stresspanel{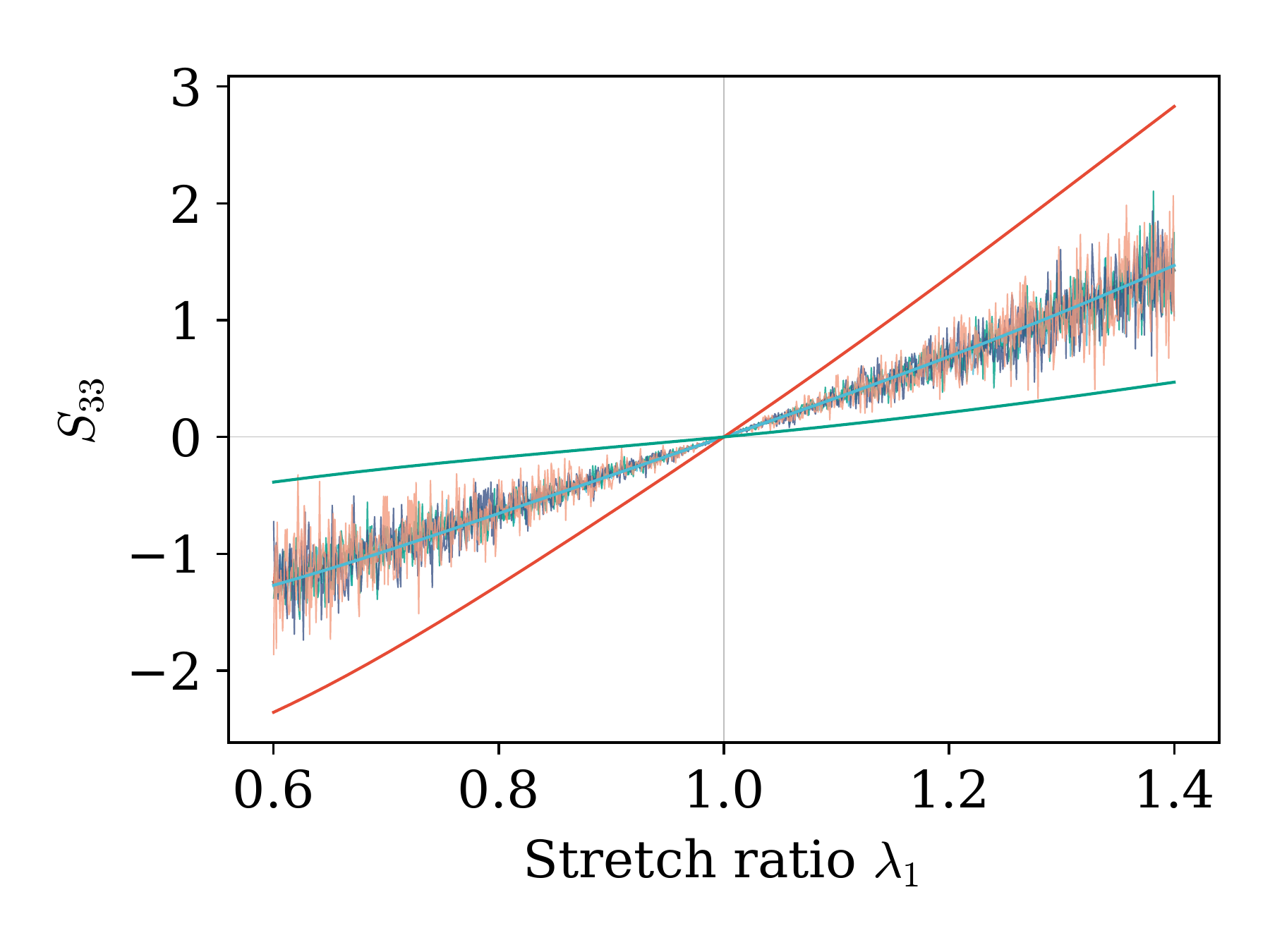}
  \caption{}
\end{subfigure}

\vspace{0.001in}
\begin{subfigure}[t]{\textwidth}
  \centering
  \stresslegend{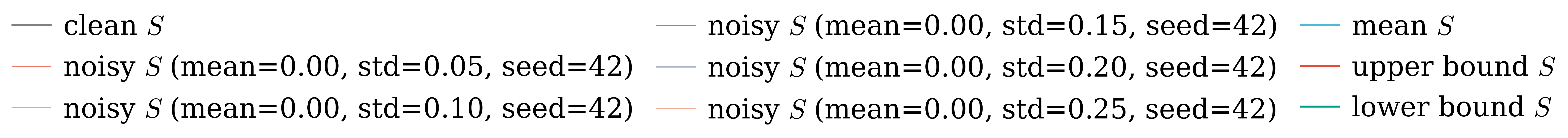}
\end{subfigure}
\stressfigcaption{\textbf{E4}: Stress components ($S_{11}$, $S_{22}$, $S_{33}$) on training (left) and test (right) data.}
\label{fig:noise_stress_std}
\end{figure}

\subsubsection{Fuzzy sets, $\alpha$-cuts and  membership analysis}\label{sec:membership_analysis}
Following the $\alpha$-cut construction of Sections~\ref{sec:fuzzy_theory} and \ref{sec:fem_membership}, 
we now quantify how well the learned fPANN enclosure contains the noisy stress observations for each of the 
experiment types reported above. Since, by construction of the data-generation and training stages, the off-diagonal components of the second Piola-Kirchhoff stress 
vanish identically and therefore carry no enclosure information, we restrict the analysis to the three 
diagonal components. In continuation of Eq.~\eqref{eq:alphaSIcut}, the componentwise $\alpha$-cuts on 
the diagonal entries are written as
\begin{equation}
S^*_{ii,\,\alpha}(\mathbf{F}) = \left[\,\underline{S}^*_{ii,\,\alpha}(\mathbf{F}),\;\overline{S}^*_{ii,\,\alpha}(\mathbf{F})\,\right],\qquad i=1,2,3,\quad \alpha\in(0,1],
\label{eq:alphaSij_cut}
\end{equation}
where the lower and upper bounds are obtained by applying Eq.~\eqref{s_lb_ub2} 
to the interpolated free energies $\underline{\Psi}^*_\alpha,\,\overline{\Psi}^*_\alpha$ of 
Eqs.~\eqref{eq:interpolation1} and~\eqref{eq:interpolation2}.

Given $N$ noisy samples $\{S^{(n)}_{ii,\,\mathrm{noisy}}\}_{n=1}^{N}$ at deformations $\{\mathbf{F}^{(n)}\}_{n=1}^{N}$, we quantify how completely the $\alpha$-cut stress enclosure $S^*_{ii,\,\alpha}(\mathbf{F})$ of Eq.~\eqref{eq:alphaSij_cut} contains the observations. For each diagonal component $i=1,2,3$, define
\begin{equation}
b_{ii}(\alpha)=\frac{1}{N}\sum_{n=1}^{N}\mathbf{1}\!\left\{
S^{(n)}_{ii,\,\mathrm{noisy}}\in S^*_{ii,\,\alpha}\!\left(\mathbf{F}^{(n)}\right)
\right\},\qquad i=1,2,3,
\label{eq:b_ii}
\end{equation}
where $\mathbf{1}$ denotes the indicator function. Thus, $b_{ii}(\alpha)$ is the fraction of samples whose noisy $S_{ii}$ lies inside the predicted $\alpha$-cut enclosure.
Averaging over the three diagonal components we define the \emph{empirical membership fraction} (emf)
\begin{equation}
b_{\mathrm{emf}}(\alpha)=\tfrac{1}{3}\big(b_{11}(\alpha)+b_{22}(\alpha)+b_{33}(\alpha)\big).
\label{eq:b_emf}
\end{equation}
When an experiment comprises $K$ independent noise realizations (as in \textbf{E2}, \textbf{E3} and \textbf{E4}),
$N$ counts all realization--sample pairs, so $N=K\,N_{\mathrm{def}}$, where $N_{\mathrm{def}}$ is the number of 
training or test samples of Section~\ref{sec:results_500} ($N_{\mathrm{def}}=500$ for training and $N_{\mathrm{def}}=1000$ for testing).

Applying Eqs.~\eqref{eq:b_ii} and~\eqref{eq:b_emf} to the training and test set yields $b_{\mathrm{emf,train}}(\alpha)$ 
and $b_{\mathrm{emf,test}}(\alpha)$, respectively.
In Figs.~\ref{fig:membership_combo_subset500} and~\ref{fig:membership_combo_diffseeds}, panels~(b)--(c) report these 
quantities as $b_{\mathrm{emf}}$ (\%) (i.e., $100\,b_{\mathrm{emf}}(\alpha)$), while panel~(d) 
shows the absolute train--test gap $|\Delta b_{\mathrm{emf}}|\equiv 100\,\big|b_{\mathrm{emf,train}}(\alpha)-b_{\mathrm{emf,test}}(\alpha)\big|$.

We plot these quantities against the $|a|\in[0,1]$, 
related to the membership level by $\alpha = 1 - |a|$, as illustrated schematically in Fig.~\ref{fig:fuzzy_triangle}.
The two extremes correspond to the trivial
$\alpha$-cuts: $|a|=0\ (\alpha=1)$ collapses the enclosure to the mean response $\Psi^*$, 
while $|a|=1\ (\alpha=0)$ recovers the widest interval $[\underline{\Psi}^*,\overline{\Psi}^*]$.
As shown in Section~\ref{sec:results_500} for \textbf{E1} (and analogously for \textbf{E2}--\textbf{E4}), the mean branch 
closely tracks the clean Gent-Gent reference stress across all diagonal components on both training and 
test data; this agreement is central to the membership interpretation, because the innermost 
$\alpha$-cut is anchored on the mean response. 
The panels in Figs.~\ref{fig:membership_combo_subset500} 
and~\ref{fig:membership_combo_diffseeds} therefore overlay the clean reference, the learned mean, 
the noisy observations, and the family of $\alpha$-cut bounds.

We illustrate the membership analysis for the representative experiments \textbf{E1} and \textbf{E2}; 
the remaining noise variants (\textbf{E3} and \textbf{E4}) exhibit qualitatively similar behavior and 
are omitted to avoid redundancy. In both experiments, the enclosures widen monotonically as $|a|$ increases 
(equivalently, as $\alpha$ decreases), and the empirical membership fraction approaches $100\%$ as $\alpha\to 0$.
The synthetic noise in Eq.~\eqref{eq:hetero_noise} is Gaussian, the curves $b_{\mathrm{emf}}(|a|)$ 
in panels~(b)--(c) of Figs.~\ref{fig:membership_combo_subset500} and~\ref{fig:membership_combo_diffseeds} take on the
characteristic sigmoidal-like shape of non-negative branch of the error-function profile. 
The coverage rises rapidly at small $|a|$ and then saturates on extremes.
Another important point worth noting is that the widest $\alpha$-cut ($|a|=1$) 
does not appear especially tight relative to the noisy samples, yet panels~(b)--(c) show that a large fraction of
the observations is captured only for $|a|$ close to $1$.
This is expected for approximately Gaussian residuals.
Intermediate cuts (e.g., $|a|=0.5$) leave out a nontrivial fraction even when the extreme bounds look only moderately wide.

\begin{figure}[H]
\centering
\membershiplegend{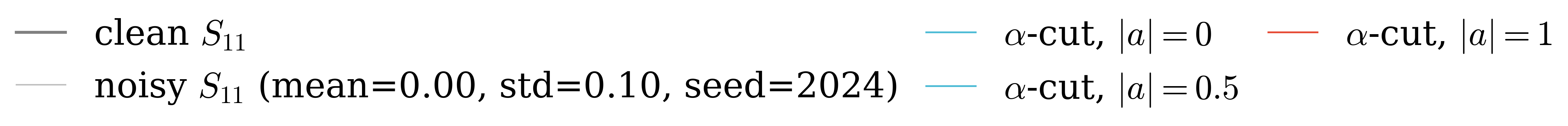}

\vspace{0.12cm}
\centering
\begin{subfigure}[t]{0.72\textwidth}
  \centering
  \membershipbounds{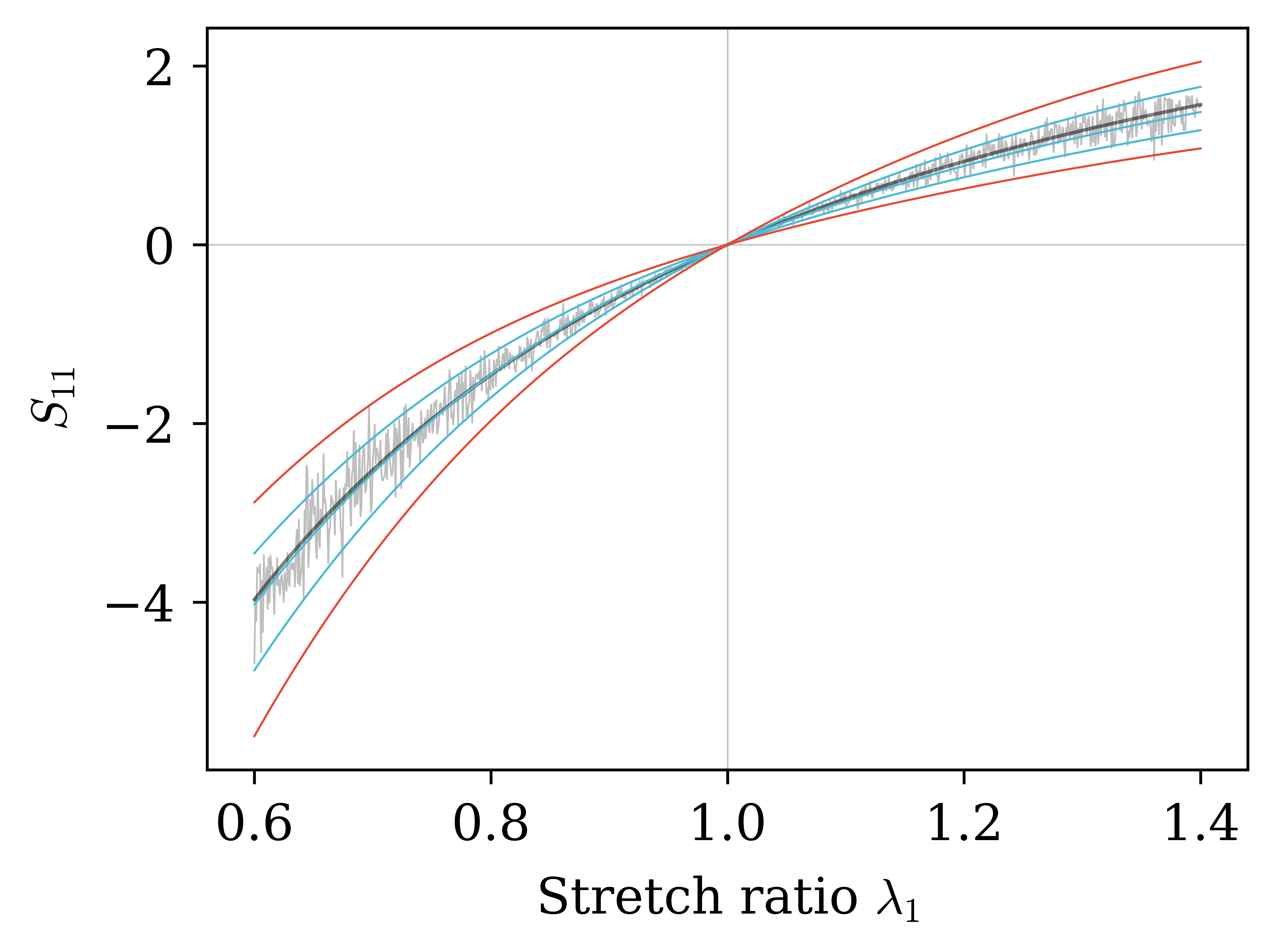}
  \caption{}
\end{subfigure}

\vspace{0.18cm}

\begin{subfigure}[t]{0.32\textwidth}
  \centering
  \maybeincludegraphics[width=\linewidth]{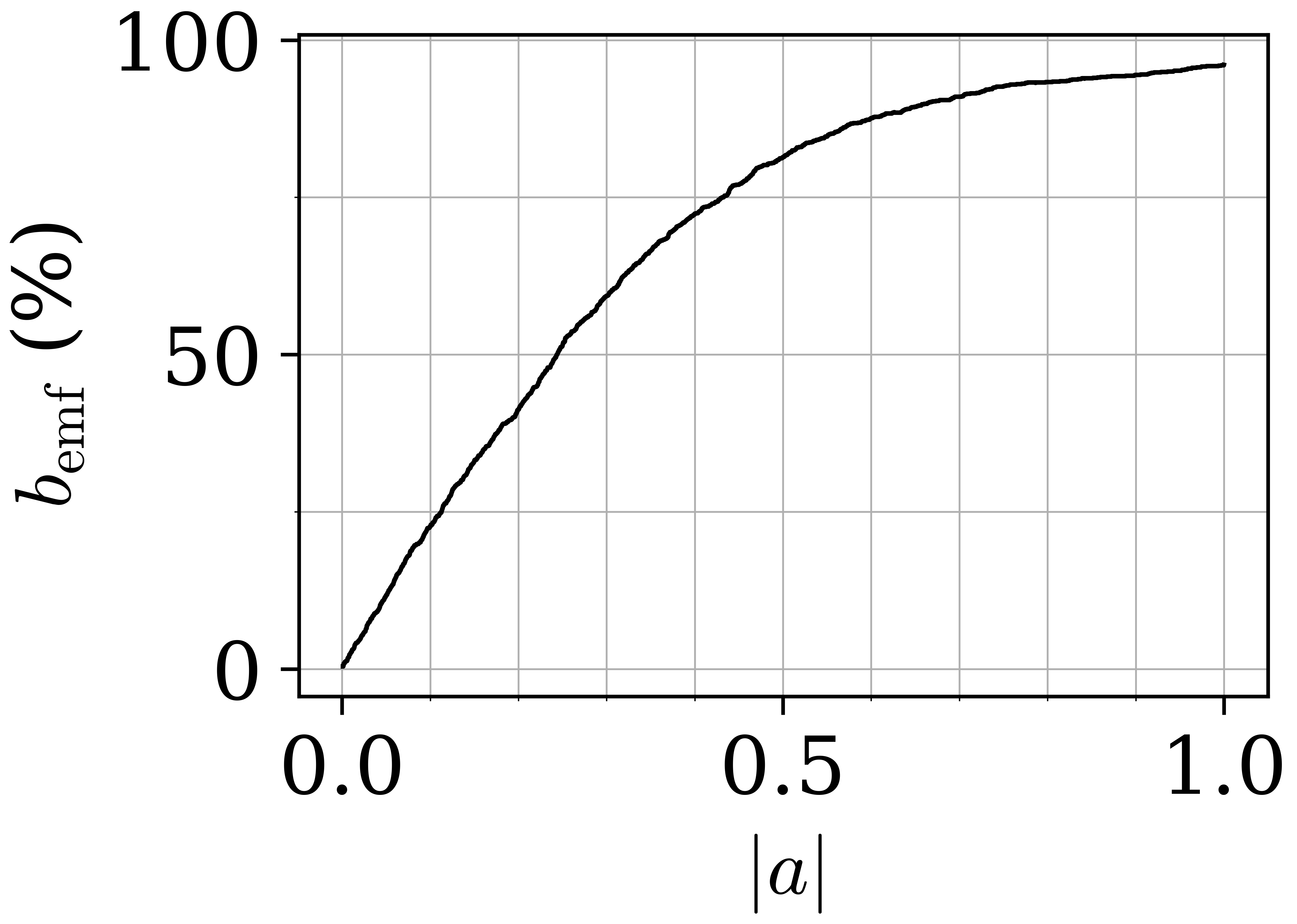}
  \caption{}
\end{subfigure}
\hfill
\begin{subfigure}[t]{0.32\textwidth}
  \centering
  \maybeincludegraphics[width=\linewidth]{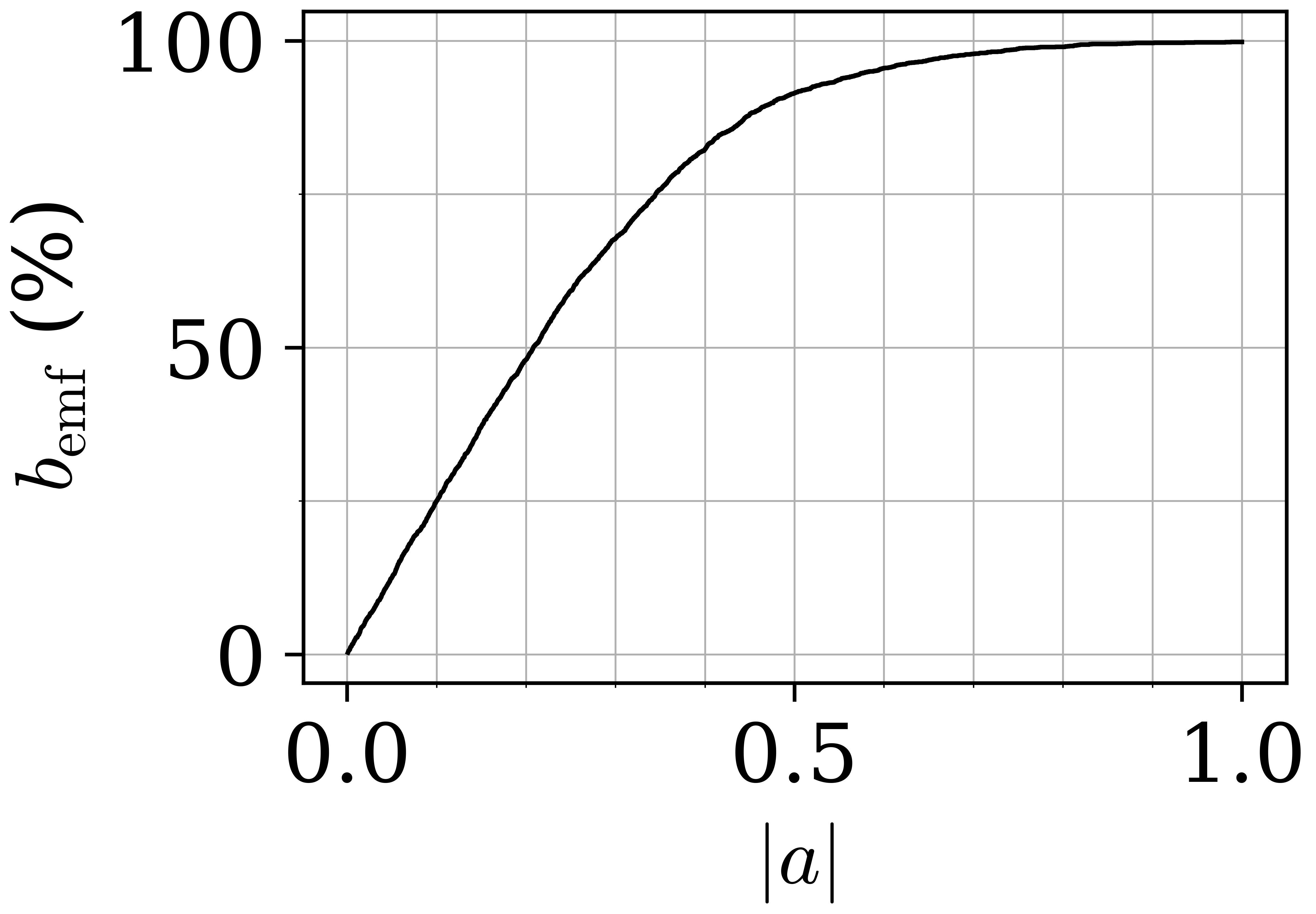}
  \caption{}
\end{subfigure}
\hfill
\begin{subfigure}[t]{0.32\textwidth}
  \centering
  \maybeincludegraphics[width=\linewidth]{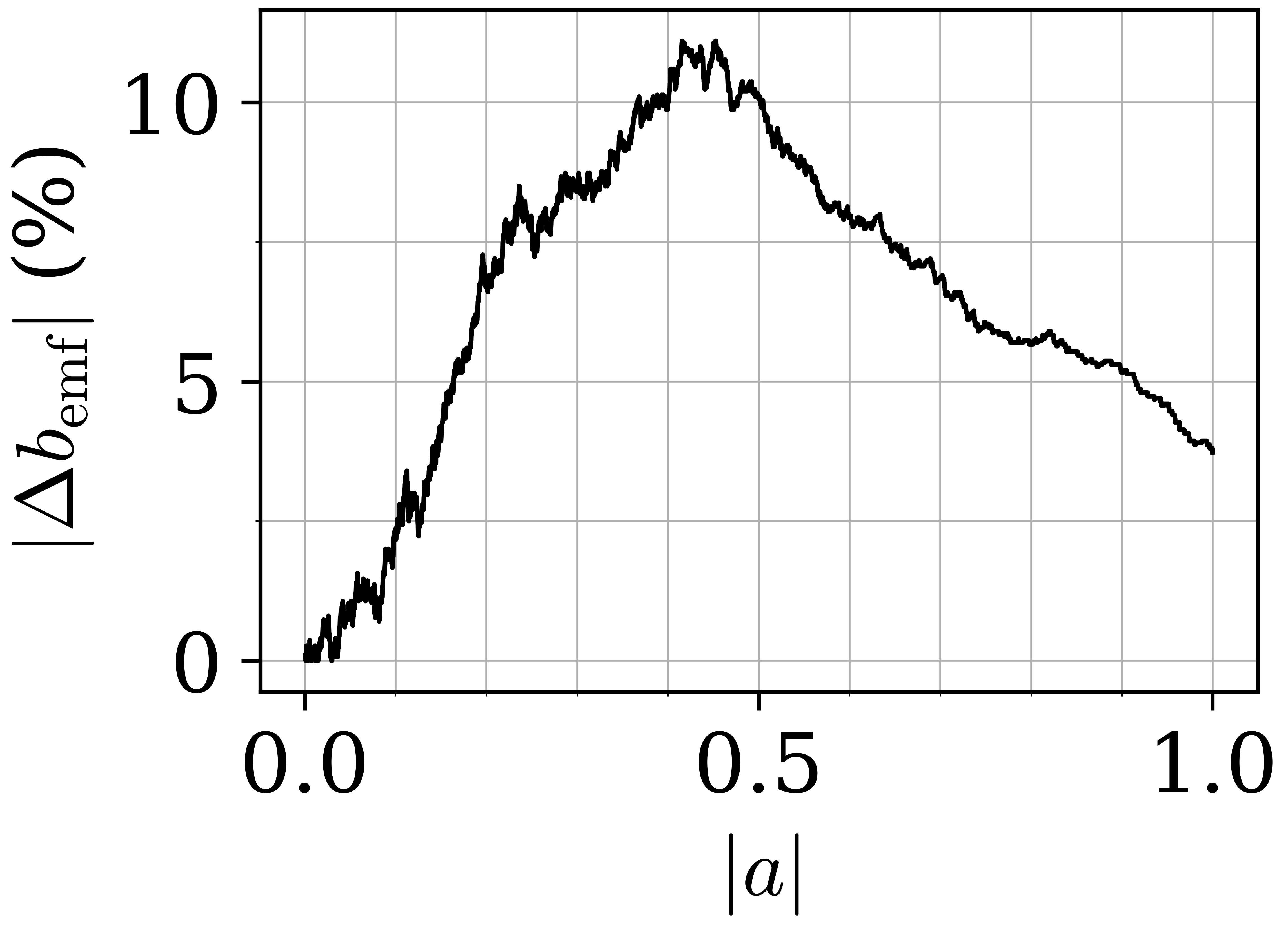}
  \caption{}
\end{subfigure}

\caption{\textbf{E1}: Membership analysis. (a) Test $S_{11}$ data and learned $\alpha$-cut bounds. (b)--(c) Training and test empirical membership fractions. (d) Absolute train--test gap.}
\label{fig:membership_combo_subset500}
\end{figure}

\begin{figure}[H]
\centering
\membershiplegend{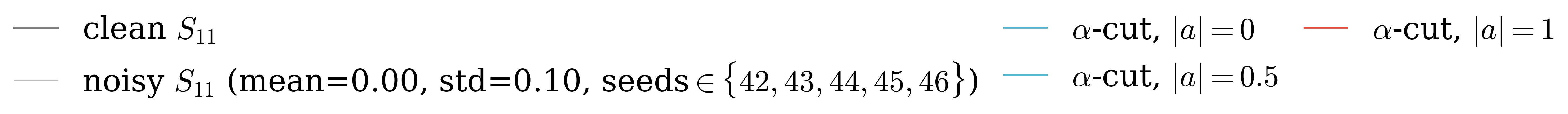}

\vspace{0.12cm}
\centering
\begin{subfigure}[t]{0.72\textwidth}
  \centering
  \membershipbounds{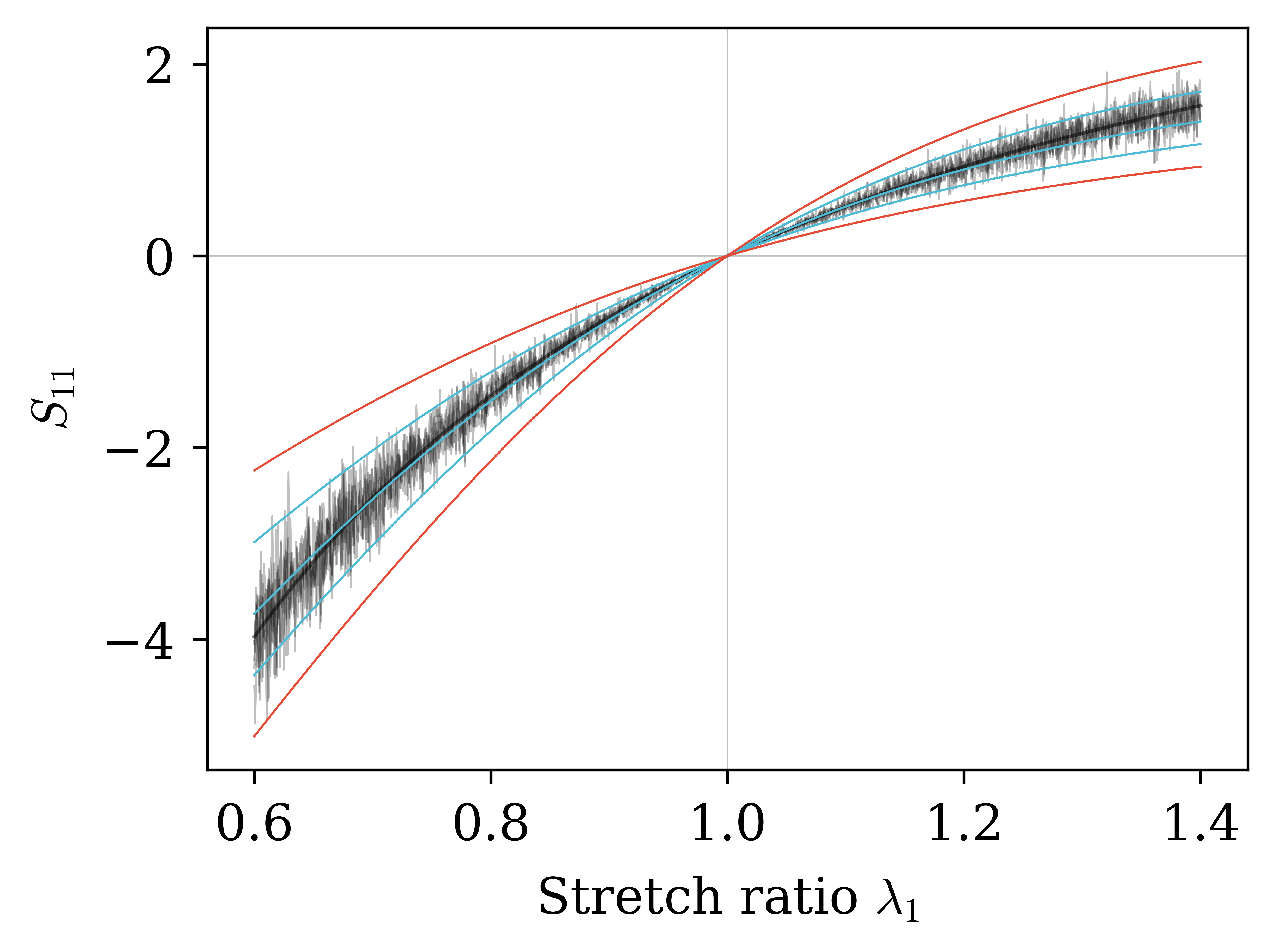}
  \caption{}
\end{subfigure}

\vspace{0.18cm}

\begin{subfigure}[t]{0.32\textwidth}
  \centering
  \maybeincludegraphics[width=\linewidth]{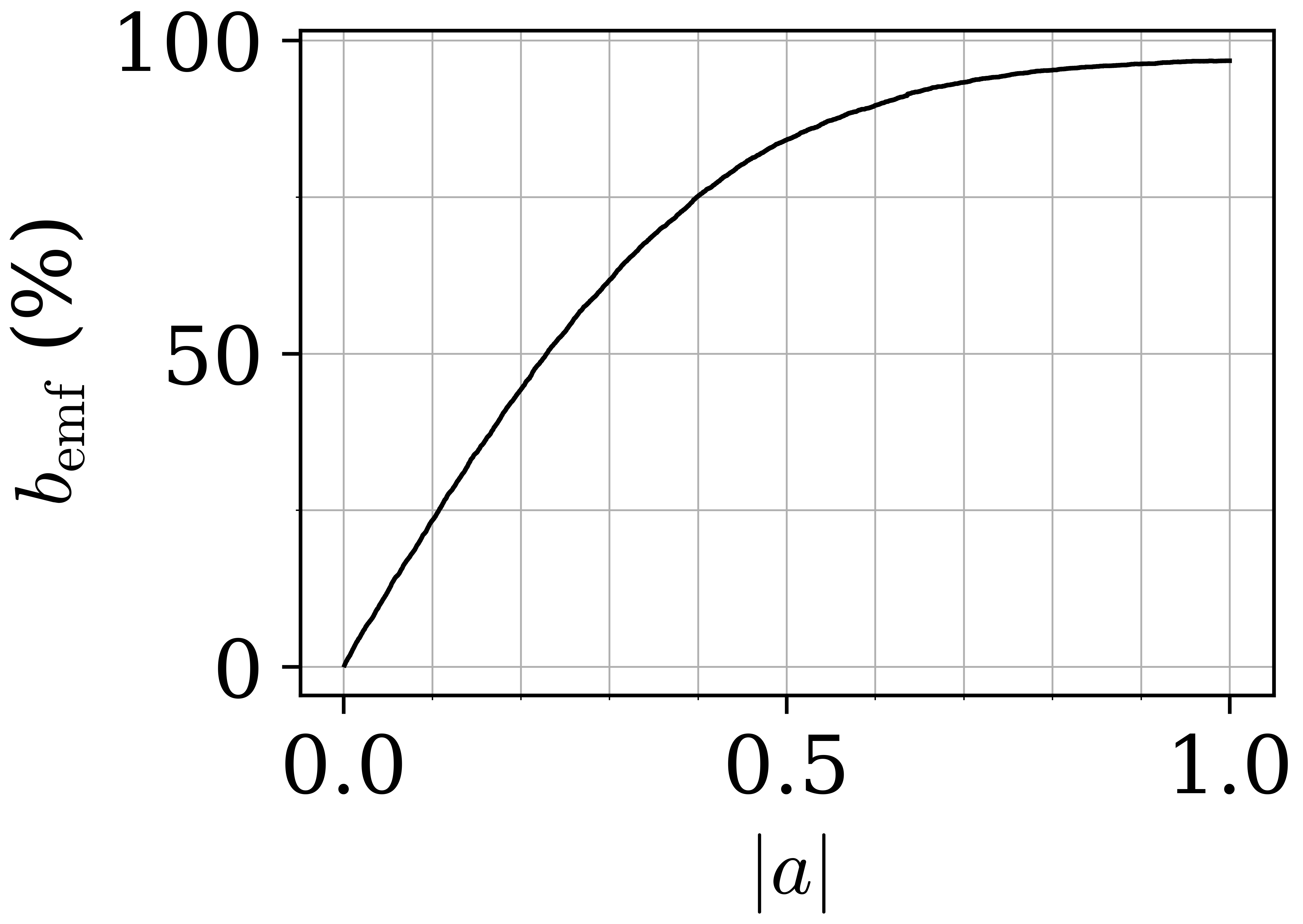}
  \caption{}
\end{subfigure}
\hfill
\begin{subfigure}[t]{0.32\textwidth}
  \centering
  \maybeincludegraphics[width=\linewidth]{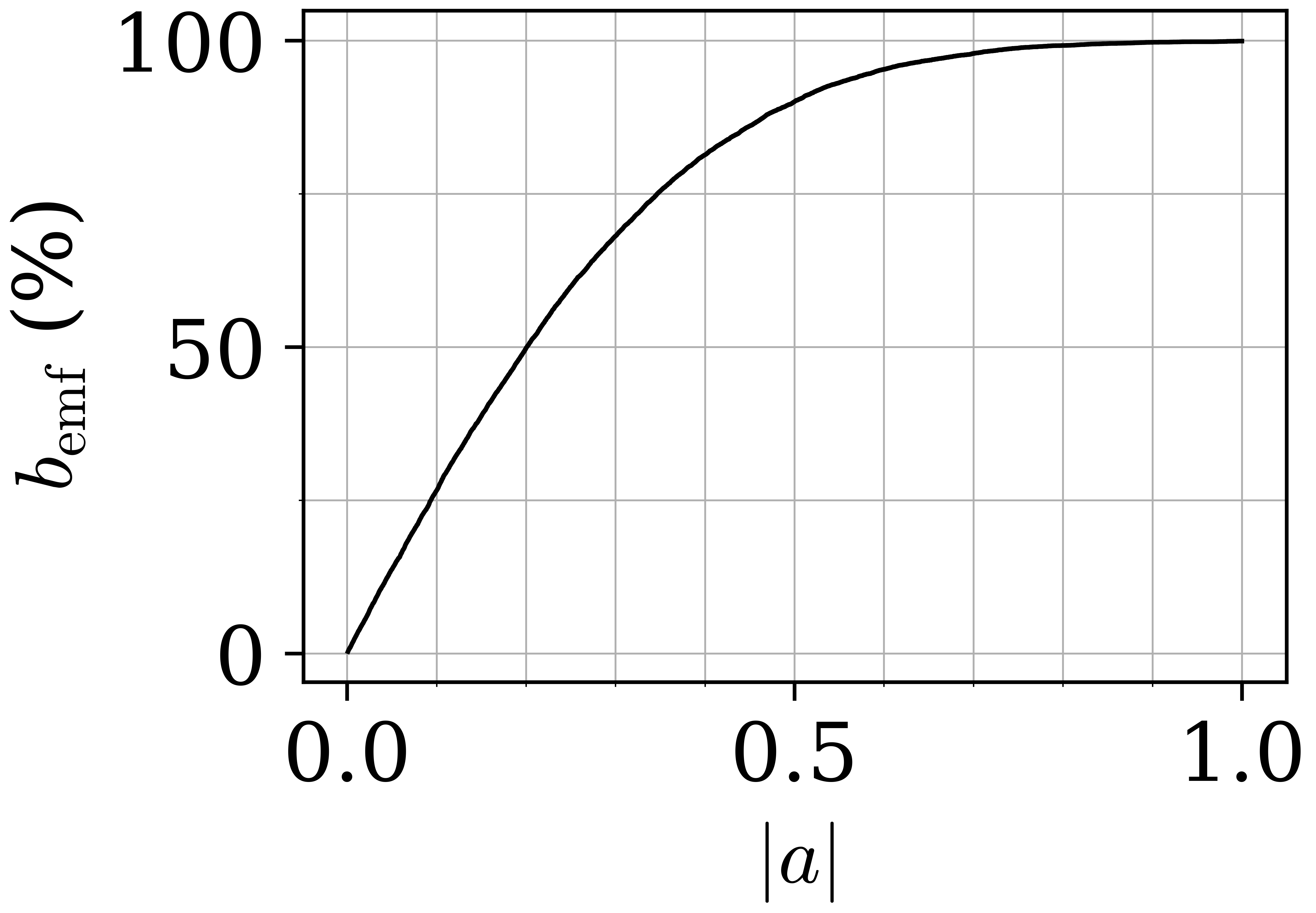}
  \caption{}
\end{subfigure}
\hfill
\begin{subfigure}[t]{0.32\textwidth}
  \centering
  \maybeincludegraphics[width=\linewidth]{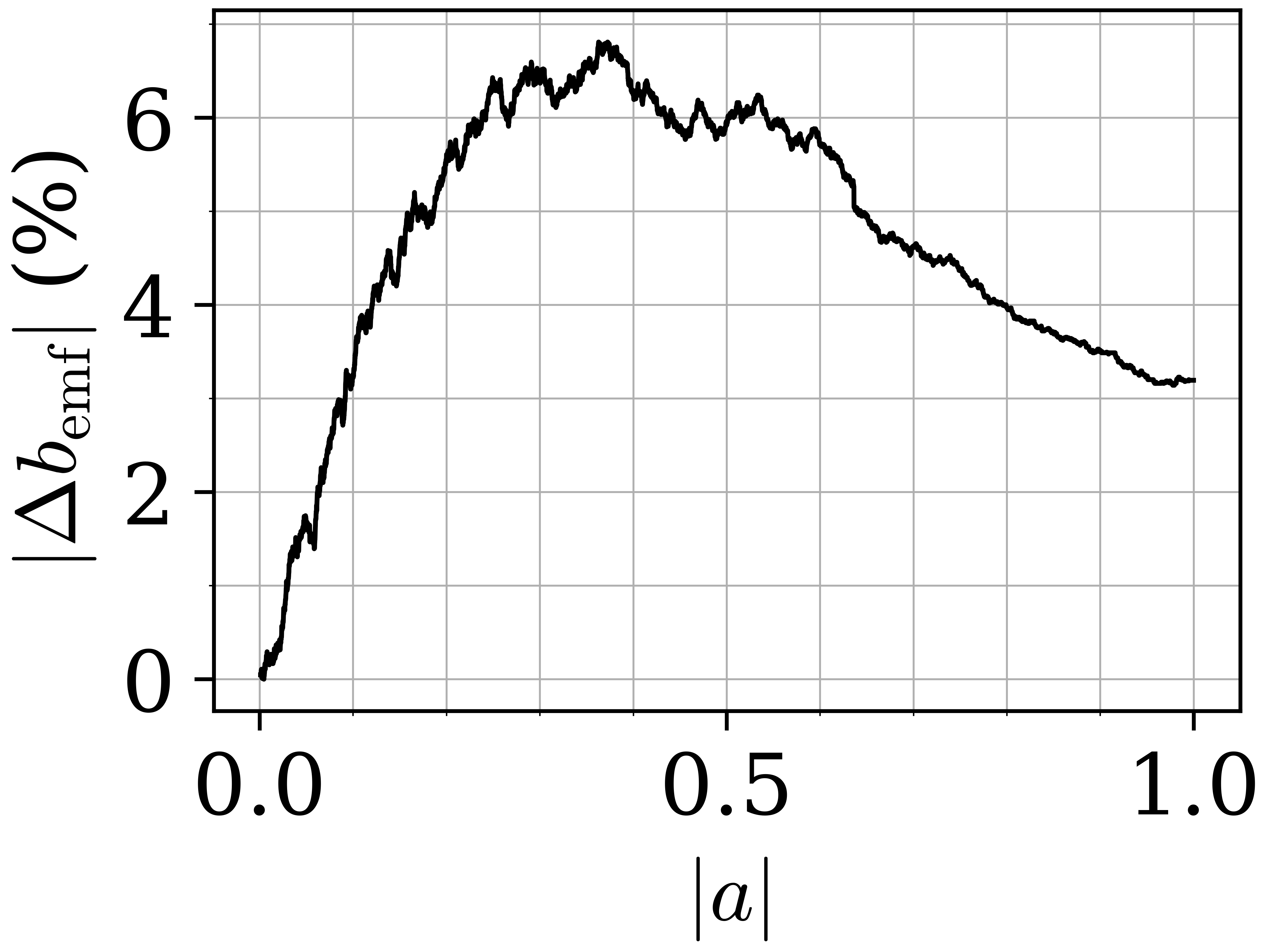}
  \caption{}
\end{subfigure}

\caption{\textbf{E2}: Membership analysis with five noise realizations. (a) Test $S_{11}$ data and learned $\alpha$-cut bounds. 
(b)--(c) Training and test empirical membership fractions with $K=5$. (d) Absolute train--test gap.}
\label{fig:membership_combo_diffseeds}
\end{figure}

\subsubsection{Uncertainty propagation in a FEM setting }\label{sec:fem_validation_500}

To demonstrate that the learned closed-form free energy density expressions are directly usable in a downstream mechanics simulation, 
we deploy the three learned potentials from the baseline experiment \textbf{E1} of Section~\ref{sec:results_500} ($\underline{\Psi}^*(\mathbf{F})$, 
${\Psi}^*(\mathbf{F})$, and $\overline{\Psi}^*(\mathbf{F})$), each corrected to a stress-free reference configuration, 
in a 3D finite element solver implemented in FEniCS~\cite{fenics1,fenics2}.
\rev{The domain is the unit cube $[0,1]^3\,\mathrm{m}^3$ (edge length $1\,\mathrm{m}$), discretized with a structured $15\times15\times15$ linear tetrahedral mesh on the full domain
(no symmetry reduction).
The loading is a displacement-driven end-clamped stretch along the $z$-axis:
on the bottom face $z=0$ we prescribe $\mathbf{u}=\mathbf{0}$, and on the top face $z=1$ we prescribe
$u_x=u_y=0$ together with an axial displacement $u_z=\delta$, where $\delta$ is ramped from $0$ to $0.2\,\mathrm{m}$
over $20$ equal load steps (a $20\%$ end displacement relative to the unit edge length).
The lateral faces are traction-free, and no body forces or surface tractions are applied.}
Fig.~\ref{fig:appD_fem_validation_500} shows the deformed configurations at the final load step. All three potentials yield smooth, physically admissible states.
As observed in the figure, the upper-bound configuration develops higher $S_{33}$ magnitudes than the mean, while the lower-bound configuration develops lower magnitudes.

\begin{figure}[H]
\centering
\begin{subfigure}[t]{0.32\linewidth}
  \centering
  \maybeincludegraphics[width=\linewidth]{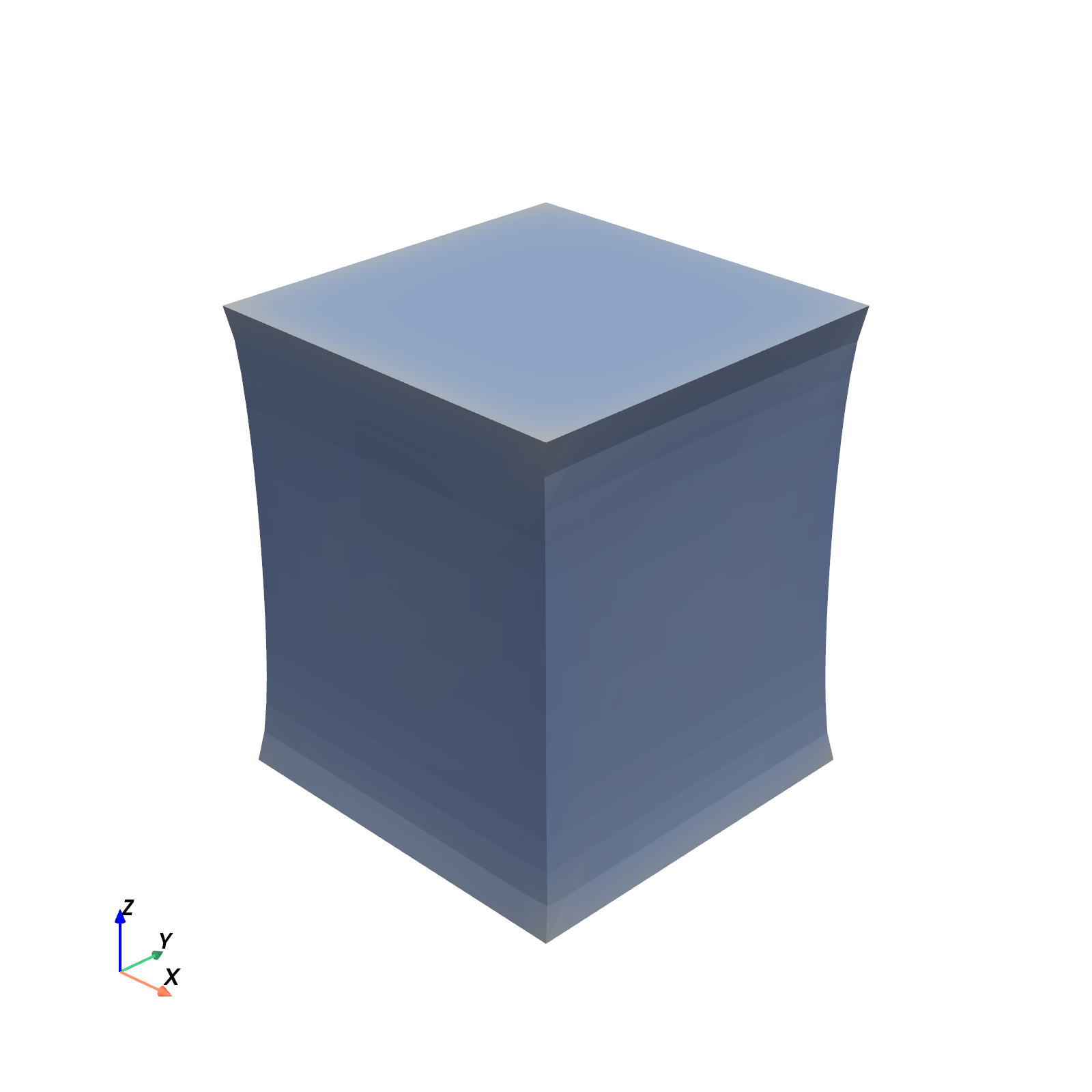}
  \caption{Deformed, $\underline{\Psi}^*$ (LB)}
  \label{fig:appD_lb}
\end{subfigure}\hfill
\begin{subfigure}[t]{0.32\linewidth}
  \centering
  \maybeincludegraphics[width=\linewidth]{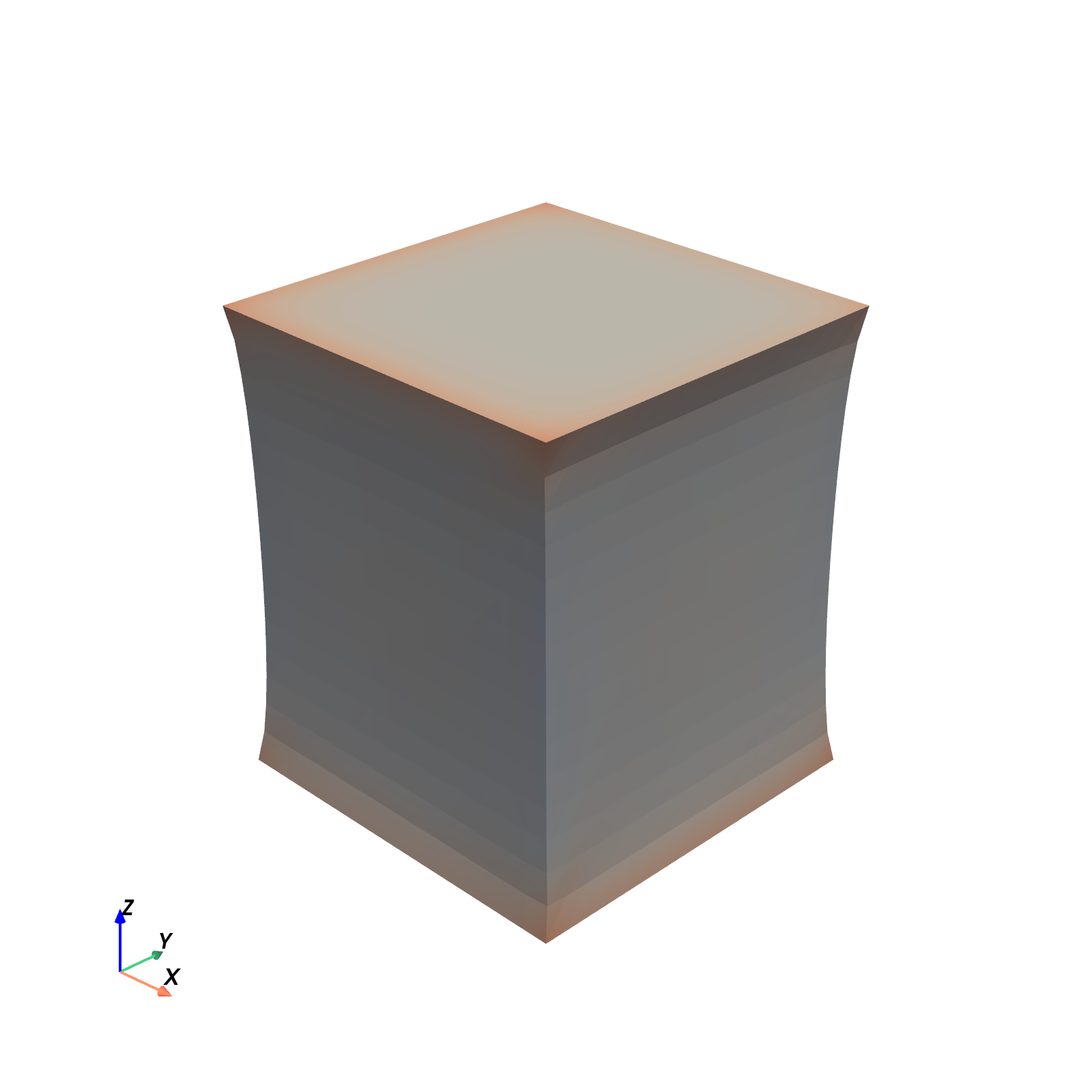}
  \caption{Deformed, $\Psi^*$ (Mean)}
  \label{fig:appD_mean}
\end{subfigure}\hfill
\begin{subfigure}[t]{0.32\linewidth}
  \centering
  \maybeincludegraphics[width=\linewidth]{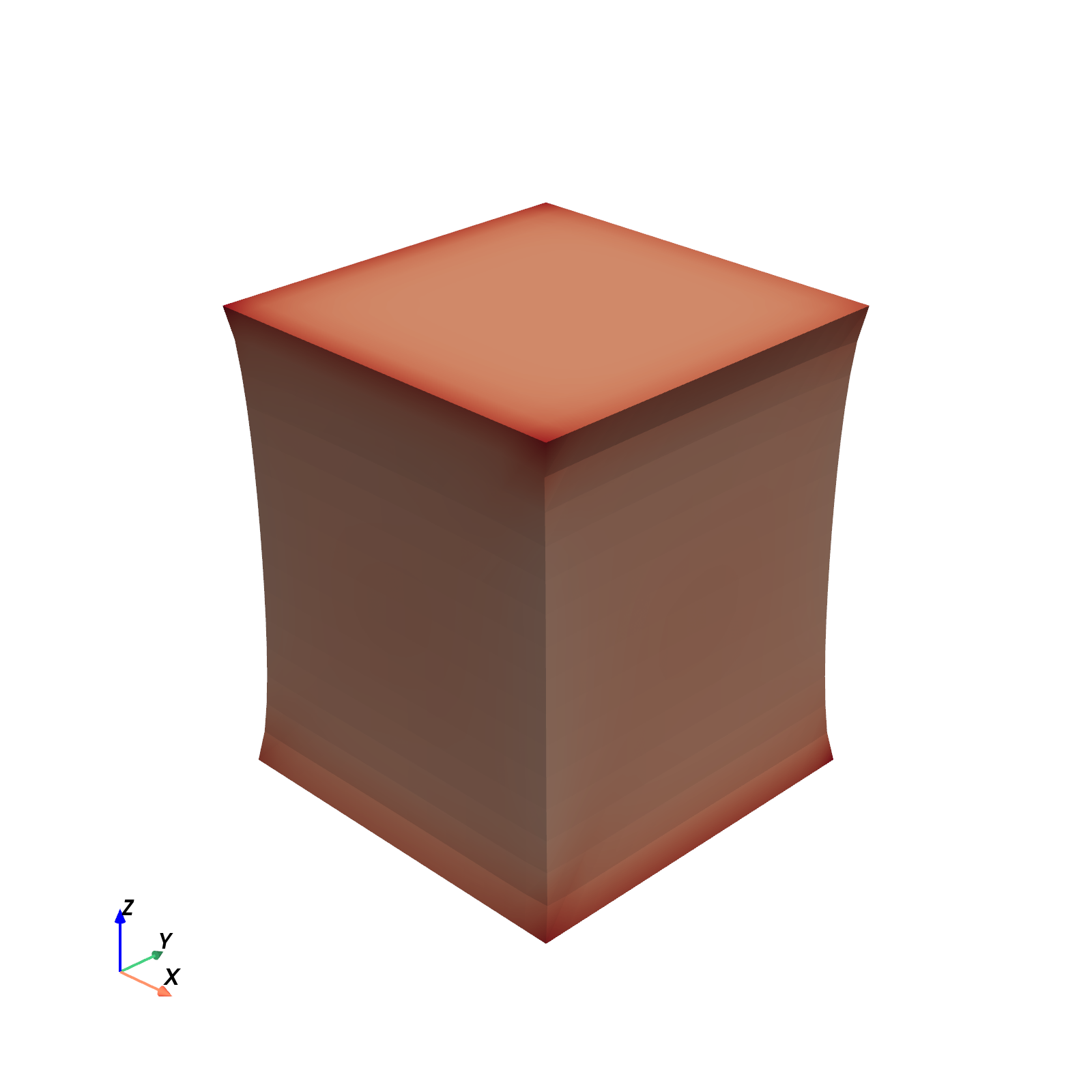}
  \caption{Deformed, $\overline{\Psi}^*$ (UB)}
  \label{fig:appD_ub}
\end{subfigure}

\vspace{0.15cm}
\centering
\maybeincludegraphics[width=0.62\linewidth]{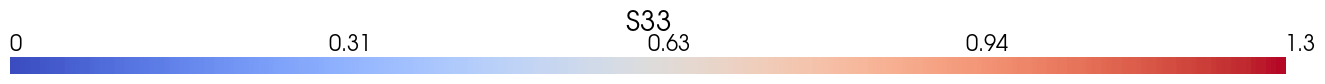}

\caption{\textbf{E1}: FEM demonstration of learned free energy density.}
\label{fig:appD_fem_validation_500}
\end{figure}

\FloatBarrier

\section{Conclusion}\label{sec:conclusion}

This work introduced interval and fuzzy physics-augmented neural network framework for uncertainty-aware  hyperelastic constitutive modeling.
Relative to widely-used Bayesian neural network approaches, as illustrated in Section~\ref{sec:simple_bayesian_example}, the proposed formulation offers a complementary non-probabilistic route to uncertainty quantification that is comparatively lightweight at inference: once trained, each constitutive branch is evaluated deterministically, without posterior sampling or Monte Carlo propagation.
This deterministic enclosure is attractive in mechanics settings where labeled stress data are limited, distributional assumptions on aleatoric noise are 
difficult to justify, and practitioners often require worst- or best-case bounds on the observed response rather than full predictive distributions.
In addition, smoothed $L_0$ sparsification yields compact, inspectable free energy density expressions, in contrast to typically dense Bayesian surrogates, while input-convex architectures promote convexity with respect to the chosen invariants. 

The framework is built in two layers that share the same learned mean, lower, and upper energy branches.
At the interval level, iPANNs combine these ingredients to produce physics-consistent stress predictions through automatic differentiation and to certify whether noisy observations lie inside a deterministic admissible enclosure.
At the fuzzy level, fPANNs embed the three iPANN branches into a membership-based representation through $\alpha$-cut interpolation, yielding a nested family of admissible responses between the central mean ($\alpha=1$) and the extreme bounds ($\alpha=0$).
Thus, whereas an iPANN is appropriate when a binary containment statement or a conservative envelope suffices, an fPANN additionally lets the practitioner tune conservatism through $\alpha$ and interpret coverage in membership terms, as quantified in Section~\ref{sec:membership_analysis}, without retraining separate models for each target coverage level.
The numerical studies demonstrate that the framework can capture aleatoric uncertainty arising from synthetic noise, representative 
of different modalities (experimental data acquisition noise, batch-to-batch material variability, etc). 
Despite this conceptual distinction, both representations remain equally suited to forward deployment in finite element analysis: any selected iPANN bound or fPANN $\alpha$-cut defines a physics-consistent energy branch that can be passed directly to a solver for uncertainty propagation.

Future work will extend the proposed framework beyond synthetic stress–deformation datasets  toward richer experimental settings and broader material behaviors.
A key direction is the  integration of multi-modal experiments, where global force–displacement data, local deformation  measurements, imaging, and possibly microstructural descriptors are used jointly to constrain the learned uncertainty bounds.
Full-field experimental data from digital image correlation or digital  volume correlation will be particularly important for training and validating the framework under  heterogeneous deformation states that are more representative of realistic boundary value problems.
In addition, future developments will incorporate dissipative mechanisms, including  viscoelasticity, plasticity and damage, where uncertainty must be propagated not only through the free energy density but also through internal variables, evolution laws, and path-dependent dissipation.
Towards this goal, past work on modular constructions for learning inelastic responses will be central  to this effort~\cite{fuhg2023modular}.
These extensions will broaden the applicability of fuzzy physics-augmented learning toward experimentally grounded, uncertainty-aware constitutive models for complex materials and  engineering simulations that are structured in-line with data acquisition approaches from mechanics  experimental modalities, but also maintain trustworthiness in forward deployment for the solution  of specific boundary value problems.

\section*{Declaration of generative AI and AI-assisted technologies use}
During the preparation of this work, the author(s) used generative AI and AI-assisted tools, including ChatGPT and Cursor, in order to improve the language, clarity, and readability of the manuscript text. Additionally, these tools were utilized to assist in debugging and optimizing the source code used for data processing and simulation. After using these tools, the author(s) thoroughly reviewed, verified, and edited all outputs to ensure factual and technical accuracy, and take full responsibility for the final content of the publication.

\section*{Declaration of competing interest}
The authors declare that they have no known competing financial interests or personal relationships that could have appeared to influence the work reported in this paper.

\section*{Code and data availability}
The codes and data generated during the current study are available from the authors upon reasonable request.

\section*{Acknowledgements} This work was funded by the Laboratory Directed Research and Development (LDRD) program at Sandia National Laboratories; this funding is gratefully acknowledged.
Sandia National Laboratories is a multi-mission laboratory managed and operated by National Technology \& Engineering Solutions of Sandia, LLC (NTESS), a wholly owned subsidiary of Honeywell International Inc., for the U.S. Department of Energy’s National Nuclear Security Administration (DOE/NNSA) under contract DE-NA0003525. This written work is authored by an employee of NTESS. The employee, not NTESS, owns the right, title and interest in and to the written work and is responsible for its contents. Any subjective views or opinions that might be expressed in the written work do not necessarily represent the views of the U.S. Government. The publisher acknowledges that the U.S. Government retains a non-exclusive, paid-up, irrevocable, world-wide license to publish or reproduce the published form of this written work or allow others to do so, for U.S. Government purposes. 
The DOE will provide public access to results of federally sponsored research in accordance with the DOE Public Access Plan.
\bibliographystyle{elsarticle-num}
\bibliography{refs}

\appendix

\section{Hyperparameter tuning and experiment summary}\label{app:experiment_summary}

\rev{The tightness of the iPANN bounds is a tradeoff between the contributions
$\lambda_{\mathrm{mse}}\mathcal{L}_{\mathrm{mse}}$ and $\lambda_{\mathrm{bound}}\mathcal{L}_{\mathrm{bound}}$
to the total loss in Eq.~\eqref{eq:loss} (during bound training the sparsification term is turned off).
As $\lambda_{\mathrm{bound}}/\lambda_{\mathrm{mse}}\to 0$ the bound branches collapse toward the mean free energy density;
as $\lambda_{\mathrm{bound}}/\lambda_{\mathrm{mse}}\to\infty$ violations are strongly penalized and the
interval widens.
An intermediate ratio sets the practical enclosure width while keeping the mean response close to the
noise-free trend.}
Presented here is the hyperparameter tuning procedure and the experiment-specific settings used to produce the numerical results in Section~\ref{sec:results}.

\subsection{Hyperparameter tuning}\label{sec:hyperparams}

To determine optimal hyperparameters for the loss-function weights in Eq.~\eqref{eq:loss}, we perform a full factorial grid search over four orders of magnitude in $\lambda_{\mathrm{bound}}$ for experiment \textbf{E1}.
The label \emph{raw} on the axes denotes the unweighted loss terms prior to multiplication by the corresponding weights $\lambda_{\mathrm{mse}}$ or $\lambda_{\mathrm{bound}}$.

Fig.~\ref{fig:log-log} summarizes the log--log trade-off between MSE and upper-bound violation over the grid for
$\lambda_{\mathrm{mse}} \in \{1,10,10^{2},10^{3}\}$ and $\lambda_{\mathrm{bound}} \in \{10^{-1},1,10,10^{2},10^{3}\}$.

\begin{figure}[H]
\centering
\includegraphics[width=0.95\linewidth]{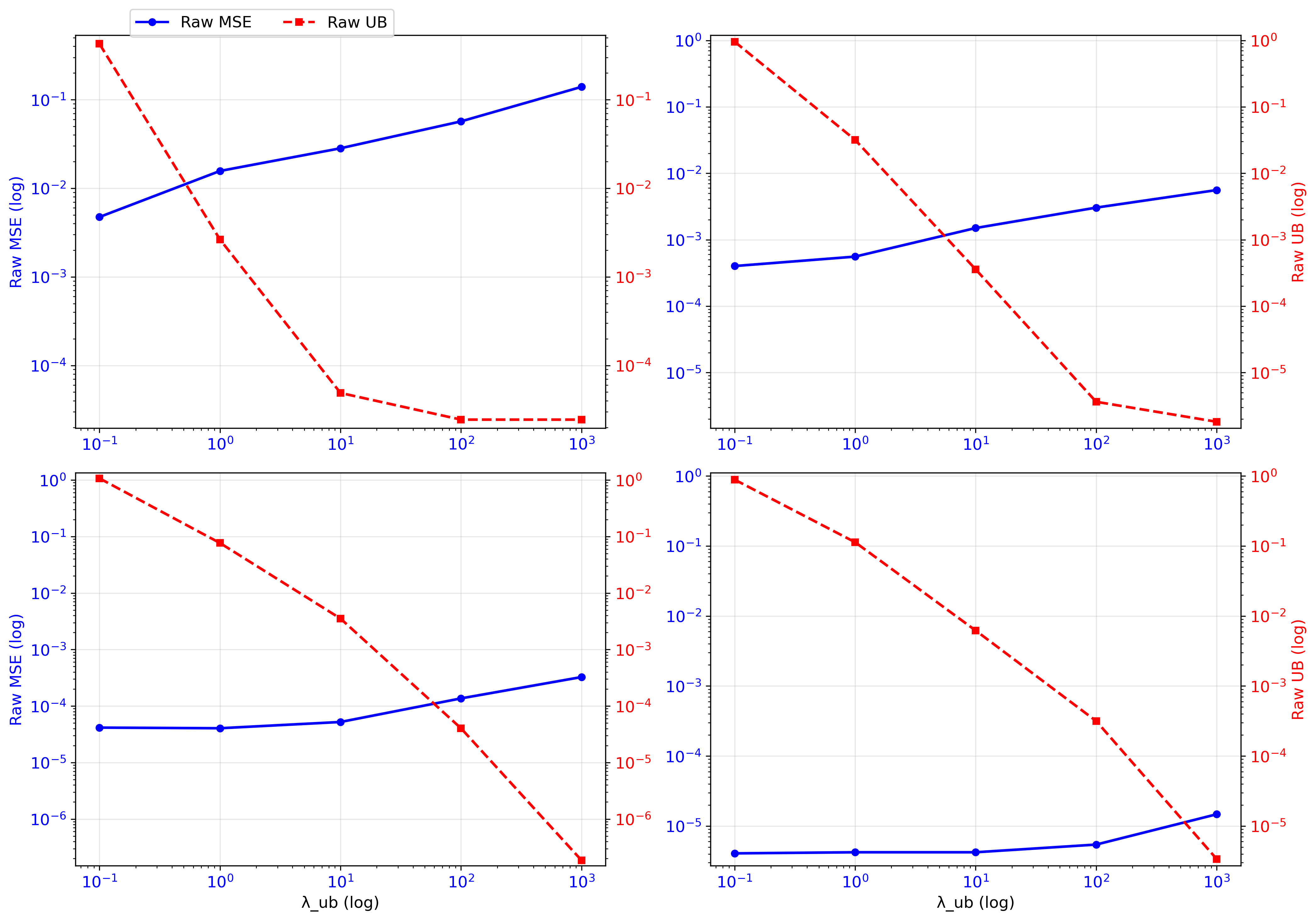}
\caption{Log--log trade-off between the raw MSE term and the raw upper-bound 
(UB) violation term. The four panels 
correspond to $\lambda_{\mathrm{mse}}=1$ (top left), $\lambda_{\mathrm{mse}}=10$ (top right),
$\lambda_{\mathrm{mse}}=10^{2}$ (bottom left), and $\lambda_{\mathrm{mse}}=10^{3}$ (bottom right).}
\label{fig:log-log}
\end{figure}

\subsection{Experiment summary}\label{app:Experiment summary}
All models are implemented in PyTorch~\cite{paszke2019pytorch} and trained with the Adam
optimizer~\cite{kingma2014adam} at a fixed learning rate of $10^{-3}$. Each network is an input
convex neural network with two hidden layers of $30$ neurons and softplus activations. 

Table~\ref{tab:experiment_summary} lists the hyperparameter values used for each experiment (\textbf{E1}--\textbf{E4}).

\begin{table}[htbp]
\centering
\caption{Experiment-specific hyperparameters.}
\label{tab:experiment_summary}
\small
\setlength{\tabcolsep}{3.5pt}
\renewcommand{\arraystretch}{1.12}
\begin{tabular}{|c|>{\raggedright\arraybackslash}p{2.55cm}|c|c|c|c|c|c|c|}
\hline
ID & Experiment & $\lambda_{\mathrm{mse}}$ & $\lambda_{\mathrm{bound (UB, LB)}}$ &
$\lambda_{\mathrm{sparse}}^{\mathrm{M (mean)}}$ & $\lambda_{\mathrm{sparse}}^{\mathrm{bound (UB, LB)}}$ &
$N_{\mathrm{epochs}}^{\mathrm{M (mean)}}$ & $N_{\mathrm{epochs}}^{\mathrm{UB}}$ & $N_{\mathrm{epochs}}^{\mathrm{LB}}$ \\
\hline
\textbf{E1} & 500 samples & $1$ & $100$ & $5{\times}10^{-3}$ & $5{\times}10^{-6}$ & $250$ & $150$ & $150$ \\
\hline
\textbf{E2} & Diff.\ seeds ($500{\times}5$) & $1$ & $100$ & $5{\times}10^{-3}$ & $5{\times}10^{-6}$ & $50$ & $50$ & $50$ \\
\hline
\textbf{E3} & Diff.\ means ($500{\times}5$) & $1$ & $100$ & $5{\times}10^{-3}$ & $5{\times}10^{-6}$ & $50$ & $50$ & $50$ \\
\hline
\textbf{E4} & Diff.\ std ($500{\times}5$) & $1$ & $1000$ & $5{\times}10^{-3}$ & $5{\times}10^{-6}$ & $50$ & $50$ & $50$ \\
\hline
\end{tabular}
\end{table}

\section{Invariant-space sampling}\label{app:invariant_sampling}

Fig.~\ref{fig:invariant_sampling} shows the $500$ FPS--SA selected invariant triplets used for training-data generation, together with the underlying convex hull of physically admissible invariant states obtained from the Latin Hypercube sampling of Section~\ref{sec:data_generation}.

\begin{figure}[p]
\centering
\begin{subfigure}[t]{\textwidth}
  \centering
  \maybeincludegraphics[width=\linewidth,height=0.27\textheight,keepaspectratio]{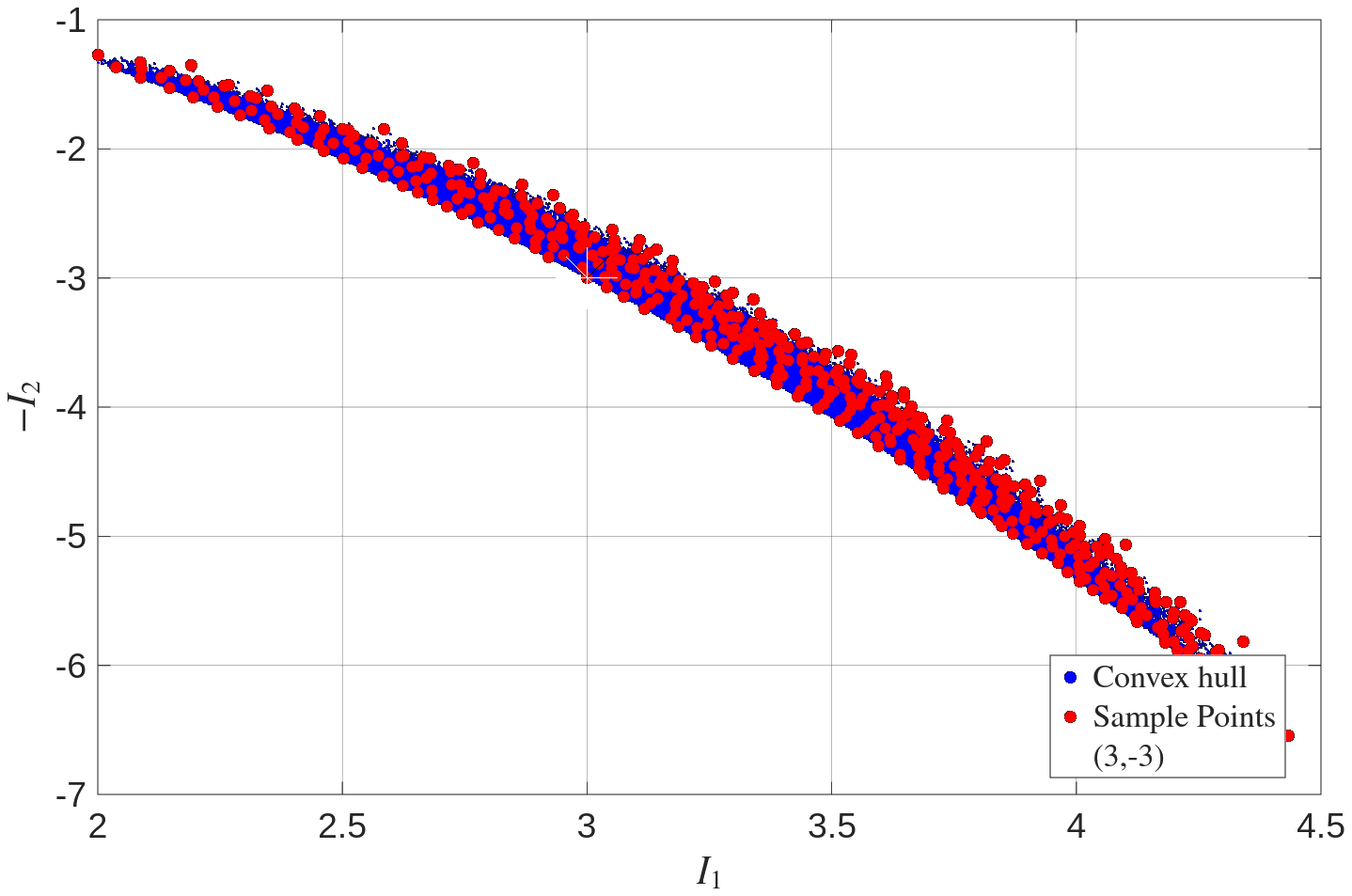}
  \caption{$(I_1,-I_2)$ plane.}
\end{subfigure}

\vspace{0.1cm}

\begin{subfigure}[t]{\textwidth}
  \centering
  \maybeincludegraphics[width=\linewidth,height=0.27\textheight,keepaspectratio]{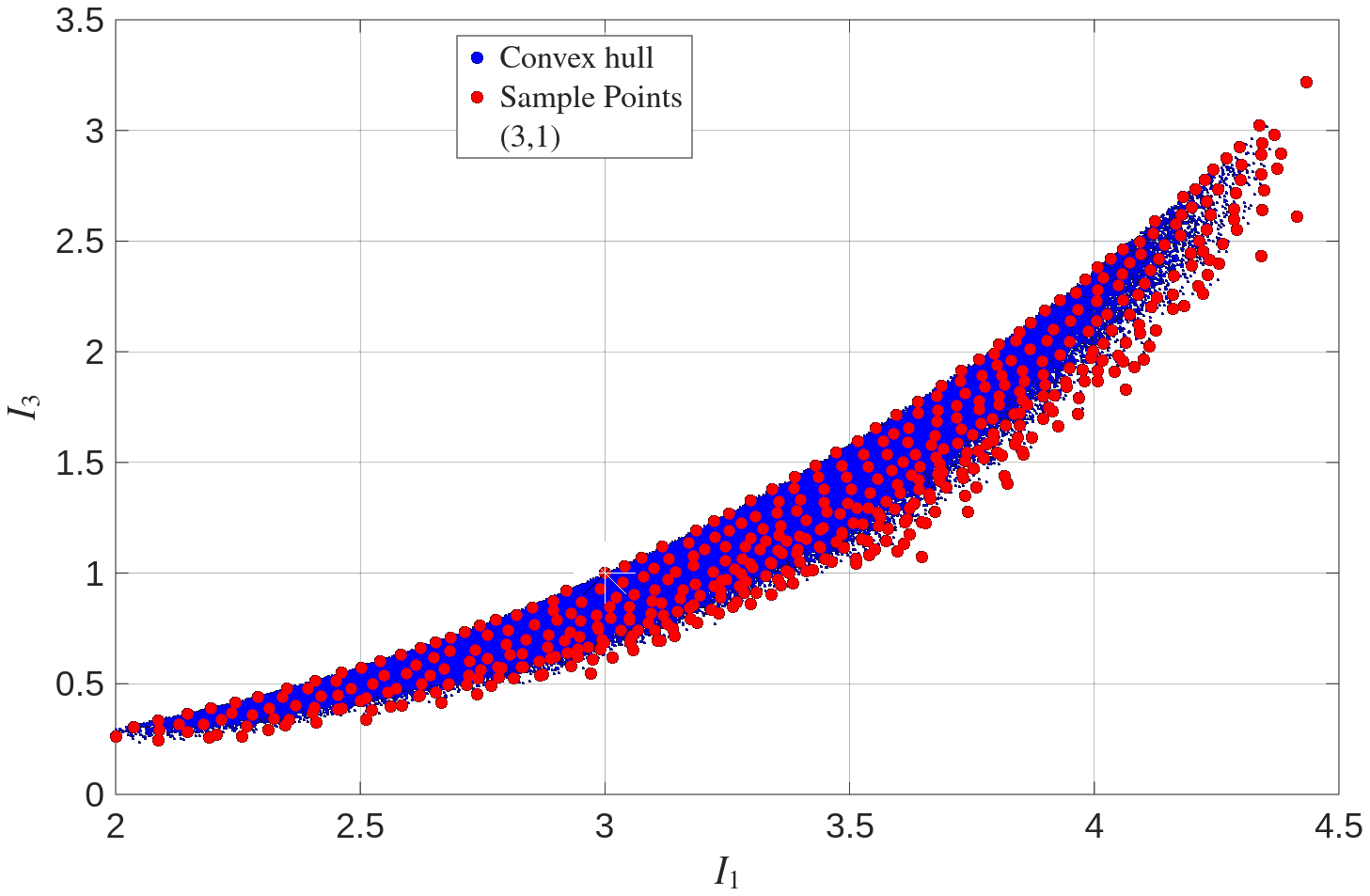}
  \caption{$(I_1,I_3)$ plane.}
\end{subfigure}

\vspace{0.1cm}

\begin{subfigure}[t]{\textwidth}
  \centering
  \maybeincludegraphics[width=\linewidth,height=0.27\textheight,keepaspectratio]{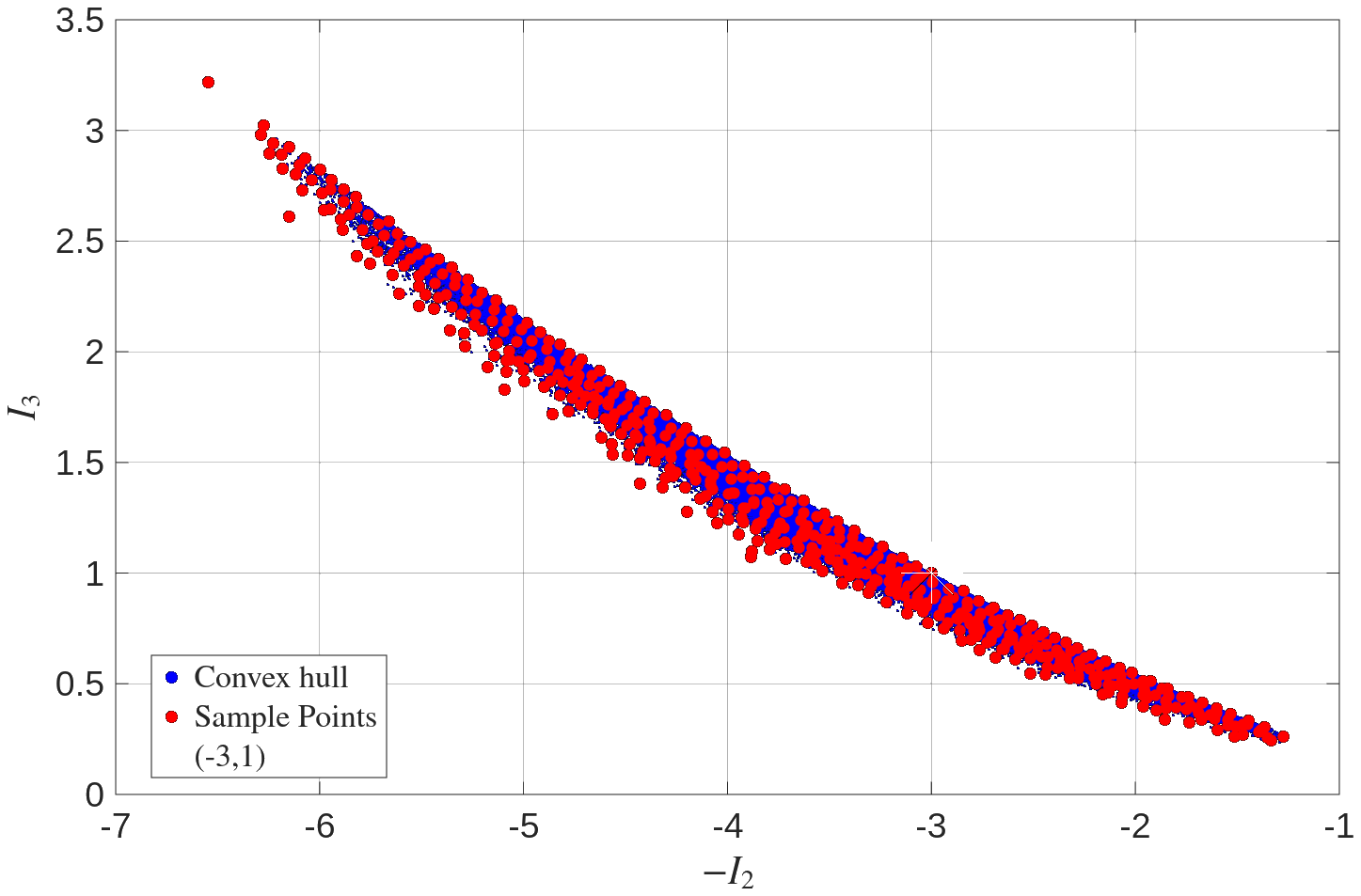}
  \caption{$(-I_2,I_3)$ plane.}
\end{subfigure}

\caption{Invariant space with convex hull and selected invariants used for synthetic data generation.}
\label{fig:invariant_sampling}
\end{figure}

\section{Learned free energy density expressions}\label{app:learned_expressions}

For each experiment, $L_0$ sparsification leaves us with an interpretable free energy density expression. 
We report the free energy density expressions for experiments \textbf{E1}--\textbf{E4} in the tables ~\ref{tab:psi_function_forms}
--\ref{tab:psi_final_forms}, \ref{tab:e1_composed_psi_coeffs}
--\ref{tab:e4_composed_psi_coeffs}. 

\begin{table}[htbp]
\centering
\footnotesize
\renewcommand{\arraystretch}{1.5}
\setlength{\tabcolsep}{6pt}
\caption{Elementary function forms used in the learned free energy density, and the experiments in which each appears.}
\label{tab:psi_function_forms}
\begin{tabular}{clc}
\hline
Symbol & Definition & Used in \\
\hline
$\mathrm{sp}(x)$ & $e^{x}+1$ & \textbf{E1}--\textbf{E4} \\
$g$ & $\mathrm{sp}(a I_1 - b J)$ \;(\textbf{E1}); \quad $\mathrm{sp}(a I_1)$ \;(\textbf{E4}) & \textbf{E1}, \textbf{E4} \\
$\phi_k$ & $\mathrm{sp}(-a_k I_2),\quad k=1,2,3$ & \textbf{E2}, \textbf{E3} \\
$\psi_\ell$ & $\mathrm{sp}(b_\ell I_1),\quad \ell=1,2$ & \textbf{E2}, \textbf{E3} \\
$\mathcal{L}_1$ & $\log\!\bigl(c_4\, g^{p_1}\, e^{d_1 I_2 + d_2 J} + 1\bigr)$ & \textbf{E1}, \textbf{E4} \\
$\mathcal{L}_1'$ & $\log\!\bigl(\phi_1^{p_1}\phi_2^{p_2}\phi_3^{p_3}\,\psi_1^{p_4}\psi_2^{p_5}\, e^{d_1 I_2 + d_2 J} + 1\bigr)$ & \textbf{E2}, \textbf{E3} \\
$\mathcal{L}_2$ & $\log\!\bigl(g^{p_2}\, e^{d_3 I_2} + 1\bigr)$ & \textbf{E1} \\
$\mathcal{L}_2'$ & $\log\!\bigl(g^{p_2}\, e^{d_3 I_2 + d_4 J} + 1\bigr)$ & \textbf{E4} \\
$\mathcal{L}_3$ & $\log\!\bigl(\mathrm{sp}(d_\star I_2)\bigr)$,\quad $d_\star=d_4$ (\textbf{E1}), $d_\star=d_3$ (\textbf{E2}, \textbf{E3}) & \textbf{E1}, \textbf{E2}, \textbf{E3} \\
\hline
\end{tabular}
\end{table}

\begin{table}[htbp]
\centering
\small
\renewcommand{\arraystretch}{1.5}
\setlength{\tabcolsep}{8pt}
\caption{Final learned free energy density for each experiment, written using the function forms of Table~\ref{tab:psi_function_forms}. Coefficient values are listed in Tables~\ref{tab:e1_composed_psi_coeffs}--\ref{tab:e4_composed_psi_coeffs}.}
\label{tab:psi_final_forms}
\begin{tabular}{cl}
\hline
ID & $\Psi(I_1,I_2,J)$ \\
\hline
\textbf{E1} & $c_1 I_1 + c_2 J + c_3 \mathcal{L}_1 + c_{10}\mathcal{L}_2 + c_{13}\mathcal{L}_3 + c_{15}$ \\
\textbf{E2} & $c_2 J + c_3 \mathcal{L}_1' + c_{10}\mathcal{L}_3 + c_{15}$ \\
\textbf{E3} & $c_1 I_1 + c_2 J + c_3 \mathcal{L}_1' + c_{10}\mathcal{L}_3 + c_{15}$ \\
\textbf{E4} & $c_1 I_1 + c_2 I_2 + c_3 J + c_{10}\mathcal{L}_1 + c_{13}\mathcal{L}_2' + c_{15}$ \\
\hline
\end{tabular}
\end{table}

\begin{table}[htbp]
\centering
\footnotesize
\begin{minipage}[t]{0.48\textwidth}
\centering
\captionof{table}{Coefficient values for \textbf{E1} (LB, M, UB).}
\label{tab:e1_composed_psi_coeffs}
\setlength{\tabcolsep}{2.5pt}
\begin{tabular}{lrrr}
\hline
Parameter & LB & M & UB \\
\hline
$c_1$ & 0.452959 & 0.594047 & 0.849586 \\
$c_2$ & 0.457589 & 0.452273 & 0.462377 \\
$c_3$ & 1.7295 & 1.87905 & 2.22363 \\
$c_4$ & 2.47051 & 2.91966 & 3.5611 \\
$a$ & 0.347429 & 0.400356 & 0.417636 \\
$b$ & 0.645014 & 0.420269 & 0.522768 \\
$p_1$ & 0.591559 & 0.722871 & 0.742294 \\
$d_1$ & $-1.10195$ & $-1.02393$ & $-1.31852$ \\
$d_2$ & $-2.2273$ & $-2.00959$ & $-1.92994$ \\
$c_{10}$ & 1.29483 & 1.44046 & 1.72594 \\
$p_2$ & 2.14771 & 2.18066 & 2.37474 \\
$d_3$ & $-0.674068$ & $-0.788953$ & $-0.937646$ \\
$c_{13}$ & 0.307104 & 0.382489 & 0.370604 \\
$d_4$ & 0.477827 & 0.61749 & 0.616056 \\
$c_{15}$ & $-3.20733$ & $-4.19261$ & $-4.87605$ \\
\hline
\end{tabular}
\end{minipage}
\hfill
\begin{minipage}[t]{0.48\textwidth}
\centering
\captionof{table}{Coefficient values for \textbf{E2} (LB, M, UB).}
\label{tab:e2_composed_psi_coeffs}
\setlength{\tabcolsep}{2.5pt}
\begin{tabular}{lrrr}
\hline
Parameter & LB & M & UB \\
\hline
$c_2$ & 0.954417 & 1.26118 & 2.3788 \\
$c_3$ & 1.92633 & 2.21234 & 2.56081 \\
$a_1$ & 1.16395 & 1.08832 & 0.980695 \\
$p_1$ & 1.6508 & 2.34019 & 2.76091 \\
$a_2$ & 0.378614 & 0.466638 & 0.519428 \\
$p_2$ & 0.553723 & 0.607677 & 0.830039 \\
$a_3$ & 0.228918 & 0.23031 & 0.370241 \\
$p_3$ & 0.323998 & 0.366854 & 0.522393 \\
$b_1$ & 0.18661 & 0.243268 & 0.273426 \\
$p_4$ & 0.145996 & 0.195872 & 0.261537 \\
$b_2$ & 0.99083 & 1.06499 & 1.14564 \\
$p_5$ & 0.61237 & 0.691066 & 0.77266 \\
$d_1$ & $-0.386977$ & $-0.404932$ & $-0.425268$ \\
$d_2$ & $-0.974285$ & $-1.22536$ & $-1.38856$ \\
$c_{10}$ & 0.530655 & 0.569918 & 0.614059 \\
$d_3$ & 0.567252 & 0.722123 & 0.751164 \\
$c_{15}$ & $-3.48289$ & $-4.55937$ & $-6.77858$ \\
\hline
\end{tabular}
\end{minipage}

\vspace{0.8em}

\begin{minipage}[t]{0.48\textwidth}
\centering
\captionof{table}{Coefficient values for \textbf{E3} (LB, M, UB).}
\label{tab:e3_composed_psi_coeffs}
\setlength{\tabcolsep}{2.5pt}
\begin{tabular}{lrrr}
\hline
Parameter & LB & M & UB \\
\hline
$c_1$ & 0.277304 & 0.427616 & 0.809184 \\
$c_2$ & 0.316756 & 0.573037 & 1.03962 \\
$c_3$ & 1.06839 & 1.34332 & 1.66632 \\
$a_1$ & 1.94751 & 1.34669 & 1.2581 \\
$p_1$ & 1.02343 & 1.90377 & 1.57522 \\
$a_2$ & 1.94579 & 1.34467 & 1.25072 \\
$p_2$ & 1.65348 & 1.22275 & 2.3485 \\
$a_3$ & 0.777113 & 0.707635 & 0.634392 \\
$p_3$ & 0.866477 & 1.25316 & 1.56753 \\
$b_1$ & 0.411716 & 0.501003 & 0.540552 \\
$p_4$ & 0.666365 & 0.846696 & 0.937391 \\
$b_2$ & 0.480243 & 0.581122 & 0.619156 \\
$p_5$ & 0.80091 & 1.0036 & 1.09526 \\
$d_1$ & $-0.600601$ & $-0.71964$ & $-0.828024$ \\
$d_2$ & $-1.08002$ & $-1.28093$ & $-1.35107$ \\
$c_{10}$ & 0.297127 & 0.449434 & 0.462303 \\
$d_3$ & 0.294613 & 0.534423 & 0.561432 \\
$c_{15}$ & $-2.03011$ & $-3.6663$ & $-5.80171$ \\
\hline
\end{tabular}
\end{minipage}
\hfill
\begin{minipage}[t]{0.48\textwidth}
\centering
\captionof{table}{Coefficient values for \textbf{E4} (LB, M, UB).}
\label{tab:e4_composed_psi_coeffs}
\setlength{\tabcolsep}{2.5pt}
\begin{tabular}{lrrr}
\hline
Parameter & LB & M & UB \\
\hline
$c_1$ & 0.187563 & 0.426722 & 0.768818 \\
$c_2$ & 0.06284 & 0.292087 & 0.446485 \\
$c_3$ & 0.260931 & 0.334159 & 0.458835 \\
$c_{10}$ & 1.63037 & 2.06464 & 2.58183 \\
$c_4$ & 1.63595 & 2.59717 & 3.87735 \\
$a$ & 0.504571 & 0.63259 & 0.673222 \\
$p_1$ & 0.107058 & 0.253714 & 0.310031 \\
$d_1$ & $-2.30536$ & $-1.68551$ & $-1.83049$ \\
$d_2$ & $-0.386197$ & $-0.115956$ & $-0.103846$ \\
$c_{13}$ & 1.44658 & 2.13406 & 2.93725 \\
$p_2$ & 0.794768 & 1.31416 & 1.58833 \\
$d_3$ & $-0.486883$ & $-0.765099$ & $-0.997504$ \\
$d_4$ & $-0.835381$ & $-0.842919$ & $-0.971181$ \\
$c_{15}$ & $-1.4934$ & $-3.58473$ & $-5.5055$ \\
\hline
\end{tabular}
\end{minipage}
\end{table}

\end{document}